\documentclass{article}
\usepackage[utf8]{inputenc}

\usepackage{graphicx}
\usepackage{amssymb}                       
\usepackage{amsmath}
\usepackage{color}
\usepackage{times}
\usepackage{bm,xcolor} 
\usepackage{mathtools}
\usepackage{bbm}
\usepackage{wrapfig}
\usepackage{subcaption}

\usepackage{booktabs}

\newcommand{\bU}{U}
\newcommand{\bS}{S}

\newcommand{\cW}{\mathcal{W}}
\newcommand{\cA}{\mathcal{A}}
\newcommand{\cB}{\mathcal{B}}

\newcommand{\pathvar}{p}

\definecolor{mahogany}{cmyk}{0, 0.77, 0.87, 0}
\definecolor{salmon}{cmyk}{0, 0.53, 0.38, 0}
\definecolor{melon}{cmyk}{0, 0.46, 0.50, 0}
\definecolor{yellowgreen}{cmyk}{0.44, 0, 0.74, 0}
\definecolor{brickred}{cmyk}{0, 0.89, 0.94, 0.28}
\definecolor{OliveGreen}{cmyk}{0.64, 0, 0.95, 0.40}
\definecolor{RawSienna}{cmyk}{0, 0.72, 1.0, 0.45}
\definecolor{ZurichRed}{rgb}{1, 0, 0} 

\usepackage{fancyhdr}
\usepackage{amsmath,amstext,amssymb,amsopn,amsthm}
\usepackage{amsmath,amssymb,amsthm}
\usepackage[mathscr]{eucal}
\usepackage[colorlinks,citecolor=cyan]{hyperref}
\usepackage{url}

\newtheorem{definition}{Definition}

\newtheorem{corollary}{Corollary}
\newtheorem{proposition}{Proposition}
\newtheorem{remark}{Remark}
\newtheorem{theorem}{Theorem}
\newtheorem{lemma}{Lemma}

\newcommand{\rev}[1]{\textcolor{black}{{#1}}}

\title{Geometric Scattering on Measure Spaces}
\author{Joyce Chew \and Matthew Hirn \and Smita Krishnaswamy\thanks{Correspondence to: smita.krishnaswamy@yale.edu} \and Deanna Needell \and Michael Perlmutter\thanks{Correspondence to: mperlmutter@boisestate.edu}\and Holly Steach \and Siddharth Viswanath \and Hau-Tieng Wu}
\date{\today}

\begin{document}

\maketitle
\abstract{The scattering transform is a multilayered, wavelet-based transform initially introduced  as a mathematical model of convolutional neural networks (CNNs) that has played a foundational role in our understanding of these networks' stability and invariance properties. In subsequent years, there has been widespread interest in extending the success of CNNs  
to data sets with non-Euclidean structure, such as graphs and manifolds, leading to the emerging field of geometric deep learning.   
In order to improve our understanding of the architectures used in this new field, several papers have proposed generalizations of the scattering transform for non-Euclidean data structures such as undirected graphs and compact Riemannian manifolds without boundary. Analogous to the original scattering transform, these works prove that these variants of the scattering transform have desirable stability and invariance properties and aim to improve our understanding of the neural networks used in geometric deep learning.

In this paper, we introduce a general, unified model for geometric scattering on measure spaces. Our proposed framework includes previous work on compact Riemannian manifolds without boundary and undirected graphs as special cases but also applies to more general settings such as directed graphs, signed graphs, and manifolds with boundary. We propose a new criterion that identifies to which groups a useful representation 
should be invariant and show that this criterion is sufficient to guarantee that the scattering transform has desirable stability and invariance properties. Additionally, we consider finite measure spaces that are obtained from randomly sampling an unknown manifold. We propose two methods for constructing a data-driven graph on which the associated graph scattering transform approximates the scattering transform on the underlying manifold. Moreover, we use a diffusion-maps based approach to prove quantitative estimates on the rate of convergence of one of these approximations as the number of sample points tends to infinity. Lastly, we showcase the utility of our method on spherical images, a directed graph stochastic block model, and on high-dimensional single-cell  data.}

\section{Introduction}

Many popular machine learning algorithms and architectures either explicitly or implicitly rely on producing a hidden, or transformed, representation of the input data. For example, popular algorithms such as word2vec\cite{church2017word2vec}, node2vec\cite{grover2016node2vec}, and graph2vec\cite{narayanan2017:graph2vec}  explicitly associate each input in a text corpus, network, or collection of networks to a point in a high-dimensional vector space. This transformed representation can then be used for a variety of tasks such as clustering or classification. Deep neural networks, on the other hand, use multilayered architectures to classify an input signal. In this case, the early layers of the network may be viewed as producing a transformed representation of the input and the final layer may be viewed as a classifier acting on the transformed data. In either case, there is a fundamental question. What properties should these hidden representations satisfy in order to be useful for downstream tasks?

In order to help answer this question,  Mallat introduced the scattering transform \cite{mallat:scattering2012}, a wavelet-based architecture which models the hidden representation produced by the early layers of a convolutional neural network (CNN).
Given a function $f\in\mathbf{L}^2(\mathbb{R}^N)$ and a scale parameter $J$, the windowed scattering transform of \cite{mallat:scattering2012} is a countable collection of functions \begin{equation}\label{eqn: basic coefs} S_Jf\coloneqq\{S_J[p]f: p=(j_1,\ldots,j_m), j_i\leq J,m\geq 0\}, 
\end{equation}
where the scattering coefficients $S_J[p]f$ are defined through an alternating sequence of $m$ wavelet convolutions (at scales $j_i$) and nonlinear activations followed by a final convolution against a low-pass averaging filter at scale $2^J$.
If one is interested in classifying many signals $\{f_i\}_{i=1}^{N_\text{signals}},$ they may first transform the input data by computing $S_Jf_i$ for each $i$ and then use these transformed representations as input to a classification model such as a support vector machine.

One of the key insights of   \cite{mallat:scattering2012} is that convolutional architectures naturally have desirable invariance and  equivariance properties with respect to the action of the translation group. Specifically, if $\tau_c$ is the translation operator $\tau_cf(x)\rev{\coloneqq}f(x-c)$, we have  the equivariance relationship
\begin{equation}\label{eqn: equivariance example}
S_J[p](\tau_cf)=\tau_c S_J[p]f,
\end{equation}
where on the right-hand side $\tau_c$ is applied  term by term. 
Moreover, when the scale parameter $J$ tends to infinity, we have the approximate invariance relationship 
\begin{equation}\label{eqn: invariance example}
S_J[p](\tau_cf)\approx S_J[p]f.
\end{equation}
Furthermore, Mallat also shows that the scattering transform is stable to the perturbations of the form $f(x-c(x))$ where is $c(x)$ is a function with bounded gradient and Hessian.

In addition to being a theoretical model, the scattering transform has also proven to be a practical object.
 A notable  difference between the scattering transform and other CNN-like architectures is that it uses predesigned wavelet filters, rather than filters learned from training data. In settings where labeled data is abundant, this may be viewed as a limitation on the expressive power of the scattering transform. However, in the context of unsupervised learning, or limited data environments, it may be difficult or impossible to train a traditional neural network. In these settings, the lack of trainable filters \emph{increases} the practical utility of the scattering transform\cite{LEONARDUZZI201811}. For instance, \cite{saito2017underwater} applied the scattering transform to Sonar data to detect unexploded bombs on the ocean floor despite there only being 14 objects in the data set. Additionally, the scattering transform  can also be used for a variety of other tasks in addition to classification. For example, \cite{bruna:multiscaleMicrocanonical2018} applied it to the texture synthesis problem and \cite{sprechmann2015audio} combined the scattering transform with nonnegative matrix factorization in order to achieve audio source separation.

While CNNs have had tremendous success for tasks related to images, audio signals, and other data with a Euclidean grid-like structure, many modern data sets have an irregular structure and are naturally modeled as more complex structures such as graphs and manifolds. This has led to the new field of \emph{geometric deep learning} \cite{Bronstein:geoDeepLearn2017} which aims to extend the success of CNNs to these irregular domains. In these more general settings, the concept of translation is not well defined. However, invariance and equivariance still play a critical role. For example, nearly all popular graph neural networks are designed so that they are naturally invariant or equivariant to the action of the permutation group, i.e., reordering of the vertices. More generally, one of the principal goals of geometric deep learning is to design architectures that respect the intrinsic symmetries and invariances of the data, which are typically modeled by  group actions \cite{bronstein2021geometric, Cahill2022}.

There are many possible ways to accomplish this goal, but here we will focus on spectral methods based on the eigendecomposition of a suitable Laplace type operator such as the graph Laplacian or \rev{Laplace-Beltrami} operator on a manifold. These methods,
which have been popularized through work such as \cite{shuman:emerging2013}, view the eigenvectors/eigenfunctions of the Laplace operator as generalized Fourier modes and define convolution as multiplication in the Fourier basis analogous to well-known results in the Euclidean case.
This notion of convolution is used in popular graph neural networks such as \cite{bruna2013spectral}, \cite{Defferrard2018}, \cite{Levie:CayleyNets2017}, and \cite{kipf2016semi} and has also been used in manifold neural networks such as \cite{wang2021stability} and 
\cite{wang2021stabilityrel}. 

Following the rise of these spectral networks, several works have introduced versions of the scattering transform for (undirected, unsigned) graphs \cite{zou:graphCNNScat2018,gama:diffScatGraphs2018,gao:graphScat2018,perlmutter2019understanding} and smooth compact manifolds without boundary \cite{perlmutter:geoScatCompactManifold2020}. In these works, the authors assume that one is given a signal $f$ defined on the graph or manifold and use generalizations of the wavelet transform \cite{hammond:graphWavelets2011,coifman:diffWavelets2006} to construct scattering coefficients similar to \eqref{eqn: basic coefs} through an alternating sequence of generalized convolutions and nonlinearities.
They then provide detailed stability and invariance analysis of their respective versions of the scattering transform proving results analogous to \eqref{eqn: equivariance example} and \eqref{eqn: invariance example}. Thereby, they help improve our understanding of spectral networks used in geometric deep learning, analogous to how \cite{mallat:scattering2012} helps us understand Euclidean CNNs. 
Moreover, there has also been work applying the graph scattering transform to combinatorial optimization \cite{Min2022MC} problems and to graph synthesis  \cite{zou:graphScatGAN2019,bhaskar2021molecular}, an important problem with potential applications to drug discovery.

\subsection{Contributions and Main Results}

In this work, we extend  the scattering transform to a general class of measure spaces  $\mathcal{X}\coloneqq(X,\mathcal{F},\mu)$.
Our framework applies both to the settings considered in previous work on geometric scattering, i.e., compact Riemannian manifolds without boundary and unweighted, unsigned graphs, and  also to other interesting examples such as signed or directed graphs, which have not been previously considered in the  literature about the scattering transform. The generality of our construction is motivated in part by the  recent book \cite{bronstein2021geometric} which laments ``there is a veritable zoo of neural network architectures for various kinds of data, but few unifying principles.'' In the same spirit, we look for the general themes that unite spectral networks on different domains and formulate a general theory of scattering networks on measure spaces.

Analogous to, e.g., \cite{perlmutter2019understanding} and \cite{perlmutter:geoScatCompactManifold2020}, we will construct two versions of the scattering transform on $\mathcal{X}$. For both of these transforms, we will assume that we are given a signal $f$ defined on $\mathcal{X}$ and analogous to \eqref{eqn: basic coefs} will represent $f$  via a sequence of scattering coefficients obtained through alternating sequences of convolutions and pointwise nonlinearities followed by a final aggregation. In the first version, which we refer to as the \emph{windowed scattering transform}, the aggregation step is given by convolution against a low-pass filter that can be viewed as a local-averaging operator.  We also define a \emph{non-windowed scattering transform} where the final aggregation is computed via a global integration. Importantly, we note that the windowed scattering transform outputs a sequence of vectors (i.e., functions) whereas the non-windowed scattering transform outputs a sequence of scalars.  

We will  examine the invariance and equivariance properties of these representations and establish results analogous to \eqref{eqn: equivariance example} and \eqref{eqn: invariance example}. Towards this end, we let $\mathcal{G}$ be a group of bijections from $X$ to $X$ with proper structure. We let $\mathcal{G}$ act on $\mathbf{L}^2(\mathcal{X})$ by composition and on a Laplacian-type operator $\mathcal{L}$ by conjugation. Specifically, for $\zeta\in\mathcal{G}$ we define 
\begin{equation*}
    V_\zeta f(x) \rev{\coloneqq} f(\zeta^{-1}(x))\quad\text{and}\quad\mathcal{L}_\zeta\rev{\coloneqq} V_\zeta\circ\mathcal{L}\circ V_\zeta^{-1}.
\end{equation*}
In the case where $\mathcal{X}$ is a graph or a manifold, it is natural to take $\mathcal{G}$ to be the permutation group or the isometry group respectively. However, on an arbitrary measure space, it is not obvious what groups our representation should be invariant to. 

Perhaps the most natural idea would be the group of all bijections that preserves measures in the sense that $\mu(\zeta^{-1}(B))=\mu(B)$ for all $\zeta \in \mathcal{G}$ and $B\in\mathcal{F}$. Indeed, for the windowed scattering transform, our analysis will show that this condition is needed prove a result analogous to \eqref{eqn: equivariance example} establishing invariance  in the limit as the scale parameter tends to infinity.  However, it will not be needed in order to establish our other primary invariance and equivariance results. This is fortunate because conservation of measure is actually a stronger condition than it appears at first glance. For example, it does not hold when $\mathcal{X}$ is a graph, $\mathcal{G}$ is the permutation group, and different vertices are assigned different weights. Instead, our other invariance and equivariance  results will only require the weaker assumption that $V_\zeta$ is an isometry on $\mathbf{L}^2(\mathcal{X})$, i.e.,
$$ \langle V_\zeta f_1,V_\zeta f_2 \rangle_{\mathbf{L}^2(\mathcal{X})}=\langle f_1,f_2\rangle_{\mathbf{L}^2(\mathcal{X})}.$$
  One may verify that this condition holds for the permutation group on graphs for arbitrary choices of the measure.

In addition to significantly generalizing previous constructions of the geometric scattering transform, we also use the methods based on diffusion maps \cite{coifman:diffusionMaps2006} and Laplacian Eigenmaps \cite{belkin:laplaciansEigen2003,belkin2007convergence} to show that the  scattering transform on manifolds can be interpreted as the limit of the scattering transform on data-driven graphs. In short, if we have a collection of points $\{x_i\}_{i=0}^{N-1}\subseteq\mathbb{R}^D$ that lie on a $d$-dimensional manifold for some $d\ll D$, we will use a kernel to construct an affinity matrix $W$ which can be interpreted as the adjacency matrix of a weighted graph. We use this affinity matrix to construct a data-driven approximation of the \rev{Laplace-Beltrami} operator which we then use to construct an approximation of the windowed and non-windowed manifold scattering transforms. We will then prove theorems guaranteeing the rates of convergence of these methods as the number of sample points tends to infinity. To the best of our knowledge, this is the first attempt to prove such convergence guarantees for any neural-network-like architecture constructed from the \rev{Laplace-Beltrami} operator.

In summary, we
  provide a theoretically justified model for understanding spectral neural networks on measure spaces paralleling the role of the original scattering transform \cite{mallat:scattering2012} in understanding CNNs. Towards this end, we note that equivariance results similar to ours can likely be obtained for other networks such as the ones considered in \cite{wang2021stability} or \cite{zhang2021magnet} constructed through the spectrum of the appropriate Laplace operator. Similarly, our methods can likely be adapted to study the convergence of other spectral manifold neural networks.

 \subsection{Notation and Organization} Throughout, we will let 
 $\mathcal{X} = (X,\mathcal{F},\mu)$ be a measure space with set $X,$ $\sigma$-algebra $\mathcal{F}$, and measure $\mu$. We let $\mathcal{H}=\mathbf{L}^2(\mathcal{X})$ denote the Hilbert space of functions which are square integrable on $\mathcal{X}$ and for $f\in\mathcal{H}$ we will denote its norm by either $\|f\|_\mathcal{H}$ or $\|f\|_{\mathbf{L}^2(\mathcal{X})}$. Similarly, for $f,g\in\mathcal{H}$, we shall denote their inner product by $\langle f,g\rangle_{\mathcal{H}}$ or $\langle f,g\rangle_{\mathbf{L}^2(\mathcal{X})}$. If $T$ is an operator on $\mathcal{H}$, we will let $\|T\|_\mathcal{H}$ denote its operator norm. If $\mathbf{x}$ and $\mathbf{y}\in\mathbb{R}^N$ are vectors, we shall use $\|\mathbf{x}\|_2$ and $\langle \mathbf{x},\mathbf{y}\rangle_2$ to denote their $\ell^2$-norm and inner product. Similarly, if $A$ is a matrix, we will let $\|A\|_2$ denote its operator norm on $\ell^2$.
  We shall let $\mathcal{L}$ be a positive semidefinite, self-adjoint operator on $\mathcal{H}$ and denote its eigenfunctions and eigenvalues by $\varphi_k$ and $\lambda_k$ for $k$ in some at most countable indexing set $\mathcal{I}.$ If $\{f_j\}_{j\in\mathcal{J}}$ is an at most countable collection of elements in $\mathcal{H}$, we shall define $\|\{f_j\}_{j\in\mathcal{J}}\|_{\ell^2(\mathcal{H})}$ by
  $$\|\{f_j\}_{j\in\mathcal{J}}\|^2_{\ell^2(\mathcal{H})} \rev{\coloneqq}\sum_{j\in\mathcal{J}} \|f\|_{\mathcal{H}}^2.$$ We shall let $\mathcal{G}$ denote a group of bijections $X\rightarrow X$, and for $\zeta\in\mathcal{G}$, we shall let $V_\zeta$ denote the operator defined by $V_\zeta f(x)=f(\zeta^{-1}(x))$.
  
  Our construction of the scattering transform will be based on a collection of wavelets $\mathcal{W}\rev{\coloneqq}\{W_j\}_{j\in\mathcal{J}}\cup \{A\}$ where $\mathcal{J}$ is an at most countable indexing set. We will let $p\rev{\coloneqq}(j_1,\ldots, j_m)$ denote a scattering path of length $m$, and for $f\in\mathcal{H}$ let $S[p]f$ and $\overline{S}[p]f$ denote corresponding windowed and non-windowed scattering coefficients. We shall let $\{H_t\}_{t\geq 0}$ denote a semigroup of operators defined on $\mathcal{H}$ defined in terms of a spectral function $g:[0,\infty)\rightarrow [0,\infty)$. When notationally convenient, we will write $H^t$ instead of $H_t$. In Section \ref{sec: numerical methods}, we will consider finite subsets $X_N\subseteq X$ of cardinality $N$ and let $\mathcal{X}_N$ be a corresponding measure space. In this setting, we will denote objects corresponding to $\mathcal{X}_N$ with either a subscript or superscript $N$. 
  
  The rest of this paper is organized as follows. In Section \ref{sec: def}, we will define convolution, the wavelet transform, and the scattering transform on a measure space $\mathcal{X}$. In Section \ref{sec: examples}, we will discuss examples of measure spaces included in our framework, some of which have been considered in previous work on the scattering transform and some which have not. In Section \ref{sec: basic properties}, we will establish fundamental continuity and invariance properties of the measure space scattering transform and in Section \ref{sec: stability} we will consider stability to perturbations. In Section \ref{sec: numerical methods}, we will introduce numerical methods for implementing the scattering transform in the case where $\mathcal{X}$ is a manifold, but one only has access to $\mathcal{X}$ through a finite number of samples. We will also prove the convergence of these methods as the number of sample points tends to infinity. In Section \ref{sec: results}, we will present numerical experiments on both synthetic data and on real-world biomedical data before providing a brief conclusion in Section \ref{sec: conclusion}.
  
\section{Definitions}\label{sec: def}

In this section, we first define convolution and wavelets on a measure space $\mathcal{X}$ and then use these wavelets to define the measure space scattering transform.

Let $\mathcal{X} = (X,\mathcal{F},\mu)$ be a measure space with set $X,$ $\sigma$-algebra $\mathcal{F}$, and measure $\mu$. Let $\text{vol}(\mathcal{X})\coloneqq\mu(X)$, and let $\mathcal{H}=\mathbf{L}^2(\mathcal{X})$ denote the Hilbert space of measurable functions such that 
%
\begin{equation*}
\|f\|_{\mathcal{H}}^2\coloneqq\|f\|_{\mathbf{L}^2(\mathcal{X})}^2 \coloneqq 
\int_X |f|^2 d\mu<\infty. 
\end{equation*}
Let $\mathcal{L}$ be a self-adjoint and positive semidefinite operator on $\mathcal{H},$ and let $\mathcal{I}$ be an at most countable  set of nonnegative integers. Without loss of generality, we assume either $\mathcal{I}$ is the natural numbers $\mathbb{N}\cup\{0\}$ or that $\mathcal{I}=\{0,\ldots,N-1\}$ for some $N\in\mathbb{N}$. We assume that there is a collection of functions $\{\varphi_k\}_{k\in\mathcal{I}}\subset \mathcal{H}$ such that $\mathcal{L} \varphi_k = \lambda_k \varphi_k,$ with 
$\lambda_0=0<\lambda_1$ and $\lambda_k\leq \lambda_{k+1}$ for $k\geq 1$.
We also assume that $\{\varphi_k\}_{k\in\mathcal{I}}$ forms an orthonormal basis for $\mathcal{H}$.

\subsection{Convolution and Wavelet Transforms}\label{sec: wavelets}
For $f\in\mathcal{H}$, we define its generalized Fourier coefficients $\widehat{f}(k),  k\in\mathcal{I}$, by
\begin{equation*}
\widehat{f}(k) \coloneqq \langle f, \varphi_k\rangle_{\mathcal{H}}. 
\end{equation*}
Since $\{\varphi_k\}_{k\in\mathcal{I}}$ is an orthonormal basis, we obtain the generalized Fourier series 
\begin{equation*}
f = \sum_{k\in\mathcal{I}} \widehat{f}(k)\varphi_k,
\end{equation*}
where the convergence is in the $\mathbf{L}^2(\mathcal{X})$ sense if $\mathcal{I}$ is infinite.
In the case when $\mathcal{X}$ is the unit circle, it is well known that convolution corresponds to multiplication in the Fourier domain. Therefore, for any $h\in\mathcal{H}$, we define a convolution operator $T_h$ by 
\begin{equation}\label{eqn: fourier multiplication}
T_h f \coloneqq h\star f \coloneqq \sum_{k\in\mathcal{I}} \widehat{h}(k)\widehat{f}(k)\varphi_k.
\end{equation}
One may verify that for any $n\geq 0$ we have  \begin{equation}\label{eqn: power}
(T_h)^n f =  \sum_{k\in\mathcal{I}} \widehat{h}(k)^n\widehat{f}(k)\varphi_k.
\end{equation}
Therefore, if $\widehat{h}(k)$ is nonnegative for all $k$ we may define, for $t\in\mathbb{R}$, $T_h^{t}$ by
\begin{equation}\label{eqn: root}
T_h^{t} f \coloneqq  \sum_{k\in\mathcal{I}} \widehat{h}(k)^{t}\widehat{f}(k)\varphi_k.
\end{equation}
In the case where $t=1/2,$ we note that we have $T_h^{1/2}T_h^{1/2}=T_h$. Therefore, we will refer to $T_h^{1/2}$ as the square root of $T_h$.

We will use this notion of spectral convolution to construct a diffusion operator $H$. 
To do this, we let $g:[0,\infty)\rightarrow [0,\infty)$ be a nonnegative and nonincreasing function with \begin{equation}\label{eqn: g decreasing} g(0)=1\ \ \text{ and } g(t)<1 \ \ \text{ for all }t>0.\end{equation} 
%
For $t\geq 0,$ we  define $H^t$ to be the operator corresponding to convolution against $\sum_{k\in\mathcal{I}} g(\lambda_k)^t\varphi_k$, i.e.,
\begin{equation}\label{eqn: h}
H^tf \rev{\coloneqq}  \sum_{k\in\mathcal{I}} g(\lambda_k)^t\widehat{f}(k)\varphi_k.
\end{equation}
We note that by construction, $\{H^t\}_{t\geq 0}$ forms a semigroup since $H^tH^s=H^{t+s}$ and $H^0=\text{Id}$ is the identity operator. We also note that $H^t=(H^1)^t$, where the exponents are interpretted as in  
\eqref{eqn: power} 
and \eqref{eqn: root}. Motivated by the interpretation of $t$ as an exponent, we will occasionally write $H$ in place of $H^1$ when convenient. 

As our primary example, we will take $g(\lambda)=e^{-\lambda}$, in which case,  one may verify that, 
for sufficiently well-behaved functions,  $u_f(x,t):=H^tf(x)$ satisfies the heat equation
$$
\partial_t u_f=-\mathcal{L}_xu_f,\quad \lim_{t\rightarrow 0} u(t,x)=f(x),
$$
since we may compute 
\begin{align}\nonumber
    \partial_t H^tf(x) &= \partial_t \sum_{k\in\mathcal{I}} e^{-\lambda_kt} \widehat{f}(k)\varphi_k(x)\\
    &=  \sum_{k\in\mathcal{I}} -\lambda_k e^{-\lambda_kt} \widehat{f}(k)\varphi_k(x)\nonumber\\
    &=-\mathcal{L}_x H^tf(x)\label{eqn: heat differentiation}.
\end{align}
Therefore, in this case, $\{H^t\}_{t\geq 0}$ is referred to as the \emph{heat semigroup} and $t$ is referred to as the \emph{diffusion time}. 

\begin{remark}\label{rem: independent of eigenbasis}
The definition of $H$ does not depend on the choice of  eigenbasis, even when some eigenvalues have multiplictiy greater than one. To see this, let $\Lambda$ be the set of distinct eigenvalues of $\mathcal{L}$ and note that
\begin{equation*}
    Hf=\sum_{k\in\mathcal{I}} g(\lambda_k)\widehat{f}(k)\varphi_k =\sum_{\lambda\in\Lambda} g(\lambda)\sum_{k:\lambda_k=\lambda}\widehat{f}(k)\varphi_k=\sum_{\lambda\in\Lambda} g(\lambda) \pi_\lambda(f),
\end{equation*}
where, for $\lambda\in\Lambda,$ $\pi_\lambda$ is the operator which projects a function onto the eigenspace corresponding to $\lambda$. \end{remark}

Given this diffusion operator we define the wavelet transform  \begin{equation}\label{eqn: diffusion wavelets}
\mathcal{W}_Jf \coloneqq \{W_jf\}_{j=0}^J\cup \{A_Jf\},
\end{equation}
where $W_0\coloneqq\text{Id}-H^1,$ $A_{J}\coloneqq H^{2^J}$, and for $1\leq j \leq J$
\begin{equation*}
W_j \coloneqq H^{2^{j-1}}-H^{2^{j}}.
\end{equation*}
The wavelets aim to capture the geometry of $\mathcal{X}$ at different scales. In particular, the $W_j$ track changes between the geometry at different diffusion times. The operator $A_J$ performs a localized averaging operation and may be interpreted as a low-pass filter. Our construction uses a minimal time scale of $1$ for simplicity and notational  convenience. However, if one wishes to obtain wavelets which are sensitive to smaller time scales, they may simply change the spectral function $g$. For example, if $g_1(\lambda)=e^{-\lambda}$ and $g_2(\lambda)=e^{-\lambda/2}$ then the associated diffusion operators would satisfy $H_2^1 = H_1^{1/2}$. 

The following result shows that $\mathcal{W}_J$ is a nonexpansive frame on $\mathcal{H}$. Its proof is identical to the proof of Proposition 2.2 of \cite{perlmutter2019understanding}. For completeness, we give full details in Appendix \ref{sec: proof of frame}.
\begin{proposition}\label{prop: frame}
There exists a universal constant $c>0$ such that for all $f\in\mathcal{H}$
\begin{equation*}
c\|f\|^2_\mathcal{H} \leq 
\|\mathcal{W}_Jf\|^2_{\ell^2(\mathcal{H})}\coloneqq \sum_{j=0}^J \|W_jf\|^2_\mathcal{H}+ \|A_Jf\|_\mathcal{H}^2\leq \|f\|^2_\mathcal{H}.
\end{equation*}
\end{proposition}
\begin{remark}\label{rem: modify}
If we instead define our wavelets  by  $W'_0=\sqrt{Id-H}$, $W'_j = \sqrt{H^{2^{j-1}}-H^{2^{j}}}$ for $1\leq j \leq J,$ and $A'_J =\sqrt{H^{2^J}}$, we can obtain a similar result for $\mathcal{W}'_Jf=\{W_j'\}_{j=0}^J\cup \{A'_Jf\}$ but with $c=1$, so that the wavelet transform is norm-preserving, 
i.e.,
\begin{equation*}
 \sum_{j=0}^J \|W'_jf\|^2_\mathcal{H}+ \|A'_Jf\|_\mathcal{H}^2= \|f\|_\mathcal{H}.
\end{equation*}
The proof is identical to the proof of Proposition 2.1 of \cite{perlmutter2019understanding}.
\end{remark}

\subsection{The Scattering Transform} \label{sec: scat def}
In this section, we will construct the scattering transform as a multilayered architecture built off of a filter bank $\mathcal{W}$. 
For the sake of generality, we will not require our $\mathcal{W}$ to be the diffusion wavelets constructed in the previous subsection.
Instead, we let $\mathcal{J}$ be an arbitrary countable indexing set and assume 
\begin{equation*}
    \mathcal{W}=\{W_j\}_{j\in\mathcal{J}}\cup\{A\}
\end{equation*}
is any collection of operators such that
\begin{equation} 
   c\|f\|^2_{\mathcal{H}}\leq \|\mathcal{W} f\|_{\ell^2(\mathcal{H})}^2=   \sum_{j\in\mathcal{J}}\|W_jf\|^2_{\mathcal{H}}+ \|Af\|^2_{\mathcal{H}}\leq \|f\|^2_{\mathcal{H}}\label{eqn: frameAB}
\end{equation}
for some $c>0$. This generality is motivated both by the fact that several different versions of the graph scattering transform \cite{zou:graphCNNScat2018,gama:diffScatGraphs2018,gao:graphScat2018} have used different wavelet constructions and also by various works which have constructed versions of the Euclidean scattering transform using generalized, non-wavelet filter banks  \cite{czaja2019analysis,grohs:cnnCartoonFcns2016, wiatowski:frameScat2015, wiatowski:mathTheoryCNN2018}.
 Here, we note that the letter of $A$ is chosen because we typically interpret $A$ as an averaging operator analogous to the low-pass operator $A_J$ considered in \eqref{eqn: diffusion wavelets}. However, we emphasize that this is merely suggestive notation. 

The scattering transform consists of an alternating sequence of linear filterings (typically wavelet transforms) and nonlinear activations. Towards this end, we let $\sigma$ be an nonlinear function defined on either $\mathbb{R}$ or $\mathbb{C}$ such that the real part of $\sigma(x)$ is nonnegative and $\sigma$ is non-expansive in the sense that $|\sigma(x)-\sigma(y)|\leq |x-y|$. In a slight abuse of notation let  
  $\sigma:\mathcal{H}\rightarrow\mathcal{H}$ also denote the operator defined by $(\sigma f)(x)\rev{\coloneqq}\sigma(f(x))$. We note that in the case where admissible choices of $\sigma$ include the absolute value function which is commonly used in papers concerning the scattering transform,  the rectified linear unit (ReLU) which is commonly used in other neural network architectures, and the complex version of ReLU considered in \cite{zhang2021magnet}.
 
Given $\mathcal{W}_J$ and $\sigma$, we define the {\em windowed} scattering transform $\bS:\mathcal{H}\rightarrow\ell^2(\mathcal{H})$ by
\begin{equation*}
    S f \rev{\coloneqq} \{S[\pathvar]f: m\geq 0, \pathvar\rev{\coloneqq}(j_1,\ldots,j_m) \in\mathcal{J}^{m}\},
\end{equation*}
where $\mathcal{J}^m$ is the $m$-fold Cartesian product of $\mathcal{J}$, and the windowed scattering coefficients $S[\pathvar]$ corresponding to the path $p=(j_1,\ldots,j_m)\in\mathcal{J}^{m}$ are defined by 
\begin{equation*}
    S[\pathvar]f\rev{\coloneqq}AU[\pathvar]f, \quad U[\pathvar]f\rev{\coloneqq}\sigma W_{j_m}
\ldots \sigma W_{j_1}f\end{equation*}
for $m\geq 1,$ and when $m=0$ and $p_e$ is the ``empty path", we declare that \begin{equation}\label{eqn: empty path}S[p_e]f\rev{\coloneqq}Af.\end{equation}  We also define an operator $U$ by
\begin{equation}\label{eqn: defU}
    \bU f \rev{\coloneqq} \{\bU[\pathvar] f: m\geq 0, \pathvar=(j_1,\ldots,j_m) \in\mathcal{J}^{m}\}
\end{equation}
and a {\em non-windowed} scattering transform by 
\begin{equation*}
   \overline{\bS} f \rev{\coloneqq} \{\overline{\bS}[\pathvar]f: m\geq 0, \pathvar=(j_1,\ldots,j_m) \in\mathcal{J}^{m}\},
   \end{equation*}
   where the non-windowed scattering coefficients are given by 
   \begin{equation*}
   \overline{\bS}[\pathvar]f\rev{\coloneqq}\left|\int_X (\bU[\pathvar]f) \bar{\varphi}_0d\mu\right|=\left|\langle \bU[p]f,\varphi_0\rangle_\mathcal{H}\right|.
    \end{equation*}
In the case where $\mathcal{J}=\{0,\ldots,J\}$ and $S$ is constructed from the diffusion wavelets  $\mathcal{W}_J$ defined in \eqref{eqn: diffusion wavelets}, we will occasionally write $S_J[p]f$ in place of $S[p]f$  if we want to emphasize the dependency of the parameter $J$. We note that the primary difference between the windowed and non-windowed scattering transform is the use of the localized averaging operator $A$ rather than a global integration against $\varphi_0$. Indeed, the term ``windowed" refers to the idea that an average is computed within a neighborhood of each point. In particular, the windowed scattering transform should not be confused with constructions, such as those appearing in \cite{czaja2019analysis}, which construct scattering transforms (on $\mathbb{R}^N$) using Gabor filters.

The following result relates the non-windowed scattering transform $\overline{S}$ to the limit of the windowed scattering transform $S_J$ as $J\rightarrow \infty$. In particular, if $\mathcal{L}$ is either the \rev{Laplace-Beltrami} operator on a manifold or the unnormalized Laplacian on a graph, then $\varphi_0$ is constant. Therefore, the following result shows that the windowed scattering coefficients $\bS_J[\pathvar] f(x)$ converge to a constant multiple of $\overline\bS[\pathvar] f$.  Please see Appendix \ref{sec: PJlimit proof} for a proof.
\begin{proposition}\label{prop: Jlimit}
Let $\bS_J$ be the windowed scattering transform build on top of the diffuion wavelet frame $\mathcal{W}_J$ defined in \eqref{eqn: diffusion wavelets} and assume $\lambda_1>0$. Then for all $f\in\mathcal H$, and every path $\pathvar$ we have 
\begin{equation}\label{eqn: limit prop}
\lim_{J\rightarrow\infty} \| |\bS_J[\pathvar] f| - (\overline\bS[\pathvar] f) |\varphi_0|\|_\mathcal{H}=0.
\end{equation}
\end{proposition}
\begin{remark}
In the case where $\mathcal{L}=-\nabla\cdot\nabla$ is the \rev{Laplace-Beltrami} operator on a manifold or the unnormalized graph Laplacian, one may take $\varphi_0$ to be the  constant function $\varphi_0(x)=\text{vol}(\mathcal{X})^{1/2}$ and it is known that the associated heat semigroup $\{e^{-t\mathcal{L}}\}_{t\geq0}$ is positivity preserving (see, e.g., \cite{DAVIES1984335,Keller2021}). Therefore $S_J[p]f$ will be nonnegative, and \eqref{eqn: limit prop} implies \begin{equation*}
\frac{1}{\text{vol}(\mathcal{X})^{1/2}}\lim_{J\rightarrow\infty}  \bS_J[\pathvar] f = (\overline\bS[\pathvar] f).
\end{equation*}
 
\end{remark}

\section{Examples and Relationship to prior work}\label{sec: examples}

Several versions of the scattering transform for graphs \cite{gao:graphScat2018,gama:diffScatGraphs2018,zou:graphCNNScat2018} and smooth Riemannian manifolds without boundary \cite{perlmutter:geoScatCompactManifold2020,mcewen2021scattering} have been introduced in previous work. In this section, we will discuss how these constructions relate to our framework. We also discuss several other examples of measure spaces included in our theory that have not been previously considered in the scattering transform literature. Most of the techniques used to prove our theoretical results in Sections \ref{sec: def}, \ref{sec: basic properties}, and \ref{sec: stability} are natural generalization of the techniques used in these previous works on geometric (and even Euclidean) scattering. Indeed, our definitions were developed by carefully examining these papers and designing our framework in such a way that these techniques could be extended to more general settings. Additionally, we note that our convergence results (Theorems \ref{thm: heat kernel discretization}, \ref{thm: wavelet discretization}, \ref{thm: discretize U}, \ref{thm: convergence windowed}, and \ref{thm: convergence nonwindowed} stated in Section \ref{sec: numerical methods}), do not, to the best of our knowledge, have direct analogs in any previous work on the scattering transform. 

\subsection{Undirected, Unsigned Graphs}\label{sec: graph scat}
 Several  works have introduced different definitions of the {\em graph scattering transform}. These works differ primarily in two respects, i) the definition of the wavelets and ii) whether they use a windowed or unwindowed version of the graph scattering transform. Below, we explain how these constructions are related to our framework. Throughout this subsection, we let $G=(V,E,W)$ be a weighted graph with weighted adjacency matrix $A$ and weighted degree matrix $D=\text{diag}(\mathbf{d})$, where $\mathbf{d}$ denotes the degree vector. Notably, all of the work discussed in this subsection focuses on undirected, unsigned graphs, i.e., graphs for which $A$ is symmetric and has nonnegative entries.

In \cite{gama:diffScatGraphs2018}, the authors define wavelets of the form $T^{2^{j-1}}-T^{2^j}$, where $T\coloneqq\frac{1}{2}(I+D^{-1/2}AD^{-1/2})$,  for $1\leq j \leq J$ for some maximal scale $J$\footnote{\cite{gama:diffScatGraphs2018} also uses a wavelet $I-T$ for when $j=0$.}. In order to obtain these wavelets from our framework, we may choose $\mu$ to be the uniform measure which gives  weight $1$ to each vertex, let $\mathcal{L}$ be the symmetric normalized Laplacian $L_{\text{sym}}=I-D^{-1/2}AD^{-1/2}$ and choose \begin{equation}\label{eqn: gamma g} g(\lambda)=\begin{cases}
1-\lambda/2 &\text{if }0\leq \lambda \leq 2\\
0&\text{otherwise}
\end{cases}
\end{equation}
in \eqref{eqn: h}. The authors of \cite{gama:diffScatGraphs2018} are primarily concerned with graph level tasks and therefore use a \rev{non-windowed} version of the scattering transform.

In \cite{gao:graphScat2018}, the authors use wavelets of the form $P^{2^{j-1}}-P^{2^j}$ where $$P\coloneqq\frac{1}{2}(I+AD^{-1})=D^{1/2}TD^{-1/2}$$ is the lazy random walk matrix, i.e., the matrix whose entries describe the transition probabilities of a lazy random walk on the graph.  In order to incorporate these wavelets into our framework, we define $\mu$ by the rule $\mu(\{v_i\})=\frac{1}{\mathbf{d}_i}$, where $v_i\in V$ and $\mathbf{d}_i$ is the $i$-th entry of $\mathbf{d}$, and choose $\mathcal{L}$ to be the random walk normalized Laplacian $L_{\text{RW}}=I-AD^{-1}=D^{1/2}L_{\text{sym}}D^{-1/2}$. Using the fact that  $L_{\text{RW}}$ is similar to $L_{\text{sym}}$, one may imitate the proof of Lemma 1.1 of \cite{perlmutter2019understanding} to verify that $L_{\text{RW}}$ is self adjoint for this choice of $\mu$. Therefore, one can recover the wavelets from \cite{gao:graphScat2018} by again choosing $g$ as in \eqref{eqn: gamma g}. Similar to \cite{gama:diffScatGraphs2018}, the authors of \cite{gao:graphScat2018} are primarily concerned with graph level tasks and therefore also use a \rev{non-windowed} version of the scattering transform.

The wavelets used in \cite{gama:diffScatGraphs2018} and \cite{gao:graphScat2018} are based on  \cite{coifman:diffWavelets2006}. By contrast,  \cite{zou:graphCNNScat2018} uses a different wavelet construction based on \cite{hammond:graphWavelets2011}. Rather than using a single spectral function $g$, a family of wavelets $\{\psi_j\}_{j\in\mathbb{Z}}$ is constructed on the real line and used to define wavelet convolution with respect to the spectral decomposition of  the unnormalized Laplacian $L_{\text{un}}\rev{\coloneqq}D-A$. In our framework, this corresponds to defining 
$$ 
W_jf = \sum_{k=0}^{N-1} \psi_j(\lambda_k) \widehat{f}(k) \varphi_k,
$$
where $N$ is the number of vertices and $\{(\lambda_k,\varphi_k)\}_{k=0}^{N-1}$ are  eigenpairs of  $L_\text{un}$. 
We note that in Theorem \ref{thm: nonexpansive} and in Section \ref{sec: scat stab} we do not assume that our wavelets are constructed as in \eqref{eqn: diffusion wavelets} and therefore some of our results may be applied to the scattering transform constructed from these wavelets as well. We also note that analogs of many of our other results were previously  established in \cite{zou:graphCNNScat2018} in this case. 

\subsection{Signed Graphs, Directed Graphs, Hypergraphs, and Simplicial Complexes}
A directed graph is a graph where the adjacency matrix is not symmetric. This makes it non-obvious  how to apply spectral methods since naive extensions of the (unnormalized or normalized) graph Laplacian are in general not diagonalizable on the standard unweighted inner product space. Nevertheless, directed graphs are a natural model for many phenomena such as email networks or traffic networks, and so 
there have been several attempts to define directed graph Laplacians which are either real symmetric or complex Hermitian. 

In \cite{chung2005laplacians}, the author defines a directed Laplacian via a non-reversible Markov chain and provides an extensive analysis of these matrices' spectral properties. This matrix was later used as the basis for  spectral directed graph neural networks  in \cite{ma2019spectral} and \cite{tong2020digraph}. 
An alternative approach, dating back to at least \cite{lieb1993fluxes}, is to construct a complex Hermitian adjacency matrix known as the magnetic Laplacian, which may be viewed as a special case of the graph connection Laplacian (see, e.g., \cite{singer2011orientability,singer2012vector}). This matrix represents the undirected geometry of the graph in the magnitude of its entries and incorporates directional information in the phases. It has been studied by the graph signal processing community \cite{furutani2019graph} and also applied to numerous data science applications such as clustering and community detection \cite{fanuel2017magnetic,cloninger2017note,f2020characterization,fanuel2018magnetic}. Recently, \cite{zhang2021magnet} showed that the Magnetic Laplacian could be effectively incorporated into a graph neural network.
Analogously, there has also been work \cite{cucuringu2021regularized} using spectral clustering methods on signed graphs, i.e., graphs with both positive ``friend" edges and negative ``enemy" edges using methods based on signed Laplacians. Very recently, \cite{sigma,he2022msgnn,singh2022signed,ko2023spectral} proposes various signed magnetic Laplacian and uses these matrices to construct a signed and directed graph neural network. Similarly, \cite{feng2019hypergraph} has proposed a neural network on hypergraphs (graphs where generalized edges may consist of more than two nodes) based on a generalized Laplacian.

In this paper, we are agnostic to the question of which Laplacian is the best for signed and/or directed graphs. We merely note that our theory applies to all of the Laplacians discussed here and any of these Laplacians can  be used to define scattering transforms on signed and/or directed graphs. 
Additionally, we note that there has been work developing spectral clustering methods on hypergraphs using matrices which do not fit within our framework because they are not self-adjoint (see \cite{chodrow2022nonbacktracking} and the references within). Developing variants of our theory which utilize these operators would be an interesting direction of future work.

\rev{We also note the recent work \cite{saito2022multiscale}, which uses the Hodge Laplacian to construct wavelets on simplicial complices. These wavelets were then used as a basis for simplicial complex scattering transforms in the follow up work \cite{saito2023multiscale}. Furthermore, we also note several papers which have applied Hodge Laplacians to directed graphs \cite{lim2020hodge,schaub2021signal}. We remark that in some of these cases, the condition that $0=\lambda_0<\lambda_1$ may not hold. In these settings, both $U$ and the windowed scattering transform are still well-defined and most of our analysis carries through unchanged. The definition of the windowed scattering transform, however, should be modified to either be projection onto the $0$-eigenspace, in the case where $0$ has multiplicity, or to be defined via a global summation, i.e., $\overline{S}[p]f\coloneqq\sum_{v\in V} U[p]f(v)$, in the case where $0<\lambda_1$. However, it is important to note that in the latter case it is no longer true, in general, that the non-windowed scattering transform is the limit of the windowed scattering transform. (It follows from the proof of Proposition \ref{prop: Jlimit} that the windowed scattering transform converges to zero in this case.)}



\subsection{Manifolds}
In \cite{perlmutter:geoScatCompactManifold2020}, the authors constructed a scattering transform for smooth and compact manifolds without boundary via the spectral decomposition of the \rev{Laplace-Beltrami} operator. If we choose $g(\lambda)=e^{-\lambda}$, the wavelets proposed in Section \ref{sec: wavelets} are a minor variation of those considered there. Indeed, if we add an additional square root term as discussed in Remark \ref{rem: modify}, then the wavelets from Section \ref{sec: wavelets} will exactly coincide with those considered in \cite{perlmutter:geoScatCompactManifold2020}. We also note \cite{mcewen2021scattering} which replaced with wavelets used in \cite{perlmutter:geoScatCompactManifold2020} with wavelets optimized for fast computation on the sphere. As with the wavelets considered in \cite{zou:graphCNNScat2018}, these wavelets are not a special case of the wavelets constructed in Section \ref{sec: wavelets}. However, it is likely that one could derive analogs of most of our results for this version of the scattering transform as well. 

Our framework can also be applied to other interesting setups not considered in previous work on the scattering transform. For example, when the \rev{Laplace-Beltrami} operator is equipped with suitable boundary conditions, our methods may also be applied to manifolds with boundary. Moreover, it may also be applied to weighted Laplacians such as those considered in \cite{hein2007graph} or \cite{hoffmann2022spectral}  or anisotropic Laplacians such as those applied to two-dimensional surfaces in \cite{boscaini2016learning}.







\section{Continuity and Invariance}\label{sec: basic properties}
In this section, we establish the fundamental continuity and invariance properties of the windowed and \rev{non-windowed} scattering transform. In Section \ref{sec: nonexpansive}, we show that both the windowed and \rev{non-windowed} scattering transforms are Lipschitz continuous with respect to additive noise, and then, in Section \ref{sec: inveq}, we establish invariance and equivariance properties for the scattering transforms under certain group actions.
\subsection{Lipschitz Continuity with Respect to Additive Noise}\label{sec: nonexpansive}
The following two theorems show that 
the windowed and non-windowed scattering transforms are Lipschitz continuous  on $\mathcal{H}.$ 
Our first result,
Theorem \ref{thm: nonexpansive}, shows that the windowed scattering is nonexpansive. Its proof is  based on analogous theorems in works such as \cite{mallat:scattering2012}, \cite{perlmutter:geoScatCompactManifold2020}, \cite{perlmutter2019understanding}, and \cite{zou:graphCNNScat2018} which consider specific measure spaces.

\begin{theorem}\label{thm: nonexpansive}
Let $\bS$ be the scattering transform built on top of the wavelet frame $\mathcal{W}$. Then the scattering transform is a nonexpansive operator from $\mathcal{H}\rightarrow \ell^2(\mathcal{H}),$ i.e., for all $f_1,f_2\in\mathcal{H},$
\begin{equation*}
    \|\bS f_1-\bS f_2\|_{\ell^2(\mathcal{H})}\leq  \|f_1-f_2\|_{\mathcal{H}}.
\end{equation*}
\end{theorem}

Theorem \ref{thm: nonexpansive nonwindow} shows that the non-windowed scattering transform is Lipschitz continuous on $\mathcal{H}$. Its proof is a generalization of Theorem 3.2 of \cite{perlmutter2019understanding}. Notably,  unlike Theorem \ref{thm: nonexpansive}, Theorem \ref{thm: nonexpansive nonwindow} requires that we use the wavelets defined in \eqref{eqn: diffusion wavelets}.

\begin{theorem}\label{thm: nonexpansive nonwindow} 
Let $\bS_J$ be the scattering transform built on top of diffusion wavelets $\mathcal{W}_J$ defined in \eqref{eqn: diffusion wavelets}. Assume that  $\inf_x|\varphi_0(x)|>0$ and $\lambda_1>0$. Then 

\begin{align*}\|\overline{\bS}f_1-\overline{\bS}f_2\|_2^2\leq \frac{1}{\min_x|\varphi_0(x)|^2\text{vol}(\mathcal{X})}\|f_1-f_2\|_\mathcal{H}.
\end{align*}
\end{theorem}
\noindent For proofs of Theorems \ref{thm: nonexpansive} and \ref{thm: nonexpansive nonwindow}, please see Appendix \ref{sec: proof nonexpansive}.

\begin{remark}
In many cases of interest such as when either i) $\mathcal{X}$ is a compact Remannian manifold without boundary and $\mathcal{L}$ is the \rev{Laplace-Beltrami} operator or ii) $\mathcal{X}$ is an unweighted and undirected graph and $\mathcal{L}$ is the unnormalized graph Laplacian, we have that $\varphi_0(x)$ is constant and therefore $\frac{1}{\min_x|\varphi_0(x)|^2\text{vol}(\mathcal{X})}=1$.
\end{remark}
\begin{remark}
Inspecting the proof, we see that results analogous to Theorem \ref{thm: nonexpansive nonwindow} can be derived for the non-windowed scattering transfrom built upon other frames as long as one is able to establish a result similar to Proposition \ref{prop: Jlimit}.
\end{remark}

\subsection{Invariance and Equivariance}\label{sec: inveq}
Let $\mathcal{G}$ be a collection of bijections $\zeta:X\rightarrow X$ which form a group under composition. For $\zeta\in\mathcal{G}$, let 
\begin{equation}\label{eqn: Xzeta}\mathcal{X}_\zeta \coloneqq (X,\mathcal{F}_\zeta,\mu_\zeta)\end{equation} be the measure space with $\sigma$-algebra $\mathcal{F}_\zeta$ and measure $\mu_\zeta$ given by
\begin{equation*}
    \mathcal{F}_\zeta\coloneqq\{\zeta^{-1}(B):B\in\mathcal{F}\},\quad
    \mu_\zeta(B)\coloneqq\mu(\zeta^{-1}(B)).
\end{equation*}
We let $\mathcal{G}$ act on $\mathcal{H}$ by function composition and we let it act on the set of linear operators by conjugation. 
Let $\mathcal{H}^{(\zeta)}$ be the Hilbert space of functions on $\mathcal{X}_\zeta$ which are square integrable with respect to $\mu_\zeta$. Let $V_\zeta: \mathcal{H}\rightarrow \mathcal{H}^{(\zeta)}$ denote the operator $V_\zeta f:=f\circ \zeta^{-1}$ and let  
$\mathcal{L}_\zeta$ denote the operator on $\mathcal{H}^{(\zeta)}$ defined by
\begin{equation*}
    \mathcal{L}_\zeta\rev{\coloneqq}V_\zeta\circ\mathcal{L}\circ V_\zeta^{-1}.
\end{equation*}
The following lemma relates the eigenfunctions of $\mathcal{L} $ and $\mathcal{L}_{\zeta}$.
\begin{lemma}\label{lem: eigenfunctions}
If $\varphi$ is an eigenfunction of $\mathcal{L}$ with $\mathcal{L}\varphi=\lambda\varphi,$ then $V_\zeta\varphi$ is an eigenfunction of $\mathcal{L}_\zeta$ and $\mathcal{L}_\zeta V_\zeta \varphi=\lambda V_\zeta\varphi$. 
\end{lemma}
\begin{proof} 
The proof is immediate from the definition:
\begin{equation*}\mathcal{L}_\zeta V_\zeta \varphi = V_\zeta\mathcal{L}V_\zeta^{-1} V_\zeta \varphi =V_\zeta\mathcal{L} \varphi=V_\zeta\lambda\varphi=\lambda V_\zeta\varphi.\qedhere  
\end{equation*}
\end{proof}

In the case where $\mathcal{X}$ is a graph or a manifold, the standard choice of $\mathcal{G}$ is the permutation group or the isometry group. The key desired property of this group is that the associated group action is an isometry from $\mathcal{H}$ to $\mathcal{H}^{(\zeta)}$. This motivates the following definition.

\begin{definition}\label{def: preserves inner products}
We say that $\mathcal{G}$ preserves inner products on $\mathcal{H}$ if for all $\zeta \in\mathcal{G}$ and all $f_1,f_2\in\mathcal{H}$ we have 
\begin{equation*}
    \langle  V_\zeta f_1,V_\zeta f_2\rangle_{\mathcal{H}^{(\zeta)}}=\langle f_1,f_2\rangle_\mathcal{H}.
\end{equation*}
\end{definition}
Importantly, we note that Definition \ref{def: preserves inner products} is satisfied both when $\mathcal{X}$ is a compact Riemannian manifold, $\mathcal{G}$ is the isometry group, and $\mu$ is the Riemannian volume and also  when $\mathcal{X}$ is a graph, $\mathcal{G}$ is the permutation group, and $\mu$ is \emph{any} measure, including both the uniform measure and measures which assign different weights to vertices depending upon their degrees.


\begin{lemma} Suppose $\mathcal{G}$ preserves inner products  on $\mathcal{H}$. Then, for all $\zeta\in\mathcal{G},$
$\mathcal{L}_\zeta$ is self-adjoint on $\mathcal{H}^{(\zeta)}$.
\end{lemma}
\begin{proof} Using the definition of 
$\mathcal{L}_\zeta$, the fact that $\mathcal{L}$ is self-adjoint on $\mathcal{H}$, and the assumption that $\mathcal{G}$ preserves inner products implies
\begin{align*}
\langle \mathcal{L}_\zeta f_1,f_2 \rangle_{\mathcal{H}^{(\zeta)}}&=
\langle V_\zeta\mathcal{L} V_\zeta^{-1} f_1,V_\zeta V_\zeta^{-1}f_2 \rangle_{\mathcal{H}^{(\zeta)}}
=\langle \mathcal{L} V_\zeta^{-1} f_1, V_\zeta^{-1}f_2 \rangle_{\mathcal{H}}\\
&=\langle  V_\zeta^{-1} f_1, \mathcal{L}V_\zeta^{-1}f_2 \rangle_{\mathcal{H}}
=\langle  V_\zeta^{-1} f_1, V_\zeta^{-1}V_\zeta\mathcal{L}V_\zeta^{-1}f_2 \rangle_{\mathcal{H}}\\
&=\langle  V_\zeta^{-1} f_1, V_\zeta^{-1}\mathcal{L}_\zeta f_2 \rangle_{\mathcal{H}}
=\langle   f_1, \mathcal{L}_\zeta f_2 \rangle_{\mathcal{H}^{(\zeta)}}.
\end{align*}
\end{proof}
Recall the operators $H^t:\mathcal{H}\rightarrow\mathcal{H}$
\begin{equation*}
H^tf =\sum_{k\in\mathcal{I}} g(\lambda_k)^t\widehat{f}(k)\varphi_k
=\sum_{k\in\mathcal{I}} g(\lambda_k)^t\langle f,\varphi_k\rangle_{\mathcal{H}}\varphi_k,
\end{equation*}
and define
\begin{equation*}
H^t_\zeta f \coloneqq\sum_{k\in\mathcal{I}} g(\lambda_k)^t\langle f,\varphi^{(\zeta)}_k\rangle_{\mathcal{H}^{(\zeta)}}\varphi^{(\zeta)}_k
\end{equation*}
to be the corresponding \rev{operator} on $\mathcal{H}^{(\zeta)},$ where $\{\varphi_k^{(\zeta)}\}_{k\in\mathcal{I}}$ is an orthonormal basis of eigenfunctions for $\mathcal{L}_\zeta.$ Let $W_j^{(\zeta)}$ and $A^{(\zeta)}$ denote wavelets and averaging operators on $\mathcal{H}^{(\zeta)}$, and let $U^{(\zeta)}$, $S^{(\zeta)}$, and $\overline{S^{(\zeta)}}$ be the analogs of $U$, $S$, and $\overline{S}$ on $\mathcal{H}^{(\zeta)}$. We observe that by Remark \ref{rem: independent of eigenbasis}, the definition of $H^t_\zeta$ does not depend on the choice of eigenbasis. Therefore, by Lemma \ref{lem: eigenfunctions}, we may assume without loss of generality that $\varphi^{(\zeta)}_k=V_\zeta \varphi_k$ and therefore that 
\begin{equation}\label{eqn: assumed form}
H^t_\zeta f=\sum_{k\in\mathcal{I}} g(\lambda_k)^t\langle f,V_\zeta\varphi_k\rangle_{\mathcal{H}^{(\zeta)}}V_\zeta\varphi_k.
\end{equation}
In light of \eqref{eqn: assumed form}, if $\mathcal{G}$ preserves inner products, we see that $H^t$ commutes with $V_\zeta$ in the sense that 
\begin{equation}\label{eqn: heat commmutes} H_\zeta^{t} V_\zeta f =V_\zeta H^{t}f\quad \text{for all }t\geq 0
\end{equation} 
since we may compute
\begin{align*}
    H_\zeta^{t} V_\zeta f &=
    \sum_{k\in\mathcal{I}}g(\lambda_k)^{t}\langle V_\zeta f, V_\zeta \varphi_k\rangle_{\mathcal{H}^{(\zeta)}} V_\zeta \varphi_k
    =\sum_{k\in\mathcal{I}}g(\lambda_k)^{t}\langle f,  \varphi_k\rangle_{\mathcal{H}} V_\zeta \varphi_k\\
    &=V_\zeta\sum_{k\in\mathcal{I}}g(\lambda_k)^{t}\langle f,  \varphi_k\rangle_{\mathcal{H}}  \varphi_k
    =V_\zeta H^{t}f.
\end{align*}
This readily leads to the following theorem which shows that the condition that $\mathcal{G}$ preserves inner products on $\mathcal{H}$ is sufficient to produce equivariance results for the wavelet transform and the windowed scattering transform analogous to \eqref{eqn: equivariance example} mentioned in the introduction. For a proof, please see Appendix \ref{sec: proof of equivariance}.

\begin{theorem}\label{thm: equivariance} Let $\mathcal{J}=\{0,\ldots,J\}$, and let $\mathcal{W}=\mathcal{W}_J$ be the diffusion wavelets constructed in \eqref{eqn: diffusion wavelets} and assume $\lambda_1>0$.
Then, if $\mathcal{G}$ preserves inner products, then $\mathcal{G}$ commutes with both the wavelet transform,  the operator $U$ defined in \eqref{eqn: defU} and the scattering transform. That is, for all $\zeta\in\mathcal{G}$, $f\in\mathcal{H}$ and $0\leq j \leq J$, we have  
\begin{align*}
W_j^{(\zeta)}V_\zeta f=V_\zeta W_jf,\ \ A^{(\zeta)}V_\zeta f=V_\zeta A f,\ \   
\bU^{(\zeta)} V_\zeta f=V_\zeta \bU f\ \ \text{and}\ \ \bS^{(\zeta)} V_\zeta f=V_\zeta\bS f. 
\end{align*}
\end{theorem}

\begin{remark}
Equations analogous to \eqref{eqn: heat commmutes} hold for any spectral filter of the form \eqref{eqn: fourier multiplication}, as long as $\widehat{h}(k)$ is a function of $\lambda_k$, i.e., $\widehat{h}(k)=\widetilde{h}(\lambda_k)$ for some function $\widetilde{h}$. Therefore, results similar to Theorem \ref{thm: equivariance} can be derived for any network constructed from such filters and pointwise nonlinearities $\sigma$. Additionally, analogous results can also be derived for the scattering transform built upon other geometric wavelet constructions. For example, the conclusions of  Proposition 4.1 of \cite{zou:graphCNNScat2018} are similar to those of Theorem \ref{thm: equivariance} above.
\end{remark}

Our next result shows that the non-windowed scattering transform $\overline{S}$ is fully invariant  under the assumption that $\mathcal{G}$ preserves inner products on $\mathcal{H}$. Importantly, we note that the windowed scattering transform is not in general invariant. Intuitively, this distinction arises from the fact that $\overline{S}$ is the composition of an equivariant operator $U$ together with a final global aggregation operator whereas $S$ uses a localized averaging operator $A$. 

\begin{theorem}\label{thm: invariance}Let $\mathcal{J}=\{0,\ldots,J\}$, and let $\mathcal{W}=\mathcal{W}_J$ be the diffusion wavelets constructed in \eqref{eqn: diffusion wavelets}. 
Assume $\mathcal{L}$ has a spectral gap, i.e., $\lambda_1>0$. Then,
if $\mathcal{G}$ preserves inner products, the non-windowed scattering transform is invariant to the action of $\mathcal{G}$, i.e.,
\begin{equation*}
    \overline{\bS^{(\zeta)}}V_\zeta f = \overline{\bS} f
\:\: \text{for all }\zeta\in\mathcal{G}\text{ and all }f\in\mathcal{H}.
\end{equation*}\end{theorem}

\begin{proof} Since $\lambda_1>0$, the eigenspace corresponding to $\lambda=0$ has dimension one. Therefore, 
$\varphi_0^{(\zeta)}=cV_\zeta \varphi_0,$ for some constant $c$ with $|c|=1,$ and so 
\begin{align*}
    \overline{\bS^{(\zeta)}}[\pathvar]V_\zeta f &= |\langle \bU^{(\zeta)}[\pathvar]V_\zeta f,cV_\zeta \varphi_0\rangle_{\mathcal{H}^{(\zeta)}}|\\&=|\langle V_\zeta\bU[\pathvar] f,V_\zeta \varphi_0\rangle_{\mathcal{H}^{(\zeta)}}|=|\langle \bU[\pathvar] f, \varphi_0\rangle_{\mathcal{H}}|=
\overline{\bS}[\pathvar] f.\qedhere
\end{align*}
\end{proof}
Unlike the non-windowed scattering transform, $\overline{S},$ the windowed scattering transform $S_J$ is not in general permutation invariant, even in the limit as $J\rightarrow \infty.$ If one wishes the windowed-scattering transform to be invariant to the action of $\mathcal{G}$, then one must also require that $\mathcal{G}$ preserves the measure $\mu$ as defined below.

\begin{definition}
We say that $\mathcal{G}$ preserves the measure $\mu$ if 
$\mathcal{F}_\zeta=\mathcal{F}$ and $\mu_\zeta(B)=\mu(B)$ for all $\zeta\in\mathcal{G}$ and all $B\in\mathcal{F}$.
\end{definition}

To better understand this definition, we note that if $\mathcal{X}$ is a Riemannian manifold, then the isometry group preserves  $\mu$ when $\mu$ is the Riemannian measure, 
but not if a general $\mu$ is chosen. 
Similarly, if $\mathcal{X}$ is a graph, the permutation group will preserve $\mu$ if it gives equal weight to each vertex, but not if, for example, $\mu$ gives different weights to vertices depending on their degrees.

Under the assumption that $\mathcal{G}$ preserves the measure $\mu$,
we are able to show that the  \rev{windowed} scattering transform is invariant to the action of $\mathcal{G}$ in the limit as $J\rightarrow\infty$ at an exponential rate.

\begin{theorem}\label{thm: Jlimit}
Let $\mathcal{J}=\{0,\ldots,J\}$, and let $\mathcal{W}=\mathcal{W}_J$ be the diffusion wavelets constructed in \eqref{eqn: diffusion wavelets}. Suppose that $\varphi_0(x)$ is constant and assume $\mathcal{G}$ preserves both measures and inner products. Then for all $\zeta\in\mathcal{G}$, we have
\begin{equation*}
\|\bS_J f-\bS_J^{(\zeta)} V_\zeta f\|_{\ell^2(\mathcal{H})}\leq    2 |g(\lambda_1)|^{2^J}
     \|\bU f\|_{\ell^2(\mathcal{H})}.
\end{equation*}
\end{theorem}
The proof of Theorem \ref{thm: Jlimit} is based on  Lemma \ref{thm: window invariance} as well as the observation that  $\lim_{J\rightarrow\infty}\|A_JV_\zeta-A_J\|_{\mathcal{H}}=0$. We note that while Theorem \ref{thm: Jlimit} assumes that the $\mathcal{W}=\mathcal{W}_J$ are the diffusion wavelets constructed in \eqref{eqn: diffusion wavelets}, Lemma \ref{thm: window invariance} does not. (Note that other geometric wavelet constructions such as the one utilized in \cite{zou:graphCNNScat2018} also lead to versions of the scattering transform where the   \eqref{eqn: commutes} condition is satisfied.) For a proof of both Theorem \ref{thm: Jlimit} and Lemma \ref{thm: window invariance}, please see Appendix \ref{sec: proof of theorem Jlimit}.  

\begin{lemma}\label{thm: window invariance}
Assume $\mathcal{G}$ preserves both measures and inner products and that $S$ is equivariant with respect to the action of $\mathcal{G}$ in the sense that 
\begin{equation}\label{eqn: commutes}\bS^{(\zeta)} V_\zeta f=V_\zeta\bS  f.\end{equation} Then for all $\zeta\in\mathcal{G}$, we have
\begin{equation*}
\|\bS f-\bS^{(\zeta)} V_\zeta f\|_{\ell^2(\mathcal{H})}\leq \|V_\zeta A-A\|_\mathcal{H} \|\bU f\|_{\ell^2(\mathcal{H})}.
\end{equation*}
\end{lemma}

\begin{remark}\label{rem: U not f}
One limitation of Theorem \ref{thm: Jlimit} is that the right-hand side is given in terms of $\|Uf\|_{\ell^2(\mathcal{H})}$ instead of $\|f\|_{\mathcal{H}}.$ This is a common issue with many asymptotic invariance results for the windowed scattering transform. However, as first noted in \cite{mallat:scattering2012}, one may use \eqref{eqn: frameAB} and the fact that $\sigma$ is nonexpansive to see 
\begin{equation*}
    \sum_{p\in\mathcal{J}^{m}}\|U[p]f\|^2_{\mathcal{H}}\leq \sum_{p\in\mathcal{J}^{m-1}}\|U[p]f\|^2_{\mathcal{H}}
\end{equation*}
for any $m\geq1$. Therefore,
\begin{equation}
    \sum_{p\in\mathcal{J}^{m}}\|U[p]f\|^2_{\mathcal{H}}\leq \sum_{p\in\mathcal{J}^{m-1}}\|U[p]f\|^2_{\mathcal{H}}\leq \ldots\leq \sum_{p\in\mathcal{J}}\|U[p]f\|^2_{\mathcal{H}}\leq \|f\|_\mathcal{H}^2
\end{equation}
and so, if one only uses $M$ scattering layers, the total energy of $Uf$ may be bounded by 
$$\sum_{m\leq M}\left(\sum_{p\in\mathcal{J}^{m}}\|U[p]f\|_\mathcal{H}^2\right)\leq (M+1)\|f\|^2_\mathcal{H}.$$
Therefore, if one only uses finitely many scattering layers, the right-hand side of Theorem \ref{thm: Jlimit} may be controlled in terms of $\|f\|_\mathcal{H}.$ Moreover, in the case where $\mathcal{X}$ is a graph, for certain classes of wavelets we have 
\begin{equation}
    \sum_{p\in\mathcal{J}^{m}}\|U[p]f\|^2_{\mathcal{H}}\leq r \sum_{p\in\mathcal{J}^{m-1}} \|U[p]f\|^2_{\mathcal{H}}
\end{equation}
for some $r<1$ (see, e.g., Proposition 3.3 of \cite{zou:graphCNNScat2018} or Theorem 3.4 of \cite{perlmutter2019understanding}). Therefore, in this case, one has 
$$\|Uf\|_{\ell^2(\mathcal{H})}^2=\sum_{m=0}^\infty\left(\sum_{p\in\mathcal{J}^{m}}\|U[p]f\|_\mathcal{H}^2\right)\leq \sum_{m=0}^\infty r^m\|f\|^2_\mathcal{H} = \frac{1}{1-r}\|f\|^2_\mathcal{H}$$
independent of the number of layers used.
\end{remark}

The main results of this section, Theorems \ref{thm: equivariance}, \ref{thm: invariance}, and \ref{thm: Jlimit}, can be summarized as follows: If $\mathcal{G}$ preserves inner products, and the scattering transform is constructed using the diffusion wavelets defined in Section \ref{sec: wavelets}, then the windowed scattering transform is equivariant and the \rev{non-windowed} scattering transform is invariant to the action of $\mathcal{G}$. If we further assume that $\mathcal{G}$ preserves measure and that $\varphi_0$ is constant, then we also have that the windowed scattering transform is invariant in the limit as $J\rightarrow\infty$. 

As alluded to in the introduction, these invariance and equivariance results show that the scattering transform respects the intrinsic structure of the data and therefore is well-suited for a variety of machine learning tasks. In particular, the equivariance result, Theorem \ref{thm: equivariance}, shows that it is well-suited for point-level tasks such as the node classification task which we will consider in Section \ref{sec: results digraph}. Similarly, the invariance results Theorems \ref{thm: invariance} and \ref{thm: Jlimit}, show that it is well-equipped to handle shape-level tasks such as the manifold classification tasks considered in Sections \ref{sec: results 2d} and \ref{sec: results bio}. 

We also note that the assumption that $\mathcal{G}$ preserves inner products is quite natural. It is satisfied both when $\mu$ is the Riemannian measure on a manifold and $\mathcal{G}$ is the isometry group and when $\mathcal{X}$ is a (possibly signed, possibly directed) graph, $\mathcal{G}$ is the permutation group, and $\mu$ is any measure. The conditions that $\mathcal{G}$ preserves volumes and $\varphi_0$ is constant are a bit stronger. For example, when $\mathcal{X}$ is a graph, permutations do not preserve measure if the $\mu$ gives different vertices different weights. Moreover, if we take $\mathcal{L}$ to be the symmetric normalized graph Laplacian (on an undirected, unsigned graph), then $\varphi_0$ is given by $\varphi_0(x)\sim\text{degree}(x)^{1/2}$ and therefore is not constant unless the graph is regular.

\section{Stability}\label{sec: stability}
In this section, we show that the measure space scattering transform is robust to small perturbations to the measure $\mu$ and the diffusion operator $H$.
In particular, we consider a measure space $\mathcal{X}=(X,\mathcal{F},\mu)$ and another measure space $\mathcal{X}'=(X',\mathcal{F}',\mu')$ which we interpret as a perturbed version of $\mathcal{X}$. We assume that these two spaces have the same underlying sets and $\sigma$-algebras and that measures are mutually absolutely continuous with bounded Radon-Nikodyn derivatives, i.e., we have $X=X', \mathcal{F}=\mathcal{F}'$ and that there exist Radon-Nikodyn derivatives such that

 \begin{equation*}
     d\mu = \frac{d\mu}{d\mu'}d\mu'\quad\text{and}\quad \rev{d\mu' = \frac{d\mu'}{d\mu}d\mu.}
 \end{equation*}
 
To quantify the distortion between measures, we let $\mathcal{H}=\mathbf{L}^2(\mathcal{X})$ and  $\mathcal{H}'=\mathbf{L}^2(\mathcal{X}')$, and we  introduce two quantities, $R=R(\mathcal{H},\mathcal{H}')$ and $\kappa=\kappa(\mathcal{H},\mathcal{H}')$,  defined by  \begin{equation}\label{eqn: R}
     R\coloneqq R(\mathcal{H},\mathcal{H}')\coloneqq \max\left\{\left\|\frac{d\mu'}{d\mu}\right\|_\infty ,\left\|\frac{d\mu}{d\mu'}\right\|_\infty\right\}
 \end{equation}
and  
 \begin{equation*}
    \kappa(\mathcal{H},\mathcal{H}') = \max\left\{ \left\|1-\frac{d\mu}{d\mu'}\right\|_\infty,\left\|1-\frac{d\mu'}{d\mu}\right\|_\infty\right\}.
 \end{equation*}
We note that these two quantities are closely related to their analogs in \cite{perlmutter2019understanding} which focused on the special case where $\mathcal{X}$ was an undirected, unsigned graph. 
In the case where $\mu=\mu',$ we have $R(\mathcal{H},\mathcal{H}')=1$ and $\kappa(\mathcal{H},\mathcal{H}')=0$. Therefore, we will consider $\mu$ and $\mu'$ to be close to one another if $R\approx 1$ and $\kappa\approx 0$.
 
To further understand these definitions, consider the case where  $\mathcal{X}$ and $\mathcal{X}'$ are two (possibly signed, possibly directed) graphs with $N$ vertices and identify both vertex sets with $\{0,1,\ldots,N-1\}$. If $\mu$ and $\mu'$ are both the uniform measure, then we automatically have $R(\mathcal{H},\mathcal{H}')=1$ and $\kappa(\mathcal{H},\mathcal{H}')=0$. In this case, bounds produced in Theorem \ref{thm: wavelet stability} will simplify considerably as discussed below. Another natural choice of measure in the graph setting is to let $\mu(i)=\mathbf{d}_i^{-1}=\text{degree}(i)^{-1}$ since this is the measure needed in order to make the random-walk Laplacian $I-AD^{-1}$ self-adjoint. In this case, we have  $\rev{R}(\mathcal{H},\mathcal{H}')=\max_{0\leq i\leq N-1} \max\left\{\frac{\mathbf{d}_i}{\mathbf{d}'_i},\frac{\mathbf{d}'_i}{\mathbf{d}_i}\right\}$. In particular, if both $\mathbf{d}$ and $\mathbf{d}'$ satisfy the entrywise bound $0<m\leq \mathbf{d}_i,\mathbf{d}'_i\leq M<\infty$, we have $R(\mathcal{H},\mathcal{H}')\leq \frac{M}{m}.$
 
Observe that the assumption  $\rev{R}(\mathcal{H},\mathcal{H}')<\infty$ implies that the sets with measure zero with respect to $\mu$ are the same as those with measure zero with respect to $\mu'$. Therefore, each function $f\in\mathcal{H}$ can be uniquely identified with an element of  $\mathcal{H}'=\mathbf{L}^2(\mathcal{X}')$ (and vice-versa) and so we may regard the Hilbert spaces $\mathcal{H}$ and $\mathcal{H}'$ as having the same elements. Therefore, if $f\in\mathcal{H}$ and $\widetilde{f}\in\mathcal{H}'$, the subtraction $f-\widetilde{f}$ is well defined. We also note that 
 \begin{equation}\label{eqn: change of measure change of norm}
     \|f\|^2_{\mathcal{H}}=\int_X |f|^2 d\mu =\int_{X} |f|^2 \frac{d\mu}{d\mu'}d\mu'\leq R(\mathcal{H},\mathcal{H}')  \|f\|^2_{\mathcal{H}'}, \end{equation}
 and similarly, 
 \rev{\begin{equation}\label{eqn: change of measure change of norm 2}
 \|f\|^2_{\mathcal{H}'}\leq \rev{R}(\mathcal{H},\mathcal{H}') \|f\|^2_{\mathcal{H}}.\end{equation}} 
 We also observe that 
 \begin{equation}\label{eqn: change of inner product}
     |\langle f,g\rangle_{\mathcal{H}}-\langle f,g\rangle_{\mathcal{H}'}|=\left|\int_X f\bar{g} \left(1-\frac{d\mu'}{d\mu}\right) d\mu\right| \leq \kappa(\mathcal{H},\mathcal{H}')\|f\|_\mathcal{H}\|g\|_\mathcal{H}.
 \end{equation}
 
Let $\mathcal{L}$ and $\mathcal{L}'$ be self-adjoint positive semidefinite operators on $\mathcal{H}$ and $\mathcal{H}'$ respectively, and let $\{\varphi_k\}_{k=0}^\infty$, $\{\varphi_k'\}_{k=0}^\infty$ be the associated eigenbases. Let $g$ be a spectral function satisfying the same assumptions as described in Section \ref{sec: wavelets} and let $H$ and $H'$ be the associated operators defined as in \eqref{eqn: h}. Importantly, we note that we use the same function $g$ when constructing both $H$ and $H'$, so we may interpret $H$ and $H'$ as being analogous operators on different spaces. For example, in the case where $\mathcal{X}$ and $\mathcal{X}'$ are manifolds and $g(\lambda)=e^{-\lambda}$, one may check that $\{H^t\}_{t\geq 0}$ is the heat semigroup on $\mathcal{X}$ and $\{(H^t)'\}_{t\geq 0}$ is the heat semigroup on $\mathcal{X}'$.

Below, we prove a stability result for the wavelet transform. Our result will give bounds in terms of $\rev{R}(\mathcal{H},\mathcal{H}')$ and  $\kappa(\mathcal{H},\mathcal{H}')$, which measure how much $\mu$ differs from $\mu'$. However, these terms are not by themselves necessarily sufficient to characterize how different $\mathcal{X}$ is from $\mathcal{X}'$. For example, consider the case where $\mathcal{X}$ is a complete graph with $N$ vertices, $\mathcal{X}'$ is a cycle graph of $N$ vertices, and $\mathcal{L}$ and $\mathcal{L}'$ are the unnormalized graph Laplacians on $\mathcal{X}$ and $\mathcal{X}'$. In both of these cases, the natural choice of measure is to assign equal mass to each vertex, and so we will have $\mu(\{x\})=\mu'(\{x\})$ for every vertex $x\in X=X'$. It follows, that $\frac{d\mu'}{d\mu}=1$ uniformly, and therefore, we have $\rev{R}(\mathcal{H},\mathcal{H}')=1$ and $\kappa(\mathcal{H},\mathcal{H}')=0$. However, a complete graph and a cycle graph are clearly very far from being isomorphic as graphs in any reasonable sense. In particular, one way in which these graphs differ is that heat will diffuse much more rapidly through a fully connected graph than through a directed cycle. This motivates us to follow the lead of \cite{gama:diffScatGraphs2018} (see also \cite{coifman:diffusionMaps2006} and \cite{nadler:dmDynamic2006}) and consider the term 
\begin{equation}\label{eqn: diffusion dist} \|H-H'\|_\mathcal{H}.
\end{equation}
We note that since we assume that $\rev{R}(\mathcal{H},\mathcal{H}')$ is finite, the operator $H'$ is well-defined on $\mathcal{H}$. We also note that unlike \cite{gama:diffScatGraphs2018}, \eqref{eqn: diffusion dist} does not take the infimum over the orbits of $\mathcal{G}.$ This is because the wavelet transform is not invariant to the action of $\mathcal{G}$, but is merely equivariant. Therefore, no infimum will appear in Theorem \ref{thm: wavelet stability} stated below which establishes the stability of the wavelet transform. The scattering transform, by contrast, is invariant to the action of $\mathcal{G}$ and therefore such infimums will emerge in Theorems \ref{thm: windowstability} and  \ref{thm: scattering stability no window} which establish stability for the windowed and non-windowed scattering transforms.
 \subsection{Stability of the Wavelet Transform} 
  
 We will decompose $H$ and $H'$ by \begin{equation}\label{eqn: Hbar}
     H=\widetilde{H} + \overline{H},\quad H'=\widetilde{H}'+\overline{H}'
 \end{equation}
 where 
 \begin{equation*}
\widetilde{H}f=\widehat{f}(0)\varphi_0,\quad\text{and}\quad\overline{H}f=\sum_{k\geq 1} g(\lambda_k)\widehat{f}(k)\varphi_k,
\end{equation*} 
and $\widetilde{H}'$ and $\overline{H}'$ are defined similarly.

\begin{equation}\label{eqn: Hbarbound}
    \|\overline{H}f\|_\mathcal{H}^2 = \|\sum_{k\geq 1} g(\lambda_k)\widehat{f}(k)\varphi_k\|_{\mathcal{H}}^2\leq g(\lambda_1)^2\|f\|_{\mathcal{H}}^2 
\end{equation}
\rev{and similarly,
\begin{equation}\label{eqn: Hbarbound 2}
\|\widetilde{H}f\|_\mathcal{H}'^2 \leq g(\lambda'_1)^2\|f\|_{\mathcal{H}'}^2.
\end{equation}}
Moreover, combining \rev{\eqref{eqn: Hbarbound 2}} with \rev{\eqref{eqn: change of measure change of norm} and \eqref{eqn: change of measure change of norm 2}} implies  
\rev{
\begin{align*}
    \|\widetilde{H}'f\|_{\mathcal{H}}^2 
    \leq R(\mathcal{H}, \mathcal{H}') g(\lambda'_1)^2 \|f\|_{\mathcal{H}'}^2 
    \leq R(\mathcal{H}, \mathcal{H}')^2 g(\lambda'_1)^2 \|f\|_{\mathcal{H}}^2.
\end{align*}}
Therefore, 
\begin{equation}\label{eqn: beta form} \beta\leq \max\{g(\lambda_1), g(\lambda'_1)R(\mathcal{H},\mathcal{H}')\}. 
\end{equation}
In light of \eqref{eqn: beta form}, in order for the requirement that $\beta<1$ to hold it suffices for $\mu$ and $\mu'$ to be well-aligned enough so that $R(\mathcal{H},\mathcal{H}')<g(\lambda_1')^{-1}$.  Therefore, Theorem \ref{thm: wavelet stability} stated below can be interpreted as a local stability result where the radius of convergence depends on the spectral gap $\lambda_1'$.
  
 \begin{theorem}\label{thm: wavelet stability}
 Let $\mathcal{W}_J$ be the diffusion wavelets on $\mathcal{X}$ defined as in \eqref{eqn: diffusion wavelets}, and let $\mathcal{W}_J'$ be the analogous wavelets on $\mathcal{X}'$. 
Let $\beta = \max\{\|\overline{H}\|_{\mathcal{H}},\|\overline{H}'\|_\mathcal{H}\}$ and assume that $\beta<1.$ Then,  \begin{align*}
    &\|\mathcal{W}_{J}-\mathcal{W}_{J}'\|_{\ell^2(\mathcal{H})}^2 \\\leq& C(\beta) \left[\|\varphi_0-\varphi_0'\|^2_\mathcal{H} R(\mathcal{H},\mathcal{H}') + R(\mathcal{H},\mathcal{H}')^2\kappa(\mathcal{H},\mathcal{H}')^2 +\|H-H'\|_\mathcal{H}^2\right]
\end{align*}
where $C(\beta)=C\frac{\beta^2+1}{(1-\beta^2)^3}$ for some absolute constant $C>0$.
\end{theorem}
\noindent For a Proof of Theorem \ref{thm: wavelet stability}, please see Appendix \ref{sec: proof of wavelet stability}. As noted above, in the case where $\mathcal{X}$ is a graph and $\mu$ is the uniform measure, we have $R(\mathcal{H},\mathcal{H}')=1$ and $\kappa(\mathcal{H},\mathcal{H}')=0$. Therefore, the result of Theorem \ref{thm: wavelet stability} simplifies to 
 \begin{align*}
    \|\mathcal{W}_{J}-\mathcal{W}_{J}'\|_{\ell^2(\mathcal{H})}^2 \leq C(\beta) \left[\|\varphi_0-\varphi_0'\|^2_\mathcal{H} + \|H-H'\|_\mathcal{H}^2\right].
\end{align*}
 Furthermore, if $\mathcal{L}$ is the unnormalized graph Laplacian, we have $\varphi_0=\varphi_0'$, and the result further simplifies to  $\|\mathcal{W}_{J}-\mathcal{W}_{J}'\|_{\ell^2(\mathcal{H})}^2 \leq C(\beta)  \|H-H'\|_\mathcal{H}^2$.
\subsection{Stability of the Scattering Transform}\label{sec: scat stab}
In this section, we prove the stability of the windowed and \rev{non-windowed} scattering transforms. As in Section \ref{sec: nonexpansive}, and following the lead of \cite{perlmutter2019understanding}, in this section, we will not assume that the scattering transform is constructed using the diffusion wavelets constructed in Section \ref{sec: wavelets}. Instead, as in Section \ref{sec: scat def}, we will let $\mathcal{J}$ be an arbitrary countable indexing set and assume that 
\begin{equation*}
    \mathcal{W}=\{W_j,A\}_{j\in\mathcal{J}} \ \ \text{and}\ \  \mathcal{W}'=\{W'_j,A'\}_{j\in\mathcal{J}}
\end{equation*}
are any frames on $\mathcal{H}$ and $\mathcal{H}'$ such that \eqref{eqn: frameAB} holds. We do this because, for any given measure space, there may be many possible ways to construct wavelets, or more generally frames satisfying \eqref{eqn: frameAB} and in the Euclidean setting there have been various works defining the scattering transform using more general non-wavelet frames \cite{czaja2019analysis,grohs:cnnCartoonFcns2016, wiatowski:frameScat2015, wiatowski:mathTheoryCNN2018}. Therefore, we will show that the stability of the underlying frame directly implies the stability of the resulting scattering transforms. Throughout this section, we will let $\mathcal{S}^\ell$ denote the set of all $\ell$-th order scattering coefficients, 
on $\mathcal{X}$, i.e.,
\begin{equation*}
    \mathcal{S}^\ell f\rev{\coloneqq}\{S[p]f: p=(j_1,\ldots,j_\ell)\},
\end{equation*} and let $(\mathcal{S}^\ell)'$ denote the corresponding set of scattering coefficients on $\mathcal{X}'$. We will also continue to assume that the sets $X$ and $X'$ and the $\sigma$-algebras $\mathcal{F}$ and $\mathcal{F}'$ are the same and also that $R(\mathcal{H},\mathcal{H}')$ and $\kappa(\mathcal{H},\mathcal{H}')$ are finite. We recall that, as noted prior to \eqref{eqn: change of measure change of norm}, this means that $\mathcal{H}$ and $\mathcal{H}'$ can be regarded as having the same elements and so the subtraction of elements $\mathcal{H}'$ from elements of $\mathcal{H}$ is well defined.

\begin{theorem}[Stability for the Windowed Scattering Transform]\label{thm: windowstability}

Let $\mathcal{X}=(X,\mathcal{F},\mu)$ and $\mathcal{X}'=(X',\mathcal{F}',\mu')$ be measure spaces with $X=X'$ and $\mathcal{F}=\mathcal{F}'$. Let $\mathcal{H}=\mathbf{L}^2(\mathcal{X}),\mathcal{H}'=\mathbf{L}^2(\mathcal{X}')$ and  let $\mathcal{J}$ be a countable indexing set. Let  $\mathcal{W}=\{W_j,A\}_{j\in\mathcal{J}}$ and  $\mathcal{W}'=\{W'_j,A'\}_{j\in\mathcal{J}}$ be frames on $\mathcal{H}$ and $\mathcal{H}'$ such that  \eqref{eqn: frameAB} holds. Let $\bS^\ell$ and  $(\bS^\ell)'$ be the $\ell$-th layers of the windowed scattering transforms on $\mathcal{X}$ and $\mathcal{X}'$ constructed from $\mathcal{W}$ and $\mathcal{W}'$.
Further assume that 
$\bS^{\ell}$ is equivariant to the action of $\mathcal{G}$ and also invariant up to a factor of $\cB$ in the sense that
\begin{equation}\label{eqn: permutation assumption}
V_\zeta \bS^\ell f = \bS^{\ell,(\zeta)}V_\zeta f,\quad\text{and}\quad \left\|V_\zeta\bS^\ell f-\bS^\ell f\right\|_{\ell^2(\mathcal{H})}\leq \mathcal{B}\|f\|_{\mathcal{H}}
\end{equation}
for all $f\in\mathcal{H}$ and $\zeta\in\mathcal{G}$. Then for all $f\in\mathcal{H}$ and $\widetilde{f}\in\mathcal{H}'$, we have
\begin{align}\label{eqn: stability bound for windowed scattering transform}
    &\left\|\bS^\ell f-(\bS^\ell)'\widetilde{f}\right\|_{\ell^2(\mathcal{H})}\\\leq&      \inf_{\zeta\in\mathcal{G}}\Bigg[\mathcal{B}\|f\|_{\mathcal{H}} +R\left(\mathcal{H},\mathcal{H}^{(\zeta)}\right)\|V_\zeta f-\widetilde{f}\|_{\mathcal{H}}\nonumber\\
    &\qquad\quad +  \left( \sqrt{2}R\left(\mathcal{H},\mathcal{H}^{(\zeta)}\right)\|\cW^{(\zeta)}-\cW'\|_{\mathcal{H}^{(\zeta)}}\left(\sum_{k=0}^\ell\|\cW'\|_{\mathcal{H}^{(\zeta)}}^{k}\right)\right)\cdot\|\widetilde{f}\|_\mathcal{H} \Bigg].\nonumber
\end{align}

\end{theorem}
For a proof of Theorem \ref{thm: windowstability}, please see Appendix \ref{sec: proof of scattering stability}.
We note that if $\mathcal{W}$ are the diffusion wavelets constructed in Section \ref{sec: wavelets}, $\mathcal{G}$ preserves measures, and $\varphi_0$ is constant, then  Theorem \ref{thm: Jlimit} and Remark \ref{rem: U not f} imply condition \eqref{eqn: permutation assumption} holds with $\mathcal{B}=\sqrt{(\ell+1)|g(\lambda_1)|^{2^J}}$ (which converges to zero as $J\rightarrow\infty$). In particular, these conditions are satisfied both when $\mathcal{X}$ is a Riemannian manifold, $\mathcal{L}$ is the \rev{Laplace-Beltrami} operator, and $\mu$ is the Riemannian volume form and when $\mathcal{X}$ is a graph, $\mu$ is the uniform measure, and $\mathcal{L}$ is the unnormalized graph Laplacian.

We also note that we can interpret each of the terms on the right-hand side of \eqref{eqn: stability bound for windowed scattering transform}. We are looking for a bijection $\zeta\in\mathcal{G}$ which will simultaneously align both the wavelets $\mathcal{W}$ (which are typically constructed  from the operators $\mathcal{L}$ of $\mathcal{X}$), the Hilbert spaces $\mathcal{H}$, and the signal $f$. Therefore, the term $R(\mathcal{H},\mathcal{H}^{(\zeta)})$ measures how well aligned the Hilbert spaces are, the term $\|\mathcal{W}^{(\zeta)}-\mathcal{W}'\|_{\mathcal{H}^{(\zeta)}}$ measures how well aligned the wavelets are, and the term $\|V_\zeta f- \widetilde{f}\|_{\mathcal{H}}$ measures how well aligned the signals are. We also note that in the case where  $\mathcal{W}$ are the diffusion wavelets constructed in Section \ref{sec: wavelets}, we can control the term $\|\mathcal{W}^{(\zeta)}-\mathcal{W}'\|_{\mathcal{H}^{(\zeta)}}$ by applying Theorem \ref{thm: wavelet stability}.

The next result is the analogue of Theorem \ref{thm: windowstability} for the non-windowed scattering transform. We note that the terms on the right-hand side of \eqref{eqn: scat stab not window} have similar interpretations as those in Theorem \ref{thm: windowstability}. Additionally, by Theorems \ref{thm: nonexpansive nonwindow} and \ref{thm: invariance}, we note that the condition \eqref{eqn: invariance assumptions} is satisfied whenever $\inf_x|\varphi_0(x)|>0,$ $\lambda_1>0$, $\mathcal{G}$ preserves inner products and $\mathcal{W}=\mathcal{W}_J$ are the diffusion wavelets constructed in \eqref{eqn: diffusion wavelets}.

\begin{theorem}[Stability for the Non-windowed Scattering Transform]\label{thm: scattering stability no window}
Let $\mathcal{X}=(X,\mathcal{F},\mu)$ and $\mathcal{X}'=(X',\mathcal{F}',\mu')$ be measure spaces with $X=X'$ and $\mathcal{F}=\mathcal{F}'$. Let $\mathcal{H}=\mathbf{L}^2(\mathcal{X}),\mathcal{H}'=\mathbf{L}^2(\mathcal{X}')$ and  let $\mathcal{J}$ be a countable indexing set. Let  $\mathcal{W}=\{W_j,A\}_{j\in\mathcal{J}}$ and  $\mathcal{W}'=\{W'_j,A'\}_{j\in\mathcal{J}}$ be frames on $\mathcal{H}$ and $\mathcal{H}'$ such that  \eqref{eqn: frameAB} holds. Let $\overline{\bS^\ell}$ and  $(\overline{\bS^\ell})'$ be the $\ell$-th layers of the non-windowed scattering transforms on $\mathcal{X}$ and $\mathcal{X}'$ constructed from $\mathcal{W}$ and $\mathcal{W}'$.
Assume that $\overline{S}$ is fully invariant to the action of $\mathcal{G}$ and also Lipschitz continuous on $\mathcal{H}$ with constant $C_L$ in the sense that 
\begin{equation}\label{eqn: invariance assumptions}
   \|\overline{\bS}f_1-\overline{\bS}f_2\|_2^2\leq C_L\|f_1-f_2\|_\mathcal{H} \quad\text{and}\quad \overline{\bS^{(\zeta)}}V_\zeta f_1 = \overline{\bS} f_1
\end{equation}

 Then for all $f\in\mathcal{H}$ and $\widetilde{f}\in\mathcal{H}'$, we have
\begin{align}
&\left\|\overline{\bS^\ell}f-(\overline{\bS^\ell})'\widetilde{f}\right\|_{2}^2 \label{eqn: scat stab not window}\\\leq&  3\inf_{\zeta\in\mathcal{G}}\Bigg[ 2C_L
\|V_\zeta f-\tilde{f}\|^2_{\mathcal{H}^{(\zeta)}}  + R(\mathcal{H}^{(\zeta)},\mathcal{H}')^2\|\varphi_0^{(\zeta)}-\varphi_0'\|^2_{\mathcal{H}'}\|\widetilde{f}\|^2_{\mathcal{H}'}\nonumber\\
    &\qquad\quad+ 2\|\mathcal{W}^{(\zeta)}-\mathcal{W}'\|_{\mathcal{H}^{(\zeta)}}^2 \left(\sum_{k=0}^{\ell-1}\|\mathcal{W}'\|_{\mathcal{H}^{(\zeta)}}^k \right)^2\|\tilde{f}\|^2_{\mathcal{H}^{(\zeta)}}+ \kappa(\mathcal{H}',\mathcal{H}^{(\zeta)})\|\widetilde{f}\|_{\mathcal{H}'}.\bigg]\nonumber
\end{align}
\end{theorem}
\noindent For a proof of Theorem \ref{thm: scattering stability no window}, please see Appendix \ref{sec: Proof of Stability no window}.

\section{Implementing the Manifold Scattering Transform from Point-Cloud Data}\label{sec: numerical methods}

In \cite{perlmutter:geoScatCompactManifold2020}, the authors showed that the manifold scattering transform was effective for classification tasks on known two-dimensional surfaces with predefined meshes. However, in many applications of interest, one is not given a predefined manifold. Instead, one is given a collection of points $\{x_i\}_{0=1}^{N-1}$ embedded in some high-dimensional Euclidean space $\mathbb{R}^D$ and one makes a modeling assumption that these points lie on (or near) a comparatively low-dimensional manifold. Thus, in this section, we will assume that  $\mathcal{X}$ is a smooth $d$-dimensional Riemannian manifold without boundary which is embedded in $\mathbb{R}^D$ for some $D\gg d$ and that $\{x_i\}_{i=0}^{N-1}$ is a discrete subset randomly and independently sampled from $\mathcal{X}$. We will use the $x_i$ to construct a weighted graph $\mathcal{X}_N$ and present two methods which use $\mathcal{X}_N$ to implement an approximation of the manifold scattering transform when one only has access to these sample points. 

Both of these methods rely on 
an affinity kernel $K_\epsilon(\cdot,\cdot)$ to construct a data-driven graph $\mathcal{X}_N$ with weighted adjacency matrix $W^{(N)}$.  In our first method, we simply define an approximate heat semigroup at time $t=1$ by 
$H^{1}_{N,\epsilon}\rev{\coloneqq}(D^{(N)})^{-1}W^{(N)}$, where $D^{(N)}=W^{(N)}\mathbf{1}$ is the degree matrix associated to $W^{(N)}$. 
One may then approximate $H^{2^j}$ by, e.g., matrix multiplication. We note that while in principle, $H^1_{N,\epsilon}$ is a dense matrix, most of its entries will typically be small and therefore one may apply a threshold operator and use sparse matrix multiplications to implement an approximation of the wavelet transform. (Notably, if one imitates the method used in \cite{tong2022data}, there is no need to ever form a dense matrix after the initial thresholding.)
In our second method, we use $W^{(N)}$ to construct a data-driven graph Laplacian $L_{N,\epsilon}$. We then define a discrete approximation of the heat semigroup using the eigenvectors and eigenvalues of $L_{N,\epsilon}$.

In either case, once we have our approximations of $H^t$, it is then straightforward to implement the wavelet transform and therefore the scattering transform. The advantage of the second, eigenvector-based method is that we will be able to use results from \cite{dunson2021spectral,cheng2021eigen} to prove a quantitative rate of convergence for the scattering transform. The first method, on the other hand, is more computationally efficient for large $N$ (if one uses a thresholding operator to promote sparsity as discussed above) since it does not require one to compute an eigendecomposition.  
We are not able to prove a convergence rate for the scattering transform computed using this method, but we note that the approximation $H^1\approx H^{1}_{N,\epsilon}=(D^{(N)})^{-1}W^{(N)}$ was shown to converge pointwise \cite{coifman:diffusionMaps2006}, albeit without a rate. 

In order to avoid confusion, we will typically denote objects corresponding to $\mathcal{X}_N$ with a subscript or superscript $N$ and objects corresponding to $\mathcal{X}$ without such subscript or subscript. For example, we will let $W_j$ denote a wavelet on $\mathcal{X}$ at scale $2^j$ and $W_{j,N}$ denote the corresponding wavelet on $\mathcal{X}_N$. Throughout the section, we will choose $\mathcal{L}=-\nabla\cdot\nabla$ 
to be the \rev{Laplace-Beltrami} operator on $\mathcal{X}$, where $\nabla$ is the intrinsic gradient. 
We will also choose $g(\lambda)=e^{-\lambda}$, in which case $\{H^t\}_{t\geq0}$ is the heat semigroup (see Equation \eqref{eqn: heat differentiation}).
We will let $h_t(x,y)$ denote the heat kernel so that 
$H_tf(x)=\int_\mathcal{X} h_t(x,y)f(y)d\mu(y)$, where $d\mu$ is the Riemannian measure, normalized so that \begin{equation}\label{eqn: normalized}\mu(\mathcal{X})=1.\end{equation} It is well known that 
\begin{equation}
    \label{eqn: integrate to one}
    \int_\mathcal{X}h_t(x,y)d\mu(y)=1
\end{equation} for all $x\in\mathcal{X}$ and all $t>0$, and \begin{equation}\label{eqn: heat spectrum}h_t(x,y)=\sum_{k=0}^\infty e^{-t\mu_k}\varphi_k(x)\varphi_k(y), \end{equation} 
  where in \eqref{eqn: heat spectrum}, and throughout  this section, we will use $\mu_k$ to denote eigenvalues of the \rev{Laplace-Beltrami} operator $\mathcal{L}=-\nabla\cdot\nabla$ and will reserve $\lambda_k$ (sometimes with additional superscripts) for eigenvalues of the data-driven graph Laplacian which we will define below.

We now construct a weighted graph. We let $K(\cdot,\cdot)$ be an affinity kernel such as 
 \begin{equation}\label{eqn: gaussian kernel} 
     K(x,x')\rev{\coloneqq}K_\epsilon(x,x')\rev{\coloneqq}\epsilon^{-d/2}\exp\left(-\frac{\|x-x'\|^2_{2}}{\epsilon}\right),\quad\epsilon >0
 \end{equation}
where in the above equation $\|x-x'\|_2$ refers to the Euclidean distance between two points in $\mathbb{R}^D$ and $\epsilon$ is a bandwidth parameter.\footnote{Notably, our construction is sensitive to the choice of this bandwidth parameter. For more on this issue, we refer the reader to \cite{lindenbaum2020gaussian} which discusses some remedies to this sensitivity.} Given this kernel, we define an affinity matrix $W^{(N)}$ and a diagonal degree matrix $D^{(N)}$ by
 \begin{equation*}
     W^{(N)}_{i,j} \rev{\coloneqq} K(x_i,x_j) \ \ \text{and}\ \  D^{(N)}_{i,i}\rev{\coloneqq}\sum_{j=0}^{N-1}W^{(N)}_{i,j}.
 \end{equation*}
Given $W^{(N)}$ and $D^{(N)}$, one may then approximate $H^1$ by
\begin{equation}\label{eqn: no evecs}
H^{1}_{N,\epsilon}\rev{\coloneqq}(D^{(N)})^{-1}W^{(N)}.
\end{equation}
While the primary motivation of this method is to avoid computing eigenvectors and eigenvalues, we do note that \eqref{eqn: no evecs} may also be equivalently obtained from \eqref{eqn: h} by choosing $\mathcal{L}$ to be the Markov normalized Graph Laplacian $I^{(N)}-(D^{(N)})^{-1}W^{(N)}$ on $\mathcal{X}_N$ and choosing $g(\lambda)=1-\lambda.$

Our second method constructs approximations of $H^t$ based on \eqref{eqn: heat spectrum}. In our implementation, we may only use finitely many eigenvalues. This motivates us to define the truncated heat semigroup by
\begin{equation*}
H^\kappa_tf(x)\rev{\coloneqq}\int_\mathcal{X} h^\kappa_t(x,y)f(y)d\mu(y),\ \ \mbox{where}\ \     h^\kappa_t(x,y)\coloneqq\sum_{k=0}^\kappa e^{-t\mu_k}\varphi_k(x)\varphi_k(y),
\end{equation*}
where $\kappa$ is chosen by the user.
Our goal is to construct a good discrete approximation of $\mathcal{L}$. This will require controlling the two sources of error: (i) that we only use the first $\kappa+1$ eigenvalues and (ii) that we do not know the eigenvalues or eigenfunctions of the \rev{Laplace-Beltrami} operator $\mathcal{L}$ and must instead use the eigenvalues and the eigenvectors of the data-driven Laplacian defined below.  
The following lemma addresses (i) by bounding the error induced by only using finitely many eigenvalues.  For a proof please see Appendix \ref{sec: proof of lemma finite K}.
\begin{lemma}\label{lem: finite K}
For $\kappa\geq 0$ and $f\in\mathbf{L}^2(\mathcal{X})$, we have 
\begin{equation}\label{eqn: finite K}
    \|H_t^\kappa f-H_tf\|_{\mathbf{L}^2(\mathcal{X})}\leq e^{-t\mu_{\kappa+1}}\|f\|_{\mathbf{L}^2(\mathcal{X})}
\end{equation}
and also
\begin{equation}\label{eqn: finite K infinity}
    \|H_t^\kappa f-H_tf\|_\infty \leq C_\mathcal{X} \|f\|_\infty,
\end{equation}
where $C_\mathcal{X}$ is a constant which depends on the geometry of $\mathcal{X}$ but does not depend on $\kappa$, $t$, or $f$.
\end{lemma}

Next, we construct an unnormalized data-driven graph Laplacian by 
$$ L_{N,\epsilon}\rev{\coloneqq}\frac{1}{\epsilon N} (D^{(N)}-W^{(N)}).$$
We will interpret $W^{(N)}$ and $L_{N,\epsilon}$ as the adjacency matrix and Laplacian matrix of a data-driven graph $\mathcal{X}_N$.
We will denote the eigenvectors and eigenvalues of $L_{N,\epsilon}$ by $\lambda_k^{N,\epsilon}$ and $\mathbf{u}_k^{N,\epsilon}$ so that
\begin{equation}\label{eqn: ukn}
L_{N,\epsilon} \mathbf{u}_k^{N,\epsilon} = \lambda_k^{N,\epsilon} \mathbf{u}_k^{N,\epsilon}.
\end{equation}
When convenient, we make the dependence on $N$ and $\epsilon$ implicit and simply write $\mathbf{u}_k$ and $\lambda_k$ in place of $\mathbf{u}_k^{N,\epsilon}$ and $\lambda_k^{N,\epsilon}$. 
We define the discrete truncated heat-kernel matrix by
\begin{equation}\label{eqn: eigen approx heat}
    H_{N,\epsilon,\kappa,t} \rev{\coloneqq} \sum_{k=0}^\kappa e^{-{t\lambda_k^{N,\epsilon}}} \mathbf{u}_k^{N,\epsilon}(\mathbf{u}_k^{N,\epsilon})^T.
\end{equation}

To accomplish goal (ii), we will need discrete approximations of our eigenfunctions $\varphi_k$ of the \rev{Laplace-Beltrami} operator, which motivates us to introduce the normalized evaluation operator $\rho: \mathcal{C}(\mathcal{X})\rightarrow \mathbb{R}^N$ given by
$$ \rho f \rev{\coloneqq} \frac{1}{\sqrt{N}} (f(x_0),\ldots,f(x_{N-1})).$$  We then define 
$$
\mathbf{v}_k\coloneqq  \rho\varphi_k.
$$
We note these definitions differ slightly from \cite{cheng2021eigen}. There, the authors do not include the normalization term $\frac{1}{\sqrt{N}}$ in the definition of the evaluation operator $\rho$ but instead include it in the definition of the vector $\mathbf{v}_k$. Importantly, we note  that in either case the definition of $\mathbf{v}_k$ is ultimately the same, i.e., $(\mathbf{v}_k)_i=\frac{1}{\sqrt{N}}\varphi_k(x_i)$ (although \cite{cheng2021eigen} uses the letter $\phi$ instead of $\mathbf{v})$). 
Our convention is chosen so that $\mathbb{E}\|\rho f\|_2^2=\|f\|^2_{\mathbf{L}^2(\mathcal{X})}$. 
Additionally, we may also use Hoeffding's inequality to derive the following lemma which shows that  $\|\rho f\|_2^2\approx \|f\|^2_{\mathbf{L}^2(\mathcal{X})}$ with high probability as $N\rightarrow \infty$. For a proof please see Appendix \ref{sec: proof of Hoeffding}.
\begin{lemma}\label{lem: apply Hoeffding}Assume that the points $\{x_i\}_{i=0}^{N-1}$ are drawn i.i.d.\ uniformly at random, and
let $f,g\in\mathcal{C}(\mathcal{X})$. Then, with probability at least $1-\frac{2}{N^9}$, we have 
\begin{equation*}
    |\langle \rho f,\rho g \rangle_2 - \langle f,g\rangle_{\mathbf{L}^2(\mathcal{X})}| \leq  \sqrt{\frac{18 \log N}{N}}\|fg\|_\infty.
\end{equation*}
\end{lemma}

Our goal is to show that for a large fixed $\kappa$, in the limit as $N\rightarrow \infty$ and $\epsilon\rightarrow 0$, we have
$    H_{N,\epsilon,\kappa,t}\rho f \approx \rho H_t f
$ in the sense that $\kappa$ is large enough so that $\|\rho H_t f- \rho H_t^\kappa f\|_2$ is negligible, which follows for large $\kappa$ from Lemmas \ref{lem: finite K} and \ref{lem: apply Hoeffding}, and   $$\|H_{N,\epsilon,\kappa,t}\rho f-\rho H_t f\|_2\rightarrow 0 \quad \text{as } N\rightarrow \infty.$$
In order to do this, we need the following result that shows that $\lambda_k^{N,\epsilon}$ and $\mathbf{u}_k^{N,\epsilon}$ are good approximations of $\mu_k$ and $\mathbf{v}_k.$ It is a special case  of Theorem 5.4 from \cite{cheng2021eigen}, which follows by setting  $\epsilon \sim N^{-2/(d+6)}$.
\begin{theorem}[Theorem 5.4 of \cite{cheng2021eigen}]\label{thm: 5.4 of Chen and Wu}
Assume that the points $\{x_i\}_{i=0}^{N-1}$ are drawn i.i.d.\ uniformly at random and that the first $\kappa+2$ eigenvalues of $\mathcal{L},$ $\mu_0,\ldots,\mu_{\kappa+1}$, all have single multiplicity. As in \eqref{eqn: ukn}, let $\mathbf{u}_k^{N,\epsilon}$ and $\lambda_k^{N,\epsilon}$  be the eigenvectors and eigenvalues of the data-driven Laplacian constructed via the Gaussian affinity kernel $K_\epsilon$ defined as in \eqref{eqn: gaussian kernel}, and let $\kappa>0$ be fixed.  Assume that $\epsilon\rightarrow 0$ and $N\rightarrow\infty$ at a rate where $\epsilon \sim N^{-2/(d+6)}$. Then, with probability at least $1-\mathcal{O}\left(\frac{1}{N^9}\right)$,
there exist scalars $\alpha_k$ with
$$|\alpha_k|=1+o(1)$$ such that for all $0\leq k\leq \kappa$
\begin{equation*}
    |\mu_k-\lambda^{N,\epsilon}_k|=\mathcal{O}\left(N^{-\frac{2}{d+6}}\right),\quad \|\mathbf{u}_k^{N,\epsilon}-\alpha_k\mathbf{v}_k\|_2=\mathcal{O}\left(N^{-\frac{2}{d+6}}\sqrt{\log N}\right),
\end{equation*}
 where the constants implied by the big-$\mathcal{O}$ notation depend on $\kappa$ and the geometry of $\mathcal{X}$.
\end{theorem}

\begin{remark}\label{rem: alpha rate}
Inspecting the proofs of Theorem 5.4 of \cite{cheng2021eigen} and the related results in that paper shows that when $\epsilon \sim N^{-2/(d+6)}$, we have 
$$
\max\left\{||\alpha_k|-1|,\,\left|\frac{1}{|\alpha_k|}-1\right|\right\}\leq  \mathcal{O}\left(\sqrt{\frac{\log N}{N}}\right)+\mathcal{O}\left(\frac{\log(N)}{N^{4/(d+6)}}\right).
$$ 
Please see Appendix \ref{sec: remark proof} for details.
\end{remark}

%
%
Given Theorem \ref{thm: 5.4 of Chen and Wu}, we may use Lemma \ref{lem: apply Hoeffding} to derive the following result which shows that $H_{N,\epsilon,\kappa,t}\rho f$ converges to $\rho H_t^\kappa f$ as $N\rightarrow \infty$. Moreover, the rate of the convergence for $H_{N,\epsilon,\kappa,t}\rho f$, is the same (up to logarithmic factors) as the convergence rate for the eigenvectors and eigenvalues provided in Theorem \ref{thm: 5.4 of Chen and Wu}.
\begin{theorem}\label{thm: heat kernel discretization}
Let $f\in\mathcal{C}(\mathcal{X})$. Then, under the assumptions of Theorem \ref{thm: 5.4 of Chen and Wu} 
we have 
\begin{align*}
\|H_{N,\epsilon,\kappa,t}\rho f-\rho H_t^\kappa f\|_2^2\leq\max\{t^2,1\}\mathcal{O}\left(\frac{\log N}{N^{\frac{4}{d+6}}}\right)\left(\|f\|_{\mathbf{L}^2(\mathcal{X})}^2 +\sqrt{\frac{\log N}{N}}\|f\|^2_\infty\right)
\end{align*}
with probability at least $1-\mathcal{O}\left(\frac{1}{N^9}\right)$ if $d\geq 2$. In the case where $d=1,$ if the assumptions of Theorem \ref{thm: 5.4 of Chen and Wu} hold, then we have \begin{align}
    &\|H_{N,\epsilon,\kappa,t}\rho f-\rho H_t^\kappa f\|_2^2\nonumber\\\nonumber
    \leq& \max\{t^2,1\}\left(\mathcal{O}\left(\frac{\log N}{N^{4/7}}\right)\|f\|_{\mathbf{L}^2(\mathcal{X})}^2+ \mathcal{O}\left(\frac{\log N}{N}\right)\|f\|^2_\infty\right).
\end{align} In both cases, the implied constants depend both on $\kappa$ and on the geometry of $\mathcal{X}$.
\end{theorem}
If we combine Theorem \ref{thm: heat kernel discretization} with Lemma \ref{lem: finite K}, we may then obtain the following corollary.
\begin{corollary}\label{cor: heat kernel discretization}
Let $f\in\mathcal{C}(\mathcal{X})$. Then, under the assumptions  of Theorem \ref{thm: 5.4 of Chen and Wu}, we have
\begin{align}
&\|H_{N,\epsilon,\kappa,t}\rho f-\rho H_t f\|_2^2\\\leq& \max\{t^2,1\}\bigg[ \hspace{-.03in} \left(\mathcal{O}\left(\frac{\log N}{N^{\frac{4}{d+6}}}\right)+2e^{-2t\mu_{\kappa+1}}\right)\hspace{-.02in}\|f\|^2_{\mathbf{L}^2(\mathcal{X})}\hspace{-.02in}
+ \hspace{-.02in}\mathcal{O}\left(\sqrt{\frac{\log{N}}{N}}\right)\|f\|^2_\infty\bigg]
\end{align}
with probability at least $1-\mathcal{O}\left(\frac{1}{N^9}\right)$, 
where the constants implied by the $\mathcal{O}$ notation depend both on $\kappa$ and the geometry of $\mathcal{X}$.
\end{corollary}
\noindent For proofs of Theorem \ref{thm: heat kernel discretization} and Corollary \ref{cor: heat kernel discretization}, please see Appendix \ref{sec: heat convergence}.

In our eigenvector based method, where we approximate the heat semigroup via \eqref{eqn: eigen approx heat}, we next define a data-driven wavelet transform \begin{equation*}
\mathcal{W}_{J,N}\mathbf{x} \coloneqq \{W_{j,N}\mathbf{x}, A_{J,N}\mathbf{x}\}_{j=0}^J,
\end{equation*}
where  $W_{0,N}\mathbf{x}=(I_N-H_{N,\epsilon,\kappa,1})\mathbf{x}$, $A_{J,N}\mathbf{x} = H_{N,\epsilon,\kappa,2^{J}}\mathbf{x}$ and for $1\leq j\leq J$,
\begin{equation*}
W_{j,N}\mathbf{x} = H_{N,\epsilon,\kappa,2^{j-1}}\mathbf{x}-H_{N,\epsilon,\kappa,2^j}\mathbf{x}.\end{equation*}
We note that these wavelets implicitly depend on both $\kappa$ and $\epsilon$  in addition to $N$, but we suppress these dependencies in order to avoid cumbersome notation.
Analogously to Section \ref{sec: scat def}, for a path $p=(j_1,\ldots,j_m)$ we define
$$\bU_N[\pathvar]\mathbf{x}\rev{\coloneqq}\sigma W_{j_m,N}
\ldots \sigma W_{j_1,N} \mathbf{x}
$$
and define data-driven scattering coefficients
by 
$$S_{J,N}[\pathvar]\mathbf{x}\rev{\coloneqq}A_{J,N}U_N[\pathvar]\mathbf{x} \quad\text{and}\quad \overline{S}_{N}[\pathvar]\mathbf{x}\rev{\coloneqq}|\langle U_N[p]\mathbf{x},\mathbf{u}_0\rangle_2|.
$$
In the case where we approximate the heat semigroup via \eqref{eqn: no evecs} rather than \eqref{eqn: eigen approx heat}, we define $W_{j,N},U_N[p],$ and $S_{J,N}$ similarly, but with $H^{2^j}_{N,\epsilon}$ in place of $H_{N,\epsilon,\kappa,2^{j}}$ and we define the non-windowed scattering transform by $\overline{S}_{N}[\pathvar]\mathbf{x}=\|U_N[p]\mathbf{x}\|_1$ in order to avoid needing to compute any eigenvalues.

The following theorem uses Corollary \ref{cor: heat kernel discretization} to bound the discretization error of the wavelets.
\begin{theorem}\label{thm: wavelet discretization}

Let $f\in\mathcal{C}(\mathcal{X})$ and assume that the heat semigroup is approximated as in \eqref{eqn: eigen approx heat}. Then, under the assumptions  of Theorem \ref{thm: 5.4 of Chen and Wu}, we have that 

\begin{align*}
    &\|W_{j,N}\rho f - \rho W_j f\|_2^2
\\\leq&  2^{2j}\bigg[\left(\mathcal{O}\left(\frac{\log N}{N^{\frac{4}{d+6}}}\right)+\mathcal{O}\left(e^{-2^{j}\mu_{\kappa+1}}\right)\right)\|f\|^2_{\mathbf{L}^2(\mathcal{X})}+ \mathcal{O}\left(\sqrt{\frac{\log{N}}{N}}\right)\|f\|_\infty^2\bigg],\end{align*}
with probability at least $1-\mathcal{O}\left(\frac{1}{N^9}\right)$, where the constants implied by the big-$\mathcal{O}$ notation depend both on $\kappa$ and on the geometry of $\mathcal{X}$.

\end{theorem}

\begin{proof} 
For $j\geq 1$,
\begin{align*}
    \|W_{j,N}\rho f - \rho W_j f\|_2^2 &\leq \|(H_{N,\epsilon,\kappa,2^{j-1}}\rho f-H_{N,\epsilon,\kappa,2^{j}}\rho f) - (\rho H^{2^{j-1}}f - \rho H^{2^{j}}f) \|_2^2\\
    &\leq 2 \|H_{N,\epsilon,\kappa,2^{j-1}}\rho f- \rho H^{2^{j-1}}f \|_2^2 + 2 \|H_{N,\epsilon,\kappa,2^{j}}\rho f- \rho H^{2^{j}}f \|_2^2. 
\end{align*}
Therefore, the result follows from Corollary \ref{cor: heat kernel discretization}.
For the case where $j=0,$ we note that $I_N\rho f = \rho\text{Id} f$. Therefore, $$W_{j,N}\rho f - \rho W_j f=H_{N,\epsilon,K,1}\rho f- \rho H^{1}f$$
and we may again conclude by applying Corollary \ref{cor: heat kernel discretization}. 
\end{proof}
Iteratively applying Theorem \ref{thm: wavelet discretization}, one may obtain the following bound for the discretization error of $U_N[p]\rho f$. For a proof, please see Appendix \ref{sec: U conv}.

\begin{theorem}\label{thm: discretize U} 

 Let $f\in\mathcal{C}(\mathcal{X})$ and assume that the heat semigroup is approximated as in \eqref{eqn: eigen approx heat}. Let $p=(j_1,\ldots,j_m)$ be a path of length $m$ for some $m\geq 1,$ and let $j_{\max}=\max_{1\leq i \leq m} j_i$. Then,
under the assumptions  of Theorem \ref{thm: 5.4 of Chen and Wu}, we have that 
    \begin{align*}
    &\|U_{N}[p]\rho f - \rho U[p] f\|_2^2\\\leq &2^{2j_{\max}}\bigg[\hspace{-.02in}\left(\mathcal{O}\left(\frac{
    \log N}{N^{\frac{4}{d+6}}}\right)\hspace{-.02in}+\hspace{-.02in}\mathcal{O}\left(e^{-\mu_{\kappa+1}}\right)\right)\|f\|^2_{\mathbf{L}^2(\mathcal{X})}\hspace{-.02in}+\hspace{-.02in} \mathcal{O}\left(\sqrt{\frac{\log{N}}{N}}\right)\|f\|^2_\infty\bigg], \end{align*}
where the constants implied by the $\mathcal{O}$ notation depend on $m,$ $\kappa$, and the geometry of $\mathcal{X}$. 
\end{theorem}
Inspecting the proof of Theorem \ref{thm: discretize U}, one may observe that the constants implied by the big-$\mathcal{O}$ notation increase exponentially with respect to $m.$ However, in practice, one typically only uses two or three scattering layers, so we do not view this as a major limitation. We also note that a similar exponential dependence on the number of layers was observed for the generalization bounds for message passing networks arising from the discretization of graphons in \cite{maskey2022generalization}. Additionally, we note that, by inspecting the proof, it is clear that the implied constants in the term  $\mathcal{O}\left(e^{-\mu_{\kappa+1}}\right)$ do not depend on $\kappa$. A similar remark holds for the analogous terms in our subsequent results.

The next  two results establish convergence of the windowed and non-windowed scattering coefficients as $N\rightarrow\infty$. For proofs, please see Appendix \ref{sec: scat conv proofs}.

\begin{theorem}\label{thm: convergence windowed}
Let $f\in\mathcal{C}(\mathcal{X})$ and assume that the heat semigroup is approximated as in \eqref{eqn: eigen approx heat}.
Let $p=(j_1,\ldots,j_m)$ be a path of length $m$ for some $m\geq 1.$ Then, under the assumptions  of Theorem \ref{thm: 5.4 of Chen and Wu},
we have \begin{align*}
    &\|S_{J,N}[\pathvar]\rho f - \rho S_J[\pathvar] f\|_2^2\\\leq& 2^{2J}\bigg[\left(\mathcal{O}\left(\frac{\log N}{N^{\frac{4}{d+6}}}\right)+\mathcal{O}\left(e^{-\mu_{\kappa+1}}\right)\right)\|f\|^2_{\mathbf{L}^2(\mathcal{X})}+ \mathcal{O}\left(\sqrt{\frac{\log{N}}{N}}\right)\|f\|^2_\infty\bigg], \end{align*}
    with probability at least $1-\mathcal{O}\left(\frac{1}{N^9}\right)$,
where the constants implied by the big-$\mathcal{O}$ notation depend on $m,$ $\kappa$, and the geometry of $\mathcal{X}$.

\end{theorem}

The following is the analog of Theorem \ref{thm: convergence windowed} for the non-windowed scattering transform. 

\begin{theorem}\label{thm: convergence nonwindowed}
Let $f\in\mathcal{C}(\mathcal{X})$ and assume that the heat semigroup is approximated as in \eqref{eqn: eigen approx heat}. Let $p=(j_1,\ldots,j_m)$ be a path of length $m$ for some $m\geq 1.$ Then, under the assumptions  of Theorem \ref{thm: 5.4 of Chen and Wu},
we have \begin{flushleft}\begin{flalign*}
    &|\overline{S}_{N}[\pathvar]\rho f -  \overline{S}[\pathvar] f|\\\leq& 
    2^{J}\left [\left(\mathcal{O}\left(\frac{\sqrt{\log N}}{N^{\frac{2}{d+6}}}\right)+\mathcal{O}\left(e^{-\mu_{\kappa+1}/2}\right)\right)\|f\|_{\mathbf{L}^2(\mathcal{X})}+ \mathcal{O}\left(\left(\frac{\log N}{N}\right)^{1/4}\right)\|f\|_\infty\right]\end{flalign*}\end{flushleft}
    with probability at least $1-\mathcal{O}\left(\frac{1}{N^9}\right)$,
where the constants implied by the big-$\mathcal{O}$ notation depend on $m,$ $\kappa$, and the geometry of $\mathcal{X}$.
\end{theorem}

The convergence guarantees presented in this section may be summarized as follows. Theorem \ref{thm: 5.4 of Chen and Wu} is a result from \cite{cheng2021eigen} which provides convergence rates for the eigenvectors and eigenvalues of $L_{N,\epsilon}$. We then use this result to obtain convergence rates for our discretization of the heat semigroup, the wavelet transform, and the scattering transform. In all of these results, both the assumptions on the manifold and the convergence rate with respect to $N$ are the same as in Theorem \ref{thm: 5.4 of Chen and Wu}. Moreover, inspecting the proofs, one will observe that any future work which builds upon Theorem \ref{thm: 5.4 of Chen and Wu} by, e.g., relaxing the assumption that the $\lambda_k$ have single multiplicity, will readily lead to improved versions of our convergence results for the scattering transform. We also note that, given a point cloud, there are many possible ways to construct a graph Laplacian which approximates the \rev{Laplace-Beltrami} operator. For example, \cite{Calder2019} proves a result analogous to Theorem \ref{thm: 5.4 of Chen and Wu} for nearest neighbor graphs and $\epsilon$ graphs. One could readily modify our method to define approximations of the manifold scattering transform using these graphs, and it is likely that one could imitate the methods presented here in order to obtain convergence results  as $N\rightarrow\infty$.
Additionally, we note that under certain assumptions on the generation of data points $\{x_i\}_{i=0}^{N-1}$, for example, when the sampling is not uniform, the users could add additional terms which account for the density of the data (see, e.g., the $\alpha$-normalization approach \cite{coifman:diffusionMaps2006,dunson2021spectral, Little2022Balancing}) when constructing $W^{(N)}$ or to implement methods based on other data-driven Laplacians such as the longest-leg path distance Laplacian considered in \cite{little2020path}.  

\section{Numerical Results}\label{sec: results}
 The numerical effectiveness of the graph scattering transform for tasks such as node classification, graph classification, and even graph synthesis has been demonstrated in numerous works such as \cite{gao:graphScat2018,gama:diffScatGraphs2018,zou:graphCNNScat2018,wenkel2022overcoming,zou:graphScatGAN2019} and \cite{bhaskar2021molecular}.
 However, the numerical effectiveness of the manifold scattering transform is much less well established. Indeed, the initial work \cite{perlmutter:geoScatCompactManifold2020} only provided numerical experiments on two-dimensional surfaces with predefined meshes. Here, in Sections \ref{sec: results 2d} and \ref{sec: results bio}, we will show that the methods proposed in Section \ref{sec: numerical methods}
are effective for both synthetic and real-world data. 
As in Section \ref{sec: numerical methods}, we assume that we may only access the manifold though a finite collection of random samples $\{x_i\}_{i=0}^{N-1}$ in both Sections \ref{sec: results 2d} and \ref{sec: results bio}. Additionally, in Section \ref{sec: results digraph}, we will show that our proposed method is effective for node classifications on directed graphs.

First, in Section \ref{sec: results 2d} we will show that the manifold scattering transform is effective for learning on two-dimensional surfaces, even without a mesh. In particular, we will consider the same toy data sets that were analyzed with a mesh-based approach in \cite{perlmutter:geoScatCompactManifold2020}. These experiments aim to provide validation for our methods and show that the manifold scattering transform can still produce good results in the more challenging setting where one does not have access to the entire manifold. Having established proof of concept on toy data sets, in Section \ref{sec: results bio} we apply the manifold scattering transform 
to high-dimensional biomedical data where one models the data as lying upon some unknown manifold. In both of these settings, we will follow the lead of \cite{gao:graphScat2018} and \cite{bhaskar2021molecular} and augment the expressive power of the scattering transform by considering higher $q$-th order scattering moments for $1\leq q\leq Q$ defined by 
\begin{equation*}
    \overline{S}[p,q]\rho f=\frac{1}{N}\|U_{N}[p]\rho f\|_q^q.
\end{equation*}
Using these higher-order moments instead of the standard \rev{non-windowed} scattering transform increases the expressive power of our representation and helps compensate for the lack of global knowledge of the manifold. Notably, these scattering moments are invariant to the ordering of the data points, since by Theorem \ref{thm: equivariance} each $U_N[p]$ is equivariant to permutations (i.e., reorderings) and  each $\overline{S}[p,q]$ is defined via a global summation. Additionally, we note that if the kernel $K(x_i,x_j)$ is a function of the Euclidean distance between $x_i$ and $x_j$  then the $\overline{S}_N[p,q]$ will be invariant to rigid motions in the embedded space. Throughout this section, we shall report all accuracies as mean $\pm$ standard deviation. 

\subsection{Two-dimensional surfaces without a mesh}\label{sec: results 2d}

When implementing convolutional networks on a two-dimensional surfaces, it is standard, e.g., \cite{boscaini2015learning,boscaini2016learning} to use triangular meshes.  
In this section, we show that mesh-free methods can also work well in this setting. Importantly, note that we are \emph{not} claiming that mesh-free methods are \emph{better} for two-dimensional surfaces. Instead, we aim to show that these methods can work relatively well thereby justifying their use in higher-dimensional 
settings. 

 We  conduct experiments using both mesh-based and mesh-free methods on a spherical version of MNIST and on the FAUST dataset which were previously considered in \cite{perlmutter:geoScatCompactManifold2020}. In both methods, we use the wavelets defined in Section \ref{sec: wavelets} with \rev{two scattering layers and} $J=8$ and use a radial basis function (RBF) kernel support vector machine (SVM) \rev{see, for example,  \cite{cristianini2000introduction,chang2011libsvm}} with cross-validated hyperparameters as our classifier. For the mesh-based methods, we use the same discretization scheme as in \cite{perlmutter:geoScatCompactManifold2020} and set $Q=1$ which was the setting implicitly assumed there. For our mesh-free experiments, we use the eigenvector-based method discussed in Section \ref{sec: numerical methods} and set $Q=4$. We show that the information captured by the higher-order moments can help compensate for the structure lost by not using a mesh. For all of our experiments on spherical MNIST and FAUST, we used an 80/20 train-test split with 10-fold cross-validation.

\begin{figure}
    \centering
    \includegraphics[width=\textwidth]{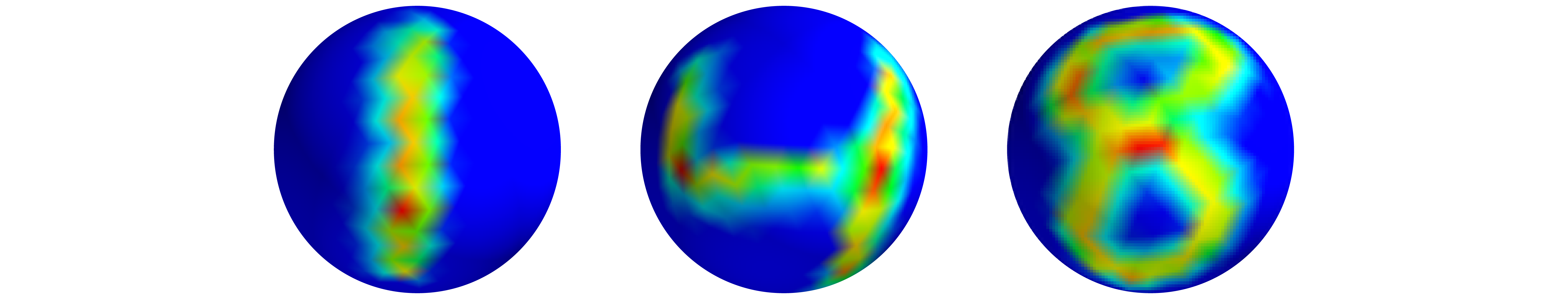}
    \caption{The MNIST dataset projected onto the sphere.}
    \label{fig:MNIST}
    \vskip -0.1in
\end{figure}

We first study the MNIST dataset projected onto the sphere as visualized in Figure \ref{fig:MNIST}. We uniformly sampled $N$ points from the unit two-dimensional sphere, and then applied random rotations to the MNIST dataset and projected each digit onto the spherical point cloud to generate a collection of signals $\{f_i\}$ on the sphere. Table \ref{tab: mnist} shows that for properly chosen $\kappa$, the mesh-free method can achieve similar performance to the mesh-based method. As noted in Section \ref{sec: numerical methods}, the implied constants in our theoretical results depend on $\kappa$. By inspecting the proof of Theorem 5.4 of \cite{cheng2021eigen} we see that for larger values of $\kappa,$ more sample points are needed to ensure the convergence of the first $\kappa$ eigenvectors in Theorem \ref{thm: 5.4 of Chen and Wu}. Thus, we want $\kappa$ to be large enough to get a good approximation of $H^1$, but also not too large. 
\begin{table}[t]
\caption{Classification accuracies for spherical MNIST averaged over 10 realizations.}
\label{tab: mnist}
\vskip 0.15in
\begin{center}
\begin{small}
\begin{sc}
\begin{tabular}{ccccc}
\toprule
Data type & $N$ &$\kappa$ & $Q$ & Accuracy (\%) \\
\midrule
Point cloud    &1200 & 200& 4&  $79 \pm 0.9$\\
Point cloud    &1200 & 400& 4&  $88 \pm 0.2$\\
Point cloud    &1200 & 642& 4&  $84 \pm 0.7$\\
Mesh & 642 & 642 & 1 & $91 \pm 0.2$\\

\bottomrule
\end{tabular}
\end{sc}
\end{small}
\end{center}
\vskip -0.1in
\end{table}

\begin{figure}
    \centering
    \includegraphics[width=\textwidth]{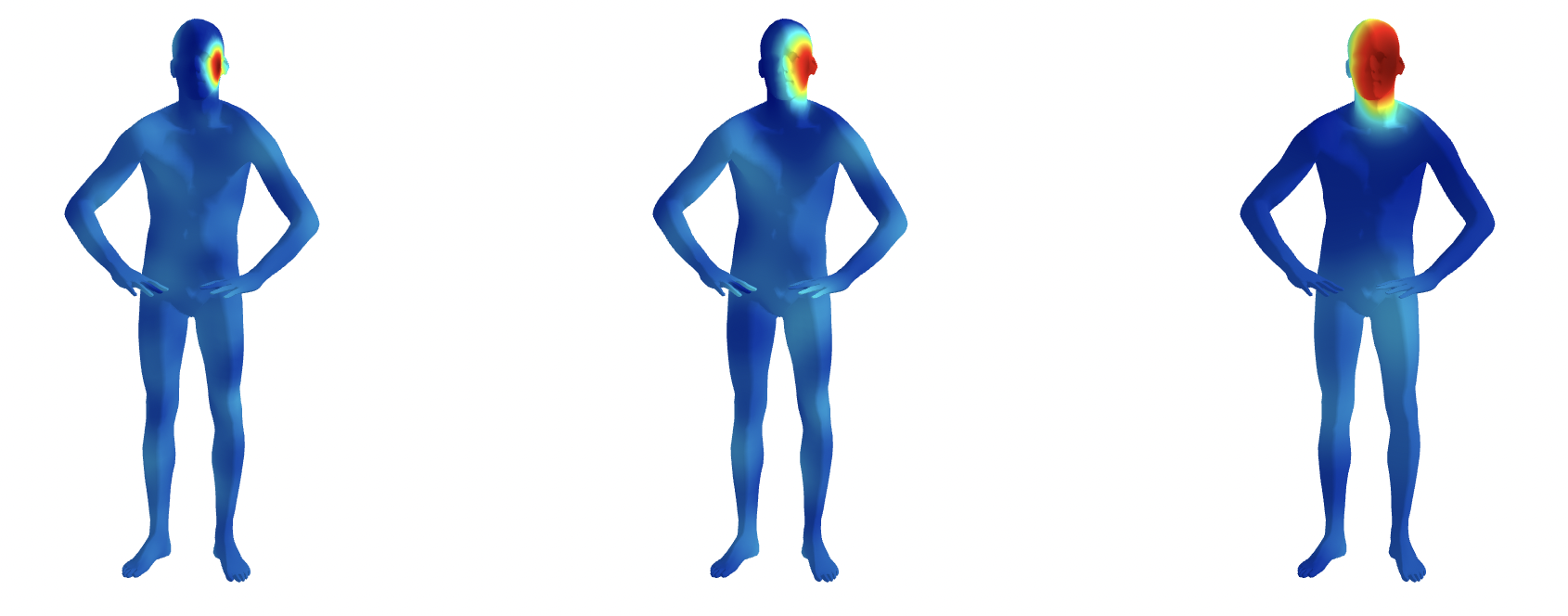}
    \caption{Wavelets on the FAUST dataset with $g(\lambda) = e^{-0.0005\lambda}$, $j = 1, 3, 5$ from left to right. Positive values are  red, while negative values are blue.}
    \label{fig:geometric wavelets faust}
\end{figure}

Next, we consider the FAUST dataset, a collection of surfaces corresponding to scans of ten people in ten different poses \cite{Bogo:CVPR:2014} as shown in Figure \ref{fig:geometric wavelets faust}. As in \cite{perlmutter:geoScatCompactManifold2020}, we use 352 SHOT descriptors \cite{tombari2010unique} as our signals. 
We use the first $\kappa = 80$ eigenvectors and eigenvalues of the approximate Laplace-Beltrami operator of each point cloud to generate scattering moments. We achieved $94 \pm 3.7$\% classification accuracy over 10 realizations
for the task of classifying different poses. This is comparable with the 95\% accuracy obtained with meshes in \cite{perlmutter:geoScatCompactManifold2020}.

\subsection{Single-cell datasets}\label{sec: results bio}

In this section, we present two experiments 
showing the utility of manifold scattering in analyzing single-cell data. We will formulate these experiments as manifold classification tasks, where each patient will correspond to a different manifold and the goal is to predict patient outcomes. In particular, each patient will correspond to a collection of cells, and each cell will correspond to a point in high-dimensional space\footnote{In order to turn the cells into points, we take single-cell protein measurements and apply a logarithmic transformation followed by $\ell^1$ normalization.}. 
Therefore, each patient will be described by a high-dimensional point cloud which we model as lying upon a low-dimensional manifold. In order to classify the patients, we compute the scattering transform on each manifold with signals corresponding to protein expression and then feed this representation into a classifier. For both of the experiments described in this section, we used a 75/25 train-test split. Notably, in both of the data sets we consider, the number of patients in fairly small. Therefore, the fact that the scattering transform uses predesigned filters is particularly advantageous in this setting.

On these datasets, we deviate slightly from our theory and demonstrate that our method can be effectively utilized with different graph constructions. In our first data set, which focuses on data derived from melanona patients, we use a $k$-NN graph with $k=5$. On our second data set, which is derived from COVID-19 patients, we  use a Gaussian kernel with an adaptive bandwidth which is designed to account for non-uniform density of the data points. Specifically, we set
\begin{equation}\label{HS_kernel}
K_{\texttt{k-nn}}(x, x')=
\frac{1}{2}\left(\exp\left(-\frac{\|x-x'\|^2_{2}}{\sigma_{k}(x)^2}\right) + \exp\left(-\frac{\|x-x'\|^2_{2}}{\sigma_{k}(x')^2}\right)\right),
\end{equation}
where $\sigma_k(x)$ is the distance from $x$ to its $k$-th nearest neighbor ($k=3$). We then approximate $H^1$ via \eqref{eqn: no evecs}.
 For the COVID data, we \rev{used three scattering layers} with $J=8$ and $Q=4$, \rev{imitating the settings used in \cite{gao:graphScat2018}} . We then apply principal component analysis (PCA) to the scattering features and train a decision tree classifier on the top 10 principal components. For the melanoma patients, we used $2$ scattering layers with $J=4$ and $Q=4$, followed by a multilayer perceptron with a single hidden layer. Additionally, with the melanoma data, in order increased the effective size of our training data, we subsample 400-point  point clouds and repeat this procedure 10 times for each point cloud (so the data set consists of 540 graphs rather than 54). Importantly, we note that we do this subsampling after splitting the data into train and test in order to ensure that no patient is in both the train and test set. As a baseline comparison, we compare our scattering-based method against a method which first preprocesses the data by using a $k$-means clustering based approach to extract features and then applies a decision tree classifier. For details on this baseline, please see Appendix \ref{sec: kmeans details}.

We first consider data collected in \cite{PtacekA59} on patients with various stages of melanoma.
All patients received checkpoint blockade immunotherapy, a treatment that licenses patient T cells to kill tumor cells. (For details on this therapy, see \cite{huang2022decade}.) 
In this dataset, 11,862 T lymphocytes from core tissue sections were taken from each of 54 patients  
diagnosed with melanoma, and 30 proteins were measured per cell. Therefore, we model our data as consisting of 54 manifolds 
embedded in 30-dimensional space  (with one dimension corresponding to each of the proteins) with 11,862 points per manifold.
We achieved $71\%$ accuracy when using scattering moments based on protein expression feature signals\footnote{Proteins were selected on the basis of having a known functional role in T cell regulation and included CD4, CD8, CD45RO, CD56, FOXP3, Granzyme B, Ki-67, LAG3, PD-1, and TIM-3.} with a decision tree classifier compared to $46\%$ accuracy using our baseline method.

We next consider data previously studied in \cite{RN3} comprised of 209 blood samples from $148$  people\footnote{In \cite{RN3} the data was taken from 168 patients. However, here we focus on the 148 patients for whom sufficient monocyte data was available.}. Of the 209 samples, 61 were taken from healthy controls, 123 were taken from patients who were COVID+ but recovered, and 25 were taken from patients who were COVID+ and died. 
Here, our goal is to predict whether the person corresponding to each blood sample died of COVID, recovered from COVID, or was a control. 
This task is particularly challenging because COVID outcome depends on a wide variety of known and unknown immunoregulatory pathways, unlike response to checkpoint blockade immunotherapy which targets a specific known immunoregulatory axis (T-cell inhibition).
We focus on innate immune (myeloid) cells, a population that has previously been shown to be predictive of patient mortality \cite{RN3}. Fourteen proteins were measured on $1,502,334$ total cells,  approximately $10,000$ cells per patient. To accommodate the size of these data sets lying in  $\mathbb{R}^{14}$, we first aggregate data points for each patient into less than 500 clusters via the diffusion condensation algorithm \cite{RN3}. We treat the centroids (with respect to Euclidean distance) of each cluster as single data points in the high-dimensional immune state space when implementing the manifold scattering transform. As with our melanoma experiments, we used signals related to protein expression, averaged across cells in each cluster. 
For the baseline method, $k$-means clustering, we set $k=3$ based on expected monocyte subtypes (classical, non-classical, intermediate). We achieved $48\%$ accuracy with scattering and a decision tree classifier compared to $40\%$ via the baseline method. See Table \ref{melanoma-table} for a summary of the results for both of the data sets discussed in this subsection. 

\begin{table}[t]
\caption{Classification accuracies for patient outcome prediction. }
\label{melanoma-table}
\vskip 0.15in
\begin{center}
\begin{small}
\begin{sc}
\begin{tabular}{lcccr}
\toprule
Data set & $N_{\text{patients}}$ & Baseline & Scattering \\
\midrule
Melanoma    &$54$ & $46.0\pm 7.1\%$ & $71.0\pm 9.0\%$ \\
COVID &$148$ &$40.1\pm 2.2\%$ & $47.7\pm 0.5\%$ \\
\bottomrule
\end{tabular}
\end{sc}
\end{small}
\end{center}
\vskip -0.1in
\end{table}

\subsection{Directed graphs}\label{sec: results digraph}

Next, we apply our framework to weighted and directed graphs $G=(V,E,W)$  with vertices $V$, edges $E$, and edge weights $W$. We turn $G$ into a measure space $\mathcal{X}=(X,\mathcal{F},\mu)$ by setting $X=V$, letting $\mathcal{F}$ be the set of all subsets of $V$, and letting $\mu$ be the uniform measure such that $\mu(\{v\})=1$ for all $v\in V$. In our experiments in this section, we will take $\mathcal{L}$ to be the normalized magnetic Laplacian described in detail below. 

We let $A$ denote the asymmetric, weighted adjacency matrix of $G$, and let $A^{(s)}=\frac{1}{2}(A+A^T)$ be its symmetric counterpart. Next, we define the symmetric, diagonal degree matrix $D^{(s)}$ by $D^{(s)}_{i,i}=\sum_{j=0}^{N-1}A^{(s)}_{i,j}$, where $N=|V|$, and $D^{(s)}_{i,j}=0$ if $i\neq j$. We then let $\Theta=A-A^T$ and define the Hermitian adjacency matrix by 
\begin{equation*}
    H^{(q)}=A^{(s)}\odot \exp(2\pi\mathbbm{i}q\Theta),
\end{equation*}
where $\mathbbm{i}=\sqrt{-1}$, $\odot$ denotes Hadamard product (componentwise multiplication), $q$ is a ``charge" parameter\footnote{This term, as well as the name Magnetic Laplacian originates from the Magnetic Laplaican serving as the quantum mechanical Hamiltonian of a particle under magnetic flux\cite{lieb1993fluxes}.}, and exponentiation is defined componentwisely, i.e.,
$$\exp(2\pi\mathbbm{i}q\Theta)_{i,j}=\exp(2\pi\mathbbm{i}q\Theta_{i,j}).$$
Notably, $H^{(q)}$ encodes the undirected geometry of the graph in the magnitude of its entries and directional information via its phases. The charge parameter $q$ allows one to balance the relationship between directed and undirected information as desired. 

Given $H^{(q)}$, we define the unnormalized and normalized magnetic Laplacians by 
$$L_U^{(q)}=D^{(s)}-H^{(q)}$$ and 
\begin{align*}
L_N^{(q)}&\,=(D^{(s)})^{-1/2}L_U^{(q)}(D^{(s)})^{-1/2}\\
&\,=I-(D^{(s)})^{-1/2}H^{(q)}(D^{(s)})^{-1/2}.
\end{align*}
By construction, both $L_U^{(q)}$ and $L_N^{(q)}$ are Hermitian and one may check (see, e.g., Theorem 1 of \cite{zhang2021magnet}) that they are positive semidefinite. Therefore, both of these matrices fit within our framework as admissible choices of $\mathcal{L}$ and can be used to define scattering transforms on directed graphs.

\begin{figure}
    \centering
    \begin{subfigure}[t]{0.3\textwidth}
    \centering
    \includegraphics[width=\textwidth]{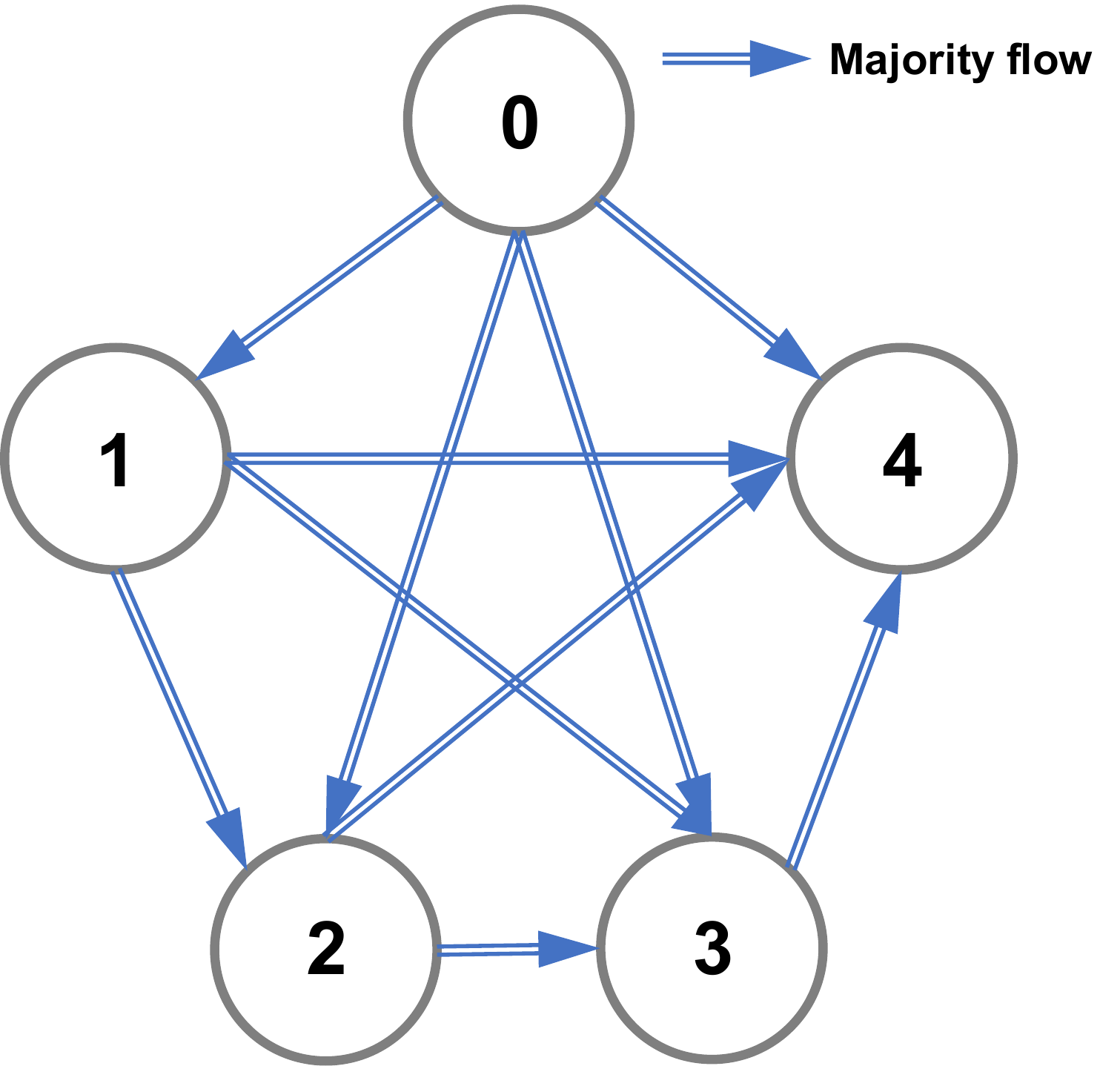}
    \caption{Ordered}
    \label{fig:ordered metagraph}
    \end{subfigure}
    \begin{subfigure}[t]{0.3\textwidth}
    \centering
    \includegraphics[width=\textwidth]{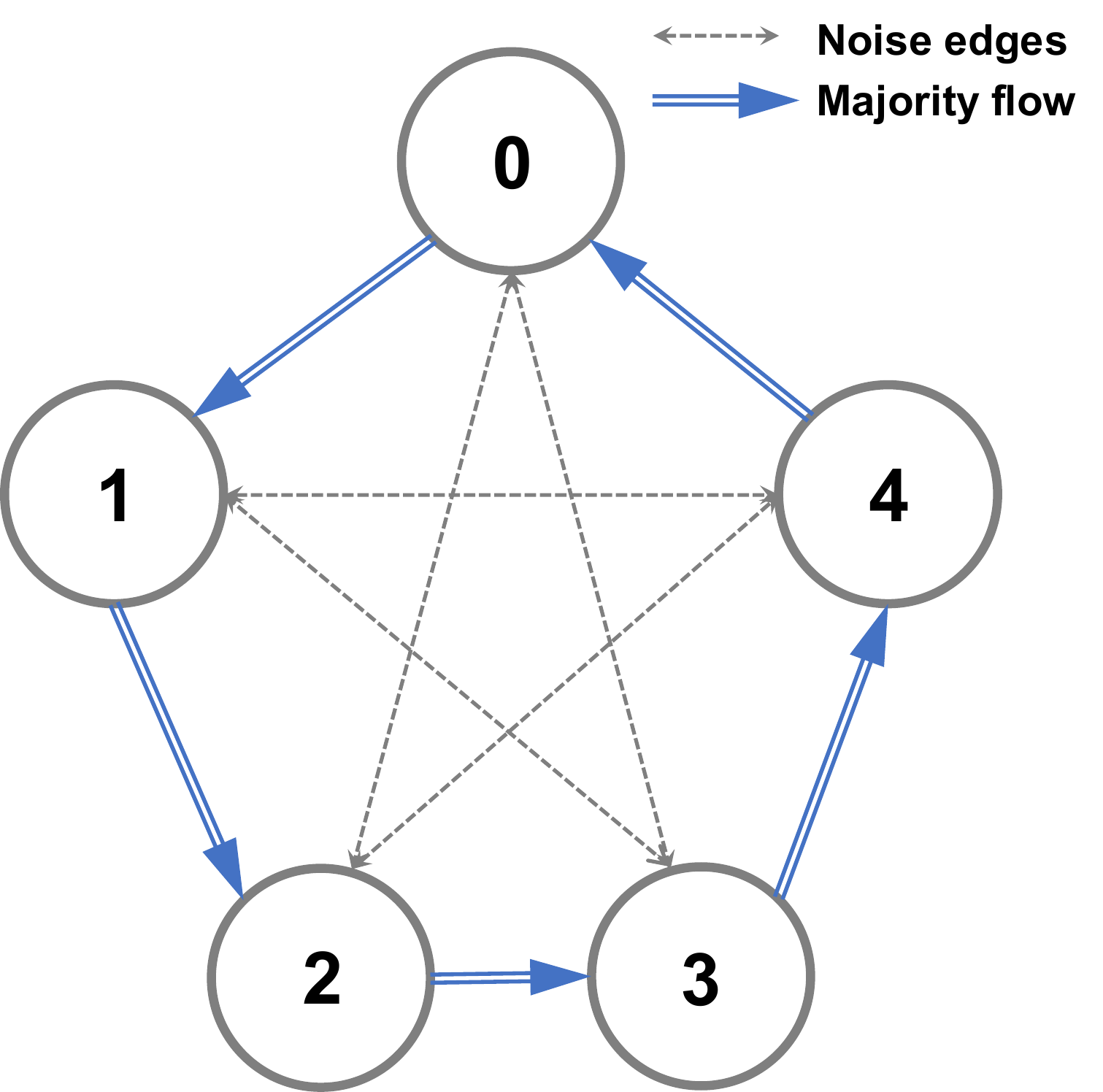}
    \caption{Noisy cyclic}    
    \label{fig: noisy cyclic metagraph}
    \end{subfigure}
    \caption{Meta-graphs for synthetic datasets (Reproduced from \cite{zhang2021magnet}).}
\end{figure}

In our experiments, we will choose $\mathcal{L}=L_N^{(q)}$ and consider the task of node classification on the following directed stochastic block model considered in  \cite{zhang2021magnet}. 
We first divide the $N$ vertices into $n_c$ equally-sized     clusters $C_1,\ldots,C_{n_c}$ for some $n_c$ which divides $N$. We let $\{\alpha_{i,j}\}_{1\leq i,j\leq n_c}$ to be a collection of probabilities, with $\alpha_{i,j}=\alpha_{j,i}$ and  $0<\alpha_{i,j}\leq 1$. For an unordered pair of vertices, $u,v\in V$, $u\neq v$ we create an undirected edge between $u$ and $v$ with probability $\alpha_{i,j}$ if $u\in C_i,v\in C_j$. We then define $\{\beta_{i,j}\}_{1\leq i,j\leq n_c}$ to be a collection of probabilities such that $\beta_{i,j}+\beta_{j,i}=1$ and $0\leq\beta_{i,j}\leq 1$. We then replace each undirected edge $\{u,v\}$, with a directed edge which points from $u$ to $v$ with probability $\beta_{i,j}$ if $u\in C_i$ and $v\in C_j$, and otherwise points from $v$ to $u$. Notably, if $\alpha_{i,j}$ is constant, then the only way to determine the clusters will be from the directional information.

For our experiments, we set $n_c = 5$ and consider three meta-graphs: ordered,  cyclic, and noisy cyclic. For all meta-graphs, we set $\beta_{i, i} = 0.5$. For the ordered meta-graph (Figure \ref{fig:ordered metagraph}), we set $\alpha_{i, j} = 0.1$ for all $i, j$ and set  $\beta_{i, j} = 0.95$ for $i < j$. For the cyclic meta-graph (Figure \ref{fig: noisy cyclic metagraph}, but without the dashed grey edges), we set 
$$\alpha_{i, j} = \left \{ \begin{array}{ll}
         0.1 & i = j \\
         0.1 & i = (j \pm 1) \text{ mod } 5 \\
         0 & \text{otherwise}
         \end{array} \right. \text{and } \beta_{i, j} = \left \{ \begin{array}{ll}
         0.5 & i = j \\
         0.95 & i = (j-1) \text{ mod } 5 \\
         0.05 & j = (i-1) \text{ mod } 5 \\
         0 & \text{otherwise}
         \end{array}\right . .$$ Finally, for the noisy cyclic meta-graph (Figure \ref{fig: noisy cyclic metagraph}), we set $\alpha_{i,j} = 0.1$ for all $i, j$ and set the edge direction probabilities as $$\beta_{i,j} = \left \{\begin{array}{ll}
         0.95 & i = (j-1) \text{ mod } 5 \\
         0.05 & j = (i-1) \text{ mod } 5 \\
         0.5 & \text{otherwise}
         \end{array}\right . .$$

Motivated by the so-called residual convolution operators used in \cite{wenkel2022overcoming},  
for improved numerical performance, we use a modified version of the windowed scattering transform given by $\bS^{\text{res}}_J[\pathvar]=H^1\bU[p]$ in our experiments. \rev{We chose our input signals to be i.i.d.\ standard Gaussian random vectors and used paths of length $m \in \{0, 1, 2\}$.}
Following the settings used in \cite{zhang2021magnet}, \rev{we set $N=2500$ and $n_c = 500$ for the ordered and cyclic meta-graphs and $N=500$ and $n_c = 100$ for the noisy meta-graph, and} we used 2\%, 10\%, and 60\% of the nodes in each cluster for training for the ordered, cyclic, and noisy cyclic meta-graphs, respectively. On all three data sets, we used 20\% of the nodes for validation and the remaining nodes were used for testing. Details on our validation procedure are provided in Appendix \ref{sec: hyperparams}. After computing the scattering transform, we used an SVM with an RBF kernel for classification. 
In Table \ref{tab: dsbm}, we report our results for each meta-graph along with the maximum scale $J$ used to compute the scattering coefficients and parameter $q$ used to compute the magnetic Laplacian\footnote{All baseline results taken from \cite{zhang2021magnet}.}. As we can see, scattering performs well on all three versions of the stochastic block model and is the top-performing method on the noisy cyclic stochastic block model\footnote{For methods not designed for di-graphs, the reported accuracies are the maximum of those obtained by a)  symmetrizing the adjacency matrix as a preprocessing step and b) running the algorithm as is.}.
\begin{table}[t]
\vskip 0.15in
\begin{center}
\begin{small}
\begin{sc}
\begin{tabular}{cccc}
\toprule
Method/Meta-graph & Ordered           & Cyclic          & Noisy cyclic      \\
\midrule
MagNet \cite{zhang2021magnet}  & $99.6\pm0.2$      & $100.0 \pm 0.0$ & $80.5 \pm 1.0$    \\ 
ChebNet \cite{Defferrard2018}& $19.9\pm0.7$      & $74.7 \pm 16.5$ & $18.3 \pm 3.1$    \\
GCN \cite{kipf2016semi}& $68.6 \pm 2.2$ & $78.87 \pm 30.0$ & $24.2 \pm 6.8$ \\
APPNP \cite{klicpera2018predict}& $97.4\pm1.8$      & $19.6 \pm 0.5$  & $17.4 \pm 1.8$    \\
SAGE \cite{hamilton2017inductive} & $20.2 \pm 1.2$  & $88.6 \pm 8.3$ & $26.4 \pm 7.7$\\
GIN \cite{xu2018how} &  $57.9 \pm 6.3$ & $75.3 \pm 21.5$ & $24.7 \pm 6.4$\\
GAT \cite{velivckovic2017graph} & $42.0 \pm 4.8$  & $98.3 \pm 2.2$ & $27.4 \pm 6.9$\\
DGCN \cite{tong:directedGCN2020}& $81.4\pm1.1$      & $83.7 \pm 23.1$ & $37.3 \pm 6.1$    \\
DiGraph \cite{tong2020digraph} & $82.5 \pm 1.4$  & $39.1 \pm 33.6$ & $18.0 \pm 1.8$\\
DiGraphIB \cite{tong2020digraph} & $99.2 \pm 0.4$  & $84.8 \pm 17.0$ & $43.4 \pm 10.1$\\
Scattering & $97.8 \pm 1.2$    & $99.8\pm0.2$    & $88.5\pm4.0$      \\
\midrule
Parameters & $J = 9, q = 0.25$ & $J = 9, q = 0$  & $J = 10, q = 0.2$ \\
\bottomrule
\end{tabular}
\end{sc}
\end{small}
\end{center}
\vskip -0.1in
\caption{Node classification accuracy on our directed stochastic block model with different meta-graph structures (ordered, cyclic, or noisy cyclic) and different graph methods. }
\label{tab: dsbm}
\end{table}

\section{Conclusion}\label{sec: conclusion}

In this work, we have extended the geometric scattering transform to a broad class of measure spaces. In particular, our construction extends several previous works defining the scattering transform on undirected, unsigned graphs and smooth compact Riemannian manifolds without boundary as special cases and also includes many other examples as discussed extensively in Section \ref{sec: examples}. Our invariance and equivariance results help clarify the relationship between the  invariance / equivariance of the scattering transform and the group of bijections to which it is invariant or equivariant. Namely, they show that the critical property for $\mathcal{G}$ to possess is that for every $\zeta\in\mathcal{G}$, the operator $V_\zeta$, defined by $V_\zeta f(x)=f(\zeta^{-1}(x))$, is an isometry on $\mathbf{L}^2(\mathcal{X})$. Additionally, we provide two numerical schemes for implementing the manifold scattering transform when one only has access to finite point clouds and provide quantitative convergence rates for one of these schemes as the number of sample points grows to infinity. The proof of this convergence result utilizes previous work showing the convergence of the eigenvectors and eigenvalues of the \rev{Laplace-Beltrami} operator. While we do not know whether or not our convergence rate is optimal, we do note that  both the assumptions of our convergence theorems and our convergence rates are the same as for the previous work on the convergence of the eigenvectors and eigenvalues. Therefore, our convergence results should be interpreted as showing that the number of sample points needed to apply scattering to high-dimensional point cloud data is the same as other manifold learning based methods.

We believe our work opens up several new exciting avenues for future research. The framework presented here provides a theoretical foundation for defining neural networks on manifolds from point-cloud data, a relatively unexplored topic except in the setting of two-dimensional surfaces. Additionally, it would be interesting to extend our methods to higher-order operators such as the connection Laplacian or to implement versions of our method that utilize anisotropic diffusions. 
Lastly, it would be interesting to improve on our convergence results by relaxing the assumptions on the data generation or developing quantitative convergence guarantees that do not require the explicit computation of eigenvectors or eigenvalues.

\appendix
\section{The Proof of Proposition \getrefnumber{prop: frame}}\label{sec: proof of frame}
\begin{proof} To prove the upper bound, we note, 
\begin{align}
& \sum_{j=0}^J\|W_jf\|_\mathcal{H}^2+ \|A_Jf\|_\mathcal{H}^2 \nonumber\\
=& \sum_{k\in\mathcal{I}}\sum_{j=0}^J \left(|\widehat{W_jf}(k)|^2 +  |\widehat{A_Jf}(k)|^2\right) \nonumber\\
=& \sum_{k\in\mathcal{I}} \left(|1-g(\lambda_k)|^2+\sum_{j=1}^J \left|g(\lambda_k)^{2^{j-1}}-g(\lambda_k)^{2^{j}}\right|^2+ \left|g(\lambda_k)^{2^J}\right|^2\right) |\widehat{f}(k)|^2\nonumber\\
\leq&  \sum_{k\in\mathcal{I}} \left(1-g(\lambda_k)+\sum_{j=1}^J \left(g(\lambda_k)^{2^{j-1}}-g(\lambda_k)^{2^{j}}\right)+ g(\lambda_k)^{2^J}\right)^2 |\widehat{f}(k)|^2\label{eq: kill the abs}\\
= & \sum_{k\in\mathcal{I}} |\widehat{f}(k)|^2\nonumber\\
=& \|f\|_{\mathcal{H}}^2,\nonumber
\end{align}
where in \eqref{eq: kill the abs}, we used the fact that $g$ is nonnegative and decreasing.

In order to prove the lower bound, we define $p_0(t)\coloneqq(1-t),$ $p_j(t) \coloneqq (t^{2^{j-1}}-t^{2^j})$ if $1\leq j \leq J$, and $p_{J+1}(t)\coloneqq t^{2^{J}}$ and observe that since $g$ is positive and decreasing we have
\begin{align*}
& \sum_{j=0}^J\|W_jf\|_\mathcal{H}^2+ \|A_Jf\|_\mathcal{H}^2 \nonumber\\
=& \sum_{k\in\mathcal{I}} \left(|1-g(\lambda_k)|^2+\sum_{j=1}^J \left|g(\lambda_k)^{2^{j-1}}-g(\lambda_k)^{2^{j}}\right|^2+ \left|g(\lambda_k)^{2^J}\right|^2\right) |\widehat{f}(k)|^2\nonumber\\
=& \sum_{k\in\mathcal{I}} \left(p_0(g(\lambda_k))^2+\sum_{j=1}^J p_j(g(\lambda_k))^2+ p_{J+1}(g(\lambda_k))^2\right) |\widehat{f}(k)|^2\nonumber\\\geq& \min_{0\leq t \leq 1} \sum_{j=0}^{J+1} p_j(t)^2\|f\|_{\mathcal{H}}^2,
\end{align*}
 where in the final inequality we use Plancherel's identity and the fact that $0\leq g(\lambda_k)\leq g(0)=1.$

 Therefore it suffices to show that $\min_{0\leq t \leq 1} \sum_{j=0}^{J+1} p_j(t)^2\geq c>0$.
To do so, we let $0\leq t\leq 1$ and consider three cases. First, if $0\leq t\leq 1/2,$ then
\begin{equation*}
    \sum_{j=0}^{J+1} p_j(t)^2 \geq p_0(t)^2 = (1-t)^2\geq \left(1-\frac{1}{2}\right)^2=1/4.
\end{equation*}
Secondly, if $t^{2^J}\geq1/2$, 
\begin{equation*}
    \sum_{j=0}^{J+1} p_j(t)^2  \geq p_{J+1}(t)^2 = \left(t^{2^J}\right)^2\geq \left(\frac{1}{2}\right)^2=1/4.
\end{equation*}
In the final case where $t^{2^{J}}<\frac{1}{2}<t,$ there exists a unique $j_0,$ $1\leq j_0\leq J,$ such that $t^{2^{j_0}} < 1/2 \leq t^{2^{j_0-1}}.$ Since $t^{2^{j_0-1}} \geq 1/2$ and $t^{2^{j_0-1}} t^{2^{j_0-1}} = t^{2^{j_0}}$ it follows that $1/4\leq t^{2^{j_0}} < 1/2$ and thus $1/2 \leq t^{2^{j_0-1}} < 1/\sqrt{2}$.
Therefore, in this case we have 
\begin{equation*}
    \sum_{j=0}^{J+1} p_j(t)^2 \geq p_{j_0}(t)^2 = (t^{2^{j_0-1}} - t^{2^{j_0}})^2 \geq \inf_{x \in [\frac{1}{2}, \frac{1}{\sqrt{2}}]} (x - x^2)^2 \eqqcolon c > 0.
\end{equation*}

\end{proof}

    \section{The Proof of Proposition \getrefnumber{prop: Jlimit}}\label{sec: PJlimit proof}
    \begin{proof}
    Using the definition on the scattering transform as well as the fact that $g(0)=1,$ one may compute 
    \begin{align*}
    \left||\bS_J[\pathvar] f(x)| - \overline\bS[\pathvar] f|\varphi_0(x)|\right|&=\left|\left|\sum_{k\geq 0} g(\lambda_k)^{2^J}\langle \bU[\pathvar]f,\varphi_k\rangle_{\mathcal{H}}\varphi_k(x) \right|-|\langle\bU[\pathvar]f,\varphi_0\rangle_{\mathcal{H}}||\varphi_0(x)| \right|\\&\leq\left|\sum_{k\geq 1} g(\lambda_k)^{2^J}\langle \bU[\pathvar]f,\varphi_k\rangle_{\mathcal{H}}\varphi_k(x)\right|.
    \end{align*}
    Therefore, Parseval's identity implies that 
    \begin{align*}
     \left\||\bS_J[\pathvar] f(x)| - \overline\bS[\pathvar] f|\varphi_0(x)|\right\|^2_{\mathcal{H}} &\leq\sum_{k\geq 1}\left|g(\lambda_k)^{2^J}\right|^2|\langle \bU[\pathvar]f,\varphi_k\rangle_\mathcal{H}|^2\\
     &\leq g(\lambda_1)^{2^{J+1}} \sum_{k\geq 1}|\langle \bU[\pathvar]f,\varphi_k\rangle_\mathcal{H}|^2\\
     &\leq g(\lambda_1)^{2^{J+1}}\|\bU[\pathvar]f\|_\mathcal{H}^2.
    \end{align*}
    Since $\lambda_1>0,$ \eqref{eqn: g decreasing} implies $g(\lambda_1)<1,$ and so the right-hand side converges to zero as $J\rightarrow\infty$.
    \end{proof}

\section{The Proof of Theorems \getrefnumber{thm: nonexpansive}
 and \getrefnumber{thm: nonexpansive nonwindow}}\label{sec: proof nonexpansive}
 
The proof of Theorem \ref{thm: nonexpansive}  is based on the 
following lemma. 
\begin{lemma}\label{lem: energylevels}
For all $f_1,f_2\in\mathcal{H},$ we have
\begin{equation}\label{eqn: energylevelsgeq}
    \sum_{\pathvar\in\mathcal{J}^m}\|\bU[\pathvar]f_1-\bU[\pathvar]f_2\|_{\mathcal{H}}^2\geq\sum_{\pathvar\in\mathcal{J}^{m+1}}\|\bU[\pathvar]f_1-\bU[\pathvar]f_2\|_{\mathcal{H}}^2+\sum_{\pathvar\in\mathcal{J}^{m}}\|\bS[\pathvar]f_1-\bS[\pathvar]f_2\|_{\mathcal{H}}^2.
\end{equation}
Moreover, for all $f\in\mathcal{H}$
\begin{equation}\label{eqn: energylevelsequals}
    \sum_{\pathvar\in\mathcal{J}^m}\|\bU[\pathvar]f\|_{\mathcal{H}}^2\geq\sum_{\pathvar\in\mathcal{J}^{m+1}}\|\bU[\pathvar]f\|_{\mathcal{H}}^2+\sum_{\pathvar\in\mathcal{J}^{m}}\|\bS[\pathvar]f\|_{\mathcal{H}}^2.
\end{equation}

\end{lemma}

\begin{proof}[The Proof of Lemma \ref{lem: energylevels}] The assumption \eqref{eqn: frameAB} implies that for all $\pathvar\in\mathcal{J}^m$
we have that 
\begin{align*}
 \|\bU[\pathvar]f_1-\bU[\pathvar]f_2\|_\mathcal{H}^2&\geq\sum_{j_{m+1}\in\mathcal{J}}   \|W_{j_{m+1}}(\bU[\pathvar]f_1-\bU[\pathvar]f_2)\|_\mathcal{H}^2+\|A(\bU[\pathvar]f_1-\bU[\pathvar]f_2)\|_\mathcal{H}^2.
\end{align*}
Therefore,
\begin{align}
 &\sum_{\pathvar\in\mathcal{J}^m}\|\bU[\pathvar]f_1-\bU[\pathvar]f_2\|_{\mathcal{H}}^2\nonumber\\
 &\geq \sum_{\pathvar\in\mathcal{J}^m}\left(\sum_{j_{m+1}\in\mathcal{J}}   \|W_{j_{m+1}}(\bU[\pathvar]f_1-\bU[\pathvar]f_2)\|_{\mathcal{H}}^2+\|A(\bU[\pathvar]f_1-\bU[\pathvar]f_2)\|_{\mathcal{H}}^2\right)\label{eqn: theothergeq}\\
 &=   \sum_{\pathvar\in\mathcal{J}^m}\left(\sum_{j_{m+1}\in\mathcal{J}}   \|W_{j_{m+1}}\bU[\pathvar]f_1-W_{j_{m+1}}\bU[\pathvar]f_2\|_{\mathcal{H}}^2+\|A(\bU[\pathvar]f_1-\bU[\pathvar]f_2)\|_{\mathcal{H}}^2\right)\nonumber\\
 &\geq   \sum_{\pathvar\in\mathcal{J}^m}\left(\sum_{j_{m+1}\in\mathcal{J}}   \|\sigma W_{j_{m+1}}\bU[\pathvar]f_1-\sigma W_{j_{m+1}}\bU[\pathvar]f_2\|_{\mathcal{H}}^2+\|A(\bU[\pathvar] f_1-\bU[\pathvar]f_2)\|_{\mathcal{H}}^2\right)\label{eqn: thegeq}\\
 &=   \sum_{\pathvar\in\mathcal{J}^m}\left(\sum_{j_{m+1}\in\mathcal{J}}   \|\bU[j_{m+1}]\bU[\pathvar]f_1-\bU[j_{m+1}]\bU[\pathvar]f_2\|_{\mathcal{H}}^2+\|A(\bU[\pathvar]f_1-\bU[\pathvar]f_2)\|_{\mathcal{H}}^2\right)\nonumber\\
 &= \sum_{\pathvar\in\mathcal{J}^{m+1}}\|\bU[\pathvar]f_1-\bU[\pathvar]f_2\|_{\mathcal{H}}^2+\sum_{\pathvar\in\mathcal{J}^{m}}\|\bS[\pathvar]f_1-\bS[\pathvar]f_2\|_{\mathcal{H}}^2\nonumber.
 \end{align}
 This completes the proof of \eqref{eqn: energylevelsgeq}. \eqref{eqn: energylevelsequals} follows from setting $f_2=0$ and noting that in this case equality holds in \eqref{eqn: theothergeq} and \eqref{eqn: thegeq}. 
\end{proof}

\begin{proof}[Proof of Theorem \ref{thm: nonexpansive}] Applying Lemma \ref{lem: energylevels}, and recalling that $U[\pathvar_e]f=f,$ we see

\begin{align*}
    \|\bS f_1-\bS f_2\|_{\ell^2({\mathcal{H}})}^2&=\lim_{N\rightarrow\infty}\sum_{m=0}^N\sum_{\pathvar\in\mathcal{J}^m}\|\bS[\pathvar] f_1-\bS[\pathvar] f_2\|_{\mathcal{H}}^2\\
    &\leq \lim_{N\rightarrow\infty}\sum_{m=0}^N  \left(\sum_{\pathvar\in\mathcal{J}^m}\|\bU[\pathvar] f_1-\bU[\pathvar] f_2\|_{\mathcal{H}}^2-\sum_{\pathvar\in\mathcal{J}^{m+1}}\|\bU[\pathvar] f_1-\bU[\pathvar] f_2\|_{\mathcal{H}}^2\right) \\
    &\leq   \| f_1- f_2\|_{\mathcal{H}}^2-\limsup_{N\rightarrow\infty}\sum_{\pathvar\in\mathcal{J}^{N+1}}\|\bU[\pathvar] f_1-\bU[\pathvar] f_2\|_{\mathcal{H}}^2 \\
&\leq   \| f_1- f_2\|_{\mathcal{H}}^2.
\end{align*} 

\end{proof}

\begin{proof}[The Proof of Theorem \ref{thm: nonexpansive nonwindow}] 
Note that 
\begin{equation*}
    \left\|\frac{|\bS_J[\pathvar] f_i|}{|\varphi_0|}-\overline\bS[\pathvar]f_i\right\|_\mathcal{H}\leq \frac{1}{\min_x|\varphi_0(x)|}\||\bS_J[\pathvar]f_i|-\overline{\bS}[\pathvar]f_i|\varphi_0|\|_\mathcal{H}.
\end{equation*}
Therefore, by Proposition \ref{prop: Jlimit},
\begin{equation*}
    \lim_{J\rightarrow \infty}\left\|\frac{|\bS_J[\pathvar] f_i|}{|\varphi_0|}-\overline\bS[\pathvar]f_i\right\|_\mathcal{H}=0
\end{equation*}
which in turn implies that 
\begin{equation*}\lim_{J\rightarrow \infty}\left\|\frac{|\bS_J[\pathvar] f_i|}{|\varphi_0|}\right\|_\mathcal{H}=\overline\bS[\pathvar]f_i\text{vol}(\mathcal{X})^{1/2}.
\end{equation*}
Thus, using Fatou's lemma, we have
\begin{align*}\|\overline{\bS}f_1-\overline{\bS}f_2\|_2^2&=\sum_\pathvar |\overline{\bS}[\pathvar]f_1-\overline{\bS}[\pathvar]f_2|^2\\
&=\frac{1}{\text{vol}(\mathcal{X})}\sum_\pathvar\lim_{J\rightarrow\infty}\left|\left\|\frac{|\bS_J[\pathvar]f_1|}{|\varphi_0|}\right\|_\mathcal{H}-\left\|\frac{|\bS_J[\pathvar]f_2|}{|\varphi_0|}\right\|_\mathcal{H}\right|^2\\
&\leq \frac{1}{\text{vol}(\mathcal{X})}\liminf_{J\rightarrow\infty}\sum_\pathvar\left|\left\|\frac{|\bS_J[\pathvar]f_1|}{|\varphi_0|}\right\|_\mathcal{H}-\left\|\frac{|\bS_J[\pathvar]f_2|}{|\varphi_0|}\right\|_\mathcal{H}\right|^2\\
&\leq \frac{1}{\text{vol}(\mathcal{X})}\liminf_{J\rightarrow\infty}\sum_\pathvar\left\|\frac{\bS_J[\pathvar]f_1-\bS_J[\pathvar]f_2}{\varphi_0}\right\|_\mathcal{H}^2\\
&\leq \frac{1}{\min_x|\varphi_0(x)|^2\text{vol}(\mathcal{X})}\liminf_{J\rightarrow\infty}\sum_\pathvar\left\|\bS_J[\pathvar]f_1-\bS_J[\pathvar]f_2\right\|_\mathcal{H}^2\\
&\leq \frac{1}{\min_x|\varphi_0(x)|^2\text{vol}(\mathcal{X})}\|f_1-f_2\|_\mathcal{H},
\end{align*}
where in the last line we applied Theorem \ref{thm: nonexpansive}.
\end{proof}

\section{The Proof of Theorem \getrefnumber{thm: equivariance}}\label{sec: proof of equivariance}
\begin{proof}

Since $A_j=H^{2^J}$ and $W_j=H^{2^{j-1}}-H^{2^j} $ it follows from \eqref{eqn: heat commmutes} that\begin{align}
A^{(\zeta)}V_\zeta f=V_\zeta A f,\quad \text{and}\quad  W^{(\zeta)}_jV_\zeta f=V_\zeta W_jf. \end{align}
By definition, for all $\zeta\in\mathcal{G}$, we have that  $V_\zeta$ commutes with $\sigma$ since 
\begin{equation*}
    (V_\zeta \sigma f)(x)=(\sigma f)(\zeta^{-1}(x))=\sigma(f(\zeta^{-1}(x))=\sigma(V_\zeta f(x))=(\sigma V_\zeta f)(x).
\end{equation*}
Therefore, since by definition,
$\bU[\pathvar]f=\sigma W_{j_m}f
\ldots \sigma W_{j_1}f$
it follows that $\bU^{(\zeta)} V_\zeta f = V_\zeta \bU f.$ Lastly since $\bS=A\bU$, we have 
\begin{equation*}
    \bS^{(\zeta)} V_\zeta f  = A^{(\zeta)}\bU^{(\zeta)} V_\zeta f = A^{(\zeta)}V_\zeta\bU  f =  V_\zeta A\bU  f = V_\zeta\bS  f. 
\end{equation*}
\end{proof}

\section{The Proof of Theorem \getrefnumber{thm: Jlimit}} \label{sec: proof of theorem Jlimit}
In this Section, we prove both Theorem \ref{thm: Jlimit} and Lemma \ref{thm: window invariance}.
\begin{proof}[Proof of Lemma \ref{thm: window invariance}]
We first note that, under the assumption that $\mathcal{G}$ preserves the measure $\mu$, the Hilbert spaces $\mathcal{H}$ and $\mathcal{H}^{(\zeta)}$ have the same elements and so 
the subtraction $\bS f-\bS V_\zeta f$ is well defined. Similarly, 
we may identify $V_\zeta$ with an operator mapping $\mathcal{H}$ into itself. Therefore, by the assumption \eqref{eqn: commutes}, we have
\begin{align*}
\|\bS f-\bS^{(\zeta)} V_\zeta f\|_{\ell^2(\mathcal{H})} &= \|A\bU f- V_\zeta Sf\|_{\ell^2(\mathcal{H})} \\&= \|A\bU f- V_\zeta A\bU f\|_{\ell^2(\mathcal{H})}
\\&\leq \|V_\zeta A-A\|_\mathcal{H} \|\bU f\|_{\ell^2(\mathcal{H})}.
\end{align*}
\end{proof}

\begin{proof}[Proof of Theorem \ref{thm: Jlimit}] Let $\zeta \in\mathcal{G}$. The assumption that $\varphi_0$ is constant implies that $V_\zeta \varphi_0-\varphi_0=0$. Therefore, 
    \begin{align*}
    \|V_\zeta A_Jf - A_Jf\|_{\mathcal{H}}&=\left\|\sum_{k\in\mathcal{I}} g(\lambda_k)^{2^J}\langle f,\varphi_k\rangle_{\mathcal{H}}(V_\zeta\varphi_k-\varphi_k)\right\|_{\mathcal{H}}\\
    &=\left\|\sum_{k\geq 1}g(\lambda_k)^{2^J}\langle f,\varphi_k\rangle_{\mathcal{H}}(V_\zeta\varphi_k-\varphi_k)\right\|_{\mathcal{H}}\\
    &\leq \left\|\sum_{k\geq 1}g(\lambda_k)^{2^J}\langle f,\varphi_k\rangle_{\mathcal{H}}V_\zeta\varphi_k\right\|_{\mathcal{H}}+\left\|\sum_{k\geq 1}g(\lambda_k)^{2^J}\langle f,\varphi_k\rangle_{\mathcal{H}}\varphi_k\right\|_{\mathcal{H}}.
   \end{align*}
 The assumption that $\mathcal{G}$ preserves inner products together with the assumption that it preserves the measure implies that for $f,g\in\mathcal{H}$ 
 \begin{align*}\langle V_\zeta f,V_\zeta g\rangle_{\mathcal{H}}=\int_X V_\zeta f \overline{V_\zeta g} d\mu=\int_X V_\zeta f \overline{V_\zeta g} d\mu^{(\zeta)}=\langle V_\zeta f ,V_\zeta g \rangle_{\mathcal{H}^{(\zeta)}}=\langle f,g\rangle_{\mathcal{H}}
     .
 \end{align*}
 Thus, $\{V_\zeta \varphi_k\}_{k\in\mathcal{I}}$ forms an orthonormal basis for $\mathcal{H}$, and so applying Parseval's identity together with the assumption that $g$ is decreasing implies 
     \begin{align*}
    \|V_\zeta Af - A_Jf\|_{\mathcal{H}}&\leq  2 |g(\lambda_1)|^{2^J}\|f\|_{\mathcal{H}}.
    \end{align*}
    Therefore, the result now follows from Lemma \ref{thm: window invariance}.

\end{proof}

\section{The Proof of Theorem \getrefnumber{thm: wavelet stability}}\label{sec: proof of wavelet stability}
 \begin{proof}

 Recall, from \eqref{eqn: Hbar} the decomposition \begin{equation*}
     H=\widetilde{H} + \overline{H},\quad H'=\widetilde{H}'+\overline{H}'.
 \end{equation*}

The operator $\widetilde{H}$ projects a function onto the zero eigenspace $\text{span}(\varphi_0)$ and the operator $\overline{H}$ maps a function into its orthogonal complement  $\text{span}(\varphi_0)^\perp$. Therefore, we have $\widetilde{H}\overline{H}=\overline{H}\widetilde{H}=0$, and we also have $\widetilde{H}^{2^j}=\widetilde{H}$ for all $j\geq 0$. Therefore, 
\begin{equation*}
    H^{2^j}=\widetilde{H}+\overline{H}^{2^j},
\end{equation*}
which implies
\begin{equation*}
    H^{2^{j+1}}-H^{2^{j}}=\overline{H}^{2^{j+1}}-\overline{H}^{2^{j}}
\end{equation*}
(with similar equations holding  for $H'$ and $\overline{H}'$).
Therefore,

\begin{align}\nonumber
    &\|\mathcal{W}_{J}-\mathcal{W}'_{J}\|^2_{\ell^2(\mathcal{H})}\\\nonumber
    \leq& \|H^{2^{J}}-(H')^{2^{J}}\|_\mathcal{H}^2 + \sum_{j=0}^{J-1} \|H^{2^j}-H^{2^{j+1}}-\left((H')^{2^j}-(H')^{2^{j+1}}\right)\|_\mathcal{H}^2+\|H-H'\|_\mathcal{H}^2\\\nonumber
    =&\|\widetilde{H}-\widetilde{H}'+(\overline{H})^{2^{J}}-(\overline{H})'^{2^{J}}\|_\mathcal{H}^2+ \sum_{j=0}^{J-1} \|(\overline{H})^{2^j}-(\overline{H})^{2^{j+1}}-\left((\overline{H}')^{2^j}-(\overline{H}')^{2^{j+1}}\right)\|_\mathcal{H}^2\nonumber\\&\qquad\qquad+\|\widetilde{H}-\widetilde{H}'+\overline{H}-\overline{H}'\|_\mathcal{H}^2\nonumber \\\nonumber
    \leq&4\|\widetilde{H}-\widetilde{H}'\|_\mathcal{H}^2+2\|(\overline{H})^{2^{J}}-(\overline{H}')^{2^{J}}\|_\mathcal{H}^2 \\\nonumber&\qquad\qquad+ 2\sum_{j=0}^{J-1} \|(\overline{H})^{2^j}-(\overline{H}')^{2^j}\|_\mathcal{H}^2
    +2\sum_{j=0}^{J-1} \|(\overline{H}')^{2^{j+1}}-(\overline{H})^{2^{j+1}}\|_\mathcal{H}^2+2\|\overline{H}-\overline{H}'\|_\mathcal{H}^2
     \\ \leq&4\left(\|\widetilde{H}-\widetilde{H}'\|_\mathcal{H}^2+ \sum_{j=0}^{J} \|(\overline{H})^{2^j}-(\overline{H}')^{2^j}\|_\mathcal{H}^2\right).\label{eqn: Qsum with bars1}
\end{align}

The following Lemma is a variant of Eq.\ (23) in \cite{gama:diffScatGraphs2018} (see also Lemma L.1 of \cite{perlmutter2019understanding}). 

\begin{lemma}\label{lem: infiniteqsum}
Let $\beta = \max \left\{ \|\overline{H}\|_{\mathcal{H}},\|\overline{H}'\|_\mathcal{H} \right\}$  and assume that   $\beta<1.$ Then 
\begin{equation*}
    \sum_{j=0}^{J}\left\|\overline{H}^{2^{j}}-(\overline{H}')^{2^j}\right\|^2_\mathcal{H} \leq C_0(\beta)\left\|\overline{H}-\overline{H}'\right\|^2_\mathcal{H}, 
\end{equation*}
where $C_0(\beta)\coloneqq \frac{\beta^2 + 1}{(1-\beta^2)^3}$.
\end{lemma}
\begin{proof}
Letting $A_j(t)=(t\overline{H}+(1-t)\overline{H}')^{2^j},$ we may check that 
\begin{equation*}\left\|\overline{H}^{2^{j}}-(\overline{H}')^{2^j}\right\|_\mathcal{H}=\|A_j(1)-A_j(0)\|_\mathcal{H}\leq \int_0^1\|A'_j(t)\|_\mathcal{H} dt \leq \sup_{0\leq t\leq 1}\|A'_j(t)\|_\mathcal{H} dt.
\end{equation*}
Since, 
\begin{equation*}
    A_j'(t) = \sum_{\ell=0}^{2^j-1} \left(t\overline{H}+(1-t)\overline{H}'\right)^\ell  \left(\overline{H} - \overline{H}' \right) \left(t\overline{H}+(1-t)\overline{H}'\right)^{2^j-\ell-1},
\end{equation*}
and
$\|\overline{H}\|_\mathcal{H},\|\overline{H}'\|_\mathcal{H}\leq \beta,$ this implies 
\begin{equation*}
    \|A_j'(t)\|_2\leq 2^j\beta^{2^j-1}\left\|\overline{H}-\overline{H}'\right\|_\mathcal{H}.
\end{equation*}
Therefore,
\begin{equation*} \sum_{j=0}^{J+1}\left\|\overline{H}^{2^{j}}-(\overline{H}')^{2^j}\right\|^2_\mathcal{H} \leq \sum_{j=0}^\infty (2^j\beta^{2^j-1})^2 \left\|\overline{H}-\overline{H}'\right\|^2_\mathcal{H} \eqqcolon C_0(\beta)\left\|\overline{H}-\overline{H}'\right\|^2_\mathcal{H}.\end{equation*}Lastly, one may compute
\begin{equation*}
    C_0(\beta)=\sum_{j=0}^\infty (2^j\beta^{2^j-1})^2=\beta^{-2}\sum_{j=0}^\infty (2^j\beta^{2^j})^2\leq \beta^{-2}\sum_{n=0}^\infty n^2 \beta^{2n}=\beta^{-2}\frac{\beta^2(\beta^2+1)}{(1-\beta^2)^3}=\frac{\beta^2+1}{(1-\beta^2)^3},
\end{equation*}
where we used the Taylor expansion $\frac{x^2(x^2+1)}{(1-x^2)^3}=\sum_{n=0}^\infty n^2x^{2n}$.
\end{proof}
Returning to the proof of the theorem, we note that by the triangle inequality we have 
\begin{equation*}
    \|\overline{H}-\overline{H}'\|_\mathcal{H}\leq \|H-H'\|_\mathcal{H}+\|\widetilde{H}-\widetilde{H}'\|_\mathcal{H}.
\end{equation*} 
Therefore, combining \eqref{eqn: Qsum with bars1} with Lemma \ref{lem: infiniteqsum}, and using the fact that $(a+b)^2\leq 2(a^2+b^2)$ for all $a,b\in\mathbb{R}$, we have
\begin{equation}\label{eqn: Qsum with bars}
    \|\mathcal{W}_{J}-\mathcal{W}'_{J}\|^2_{\ell^2(\mathcal{H})}  \leq C(\beta)\left(\|\widetilde{H}-\widetilde{H}'\|^2_\mathcal{H}+\|H-H'\|^2_\mathcal{H}\right),
\end{equation}
where $C(\beta)=C C_0(\beta)$ for some absolute constant $C$.

To estimate 
$\|\widetilde{H}-\widetilde{H}'\|^2_\mathcal{H}$, we note that for all $f\in\mathcal{H}$ we have 
\begin{align*}
    \|\widetilde{H}f-\widetilde{H}'f\|_\mathcal{H}&=\|\langle f,\varphi_0\rangle_{\mathcal{H}}\varphi_0-\langle f,\varphi'_0\rangle_{\mathcal{H}'}\varphi'_0\|_{\mathcal{H}}\\
    &\leq \|\langle f,\varphi_0-\varphi'_0\rangle_{\mathcal{H}}\varphi_0\|_\mathcal{H}+\|\langle f, \varphi_0'\rangle_{\mathcal{H}}(\varphi_0-\varphi_0')\|_\mathcal{H}+\| |\langle f,\varphi_0'\rangle_{\mathcal{H}}-\langle f,\varphi_0'\rangle_{\mathcal{H}'}|\varphi'_0\|_{\mathcal{H}}\\
    &\leq \|\varphi_0-\varphi_0'\|_\mathcal{H}\|f\|_\mathcal{H}
    +\|\varphi_0-\varphi_0'\|_\mathcal{H}\|\varphi_0'\|_\mathcal{H}\|f\|_\mathcal{H} +|\langle f,\varphi_0'\rangle_{\mathcal{H}}-\langle f,\varphi_0'\rangle_{\mathcal{H}'}|\|\varphi_0'\|_\mathcal{H}.
\end{align*} 
By \eqref{eqn: change of measure change of norm} and by 
\eqref{eqn: change of inner product}
\begin{equation*}
    \|\varphi_0'\|_\mathcal{H}\leq R(\mathcal{H},\mathcal{H}')^{1/2}\quad \text{and}\quad |\langle f,\varphi_0'\rangle_{\mathcal{H}}-\langle f,\varphi_0'\rangle_{\mathcal{H}'}|\leq \kappa(\mathcal{H},\mathcal{H}')R(\mathcal{H},\mathcal{H}')^{1/2}\|f\|_\mathcal{H}.
\end{equation*}
Therefore, we have 
\begin{equation*}
    \|\widetilde{H}-\widetilde{H}'\|_\mathcal{H} \leq \|\varphi_0-\varphi_0'\|_\mathcal{H}(1+R(\mathcal{H},\mathcal{H}')^{1/2}) + R(\mathcal{H},\mathcal{H}')\kappa(\mathcal{H},\mathcal{H}').
\end{equation*}
Thus, 
\begin{equation*}
    \|\mathcal{W}_J-\mathcal{W}_{J}'\|_{\ell^2(\mathcal{H})}^2 \leq C(\beta) \left[\|\varphi_0-\varphi_0'\|^2_\mathcal{H} R(\mathcal{H},\mathcal{H}') + R(\mathcal{H},\mathcal{H}')^2\kappa(\mathcal{H},\mathcal{H}')^2 +\|H-H'\|_\mathcal{H}^2\right]
\end{equation*}
as desired.
 \end{proof}

\section{The Proof of Theorem \getrefnumber{thm: windowstability}} \label{sec: proof of scattering stability}
In order to prove Theorem \ref{thm: windowstability}, we will need the following Lemma.

\begin{lemma}\label{lem: windowstability}

Under the assumptions of Theorem \ref{thm: windowstability},  we have
\begin{equation}\label{eqn: scatteringstabilitynopermwindow}
    \left\|\bS^\ell f-(\bS^\ell)'f\right\|_{\ell^2(\mathcal{H})}\leq  \sqrt{2}\|\cW-\cW'\|_{\mathcal{H}}\left(\sum_{k=0}^\ell\|\cW'\|_{\mathcal{H}}^{k}\right)\|f\|_{\mathcal{H}}\quad\text{for all } f\in\mathcal{H}.
\end{equation}  
\end{lemma} 
Before proving Lemma \ref{lem: windowstability}, we will show how it is used to prove Theorem \ref{thm: windowstability}.

\begin{proof}[Proof of Theorem \ref{thm: windowstability}] Let $\zeta\in\mathcal{G}$, and let $\mathcal{X}_\zeta$ be defined as in \eqref{eqn: Xzeta}.
By \eqref{eqn: permutation assumption}, we have $V_\zeta\bS^\ell f=\bS^{\ell,(\zeta)}V_\zeta f$. 
Therefore, the triangle inequality implies 
\begin{equation}\label{eqn: threeparts}
    \left\|\bS^\ell f-(\bS^\ell)'\widetilde{f}\right\|_{\ell^2(\mathcal{H})}\leq
    \left\|\bS^\ell f-V_\zeta\bS^\ell f\right\|_{\ell^2(\mathcal{H})}+\|\bS^{\ell,(\zeta)}V_\zeta f-\bS^{\ell,(\zeta)}\widetilde{f}\|_{\ell^2(\mathcal{H})}+\|\bS^{\ell,(\zeta)}\widetilde{f}-(\bS^\ell)'\widetilde{f}\|_{\ell^2(\mathcal{H})}.
\end{equation} The assumption \eqref{eqn: permutation assumption} also implies that 
\begin{equation}
    \|\bS^\ell f-V_\zeta\bS^\ell f\|_{\ell^2(\mathcal{H})}\leq \mathcal{B}\|f\|_\mathcal{H}.\label{eqn: assumption term}
\end{equation}
Similarly, by Theorem \ref{thm: nonexpansive} and \eqref{eqn: change of measure change of norm}, we have that 
\begin{align}
    \|\bS^{\ell,(\zeta)}V_\zeta f-\bS^{\ell,(\zeta)}\widetilde{f}\|_{\ell^2(\mathcal{H})}   &\leq R\left(\mathcal{H},\mathcal{H}^{(\zeta)}\right)^{1/2}\|\bS^{\ell,(\zeta)}V_\zeta f-\bS^{\ell,(\zeta)}\widetilde{f}\|_{\ell^2(\mathcal{H}^{(\zeta)})}  \nonumber\\ &\leq R\left(\mathcal{H},\mathcal{H}^{(\zeta)}\right)^{1/2}\|V_\zeta f-\widetilde{f}\|_{\mathcal{H}^{(\zeta)}}\nonumber\\&\leq R\left(\mathcal{H},\mathcal{H}^{(\zeta)}\right)\|V_\zeta f-\widetilde{f}\|_{\mathcal{H}}.\label{eqn: nonexpansiveterm}
\end{align}
Applying Lemma \ref{lem: windowstability} and \eqref{eqn: change of measure change of norm} yields
\begin{align*}
\|\bS^{\ell,(\zeta)}\widetilde{f}-(\bS^\ell)'\widetilde{f}\|_{\ell^2(\mathcal{H})}
&\leq R\left(\mathcal{H},\mathcal{H}^{(\zeta)}\right)^{1/2}
\|\bS^{\ell,(\zeta)}\widetilde{f}-(\bS^\ell)'\widetilde{f}\|_{\ell^2(\mathcal{H}^{(\zeta)})}\\
&\leq\sqrt{2}R\left(\mathcal{H},\mathcal{H}^{(\zeta)}\right)^{1/2}\|\cW^{(\zeta)}-\cW'\|_{\mathcal{H}^{(\zeta)}}\left(\sum_{k=0}^\ell\|\cW'\|_{\mathcal{H}^{(\zeta)}}^{k}\right)\|\widetilde{f}\|_{\mathcal{H}^{(\zeta)}}\\
&\leq\sqrt{2}R\left(\mathcal{H},\mathcal{H}^{(\zeta)}\right)\|\cW^{(\zeta)}-\cW'\|_{\mathcal{H}^{(\zeta)}}\left(\sum_{k=0}^\ell\|\cW'\|_{\mathcal{H}^{(\zeta)}}^{k}\right)\|\widetilde{f}\|_{\mathcal{H}}.\label{eqn: fourpointtwoterm}
\end{align*}Thus, infimizing over $\zeta$ completes the proof.
\end{proof}

\begin{proof}[The Proof of Lemma \ref{lem: windowstability}]

Let $\cA\coloneqq\|\cW- \cW'\|_{\mathcal{H}}$ and $\mathcal{C}\coloneqq\|\cW'\|_{\mathcal{H}}.$ 

To prove \eqref{eqn: scatteringstabilitynopermwindow}, we need to  show
\begin{equation}\label{eqn: korder}
    \sum_{\pathvar\in\mathcal{J}^\ell}\|\bS[\pathvar]f-\bS'[\pathvar]f\|^2_{\mathcal{H}}\leq 2\mathcal{A}^2\cdot\left(\sum_{k=0}^\ell\mathcal{C}^{k}\right)^2\|f\|_{\mathcal{H}}^2.
\end{equation} 
For $\ell=0$, we recall from \eqref{eqn: empty path} that the zeroth-order windowed scattering coefficient of $f$
is given by 
    $\bS[\pathvar_{e}]f=Af$,
where $\pathvar_{e}$ is the empty-index.
Therefore, by the definition of $\cA$ we have
\begin{equation*}
    \sum_{\pathvar\in\mathcal{J}^0}\|\bS[\pathvar]f-\bS'[\pathvar]f\|_{\mathcal{H}}^2    =\|Af-A'f\|^2_{\mathcal{H}} \leq\|\mathcal{W}f-\mathcal{W}'f\|_{\mathcal{H}}^2\leq \mathcal{A}^2\|f\|_{\mathcal{H}}^2,
\end{equation*}
and so  \eqref{eqn: korder} holds when $\ell=0.$ For the case where $\ell\geq 1,$ 
we note that for all $\pathvar\in\mathcal{J}^\ell,$  we have 
\begin{align*}
    \|\bS[\pathvar]f-\bS'[\pathvar]f\|_{\mathcal{H}}&=    \|A \bU[\pathvar]f-A' \bU'[\pathvar]f\|_{\mathcal{H}}\\
    &\leq \|(A-A')\bU[\pathvar]f\|_{\mathcal{H}} +  \|A'\bU[\pathvar]f-A'\bU'[\pathvar]f\|_{\mathcal{H}}\\
    &\leq \|A-A'\|_{\mathcal{H}}\|\bU[\pathvar]f\|_{\mathcal{H}} +  \|A'\|_{\mathcal{H}}\|\bU[\pathvar]f-\bU'[\pathvar]f\|_{\mathcal{H}},
      \end{align*}
and so using the fact that $(a+b)^2 \leq 2a^2 + 2b^2$ and summing over $\pathvar$ implies
\begin{align*}    
    \sum_{\pathvar\in\mathcal{J}^\ell}\|\bS[\pathvar]f-\bS'[\pathvar]f\|_{\mathcal{H}}^2
    &\leq2\|A-A'\|_{\mathcal{H}}^2\sum_{\pathvar\in\mathcal{J}^\ell}\|\bU[\pathvar]f\|_{\mathcal{H}}^2 +2\|A'\|_{\mathcal{H}}^2\sum_{\pathvar\in\mathcal{J}^\ell}  \|\bU[\pathvar]f-\bU'[\pathvar]f\|_{\mathcal{H}}^2. 
\end{align*}
Therefore, \eqref{eqn: korder} and thus \eqref{eqn: scatteringstabilitynopermwindow}, follow from applying Lemma \ref{lem: nonexpansiveU} stated below, noting that $\|A-A'\|_{\mathcal{H}}^2\leq \cA^2$ and $\|A'\|_{\mathcal{H}}^2\leq \mathcal{C}^2$, and using the fact that $a^2 + b^2 \leq (a + b)^2$ when $a,b \geq 0$.

\end{proof}

\begin{lemma}\label{lem: nonexpansiveU}
Let $\mathcal{A}\coloneqq\|\mathcal{W}-\mathcal{W}'\|_{\mathcal{H}}$ and $\mathcal{C}\coloneqq \|\mathcal{W}'\|_{\mathcal{H}}$. Then, for all $\ell\geq 1$,
\begin{align*}
    \sum_{\pathvar\in\mathcal{J}^\ell}\|\bU[\pathvar]f\|^2_{\mathcal{H}}   &\leq \|f\|^2_{\mathcal{H}},\quad\text{and}\quad 
    \sum_{\pathvar\in\mathcal{J}^\ell}\|\bU[\pathvar]f-\bU'[\pathvar]f\|^2_{\mathcal{H}} \leq \mathcal{A}^2 \left(\sum_{k=0}^{\ell-1}\mathcal{C}^{k}\right)^2\|f\|^2_{\mathcal{H}}.
\end{align*}
\end{lemma}

\begin{proof}
When $\ell=1,$ the first inequality follows immediately from \eqref{eqn: frameAB} and the fact that $\sigma$ is nonexpansive. Now, suppose by induction that the first inequality holds for $\ell.$ Let $f\in\mathcal{H}.$ Then
\begin{align}
    \sum_{\pathvar\in\mathcal{J}^{\ell+1}}\|\bU[\pathvar]f\|^2_{\mathcal{H}}  &= \sum_{\pathvar\in\mathcal{J}^{\ell+1}}\|\sigma W_{j_{\ell+1}}\cdots \sigma W_{j_1}f\|^2_{\mathcal{H}} \nonumber \\
    &=    \sum_{\pathvar\in\mathcal{J}^{\ell}}\left(\sum_{j_{\ell+1}\in\mathcal{J}}\|\sigma W_{j_{\ell+1}}(\sigma W_{j_{\ell}}\cdots \sigma W_{j_1}f)\|^2_{\mathcal{H}}\right) \nonumber \\
    &\leq    \sum_{\pathvar\in\mathcal{J}^{\ell}}\|\sigma W_{j_{\ell}}\cdots \sigma W_{j_1}f\|^2_{\mathcal{H}} \nonumber \\
    &\leq \|f\|^2_{\mathcal{H}}, \label{eqn: lemma 4.10 proof}
\end{align}
with the last inequality following from the inductive assumption.

To prove the second inequality, let $t_\ell \coloneqq \left(\sum_{\pathvar\in\mathcal{J}^\ell}\|\bU[\pathvar]f - \bU'[\pathvar]f \|^2_{\mathcal{H}}\right)^{1/2}.$ Since $\sigma$ is nonexpansive, 
the definition of $\cA$ 
implies 
$t_1\leq \cA \|f\|_{\mathcal{H}}$. Now, by induction, suppose  the result holds for $\ell.$ Then, recalling that $\bU[\pathvar]=\sigma W_{j_\ell}\cdots \sigma W_{j_1},$ we have
\begin{align*}
t_{\ell+1}&= \left(    \sum_{\pathvar\in\mathcal{J}^{\ell+1}}\|\sigma W_{j_{\ell+1}}\cdots \sigma W_{j_1}f-\sigma W'_{j_{\ell+1}}\cdots \sigma W'_{j_1}f\|_{\mathcal{H}}^2\right)^{1/2}\\ 
&\leq \left(    \sum_{\pathvar\in\mathcal{J}^{\ell+1}}\|(W_{j_{\ell+1}}-W'_{j_{\ell+1}})\sigma W_{j_{\ell}}\cdots \sigma W_{j_1}f\|_{\mathcal{H}}^2\right)^{1/2} \\
&\quad+\left(    \sum_{\pathvar\in\mathcal{J}^{\ell+1}}\|W'_{j_{\ell+1}}(\sigma W_{j_{i}}\cdots \sigma W_{j_1}f-\sigma W'_{j_{\ell}}\cdots \sigma W'_{j_1}f)\|_{\mathcal{H}}^2\right)^{1/2}\\
&\leq \mathcal{A}\left(    \sum_{\pathvar\in\mathcal{J}^{\ell+1}}\|\sigma W_{j_{\ell}}\cdots \sigma W_{j_1}f\|_{\mathcal{H}}^2\right)^{1/2} \\
&\quad+\mathcal{C}\left(    \sum_{\pathvar\in\mathcal{J}^{\ell}}\|\sigma W_{j_{\ell}}\cdots \sigma W_{j_1}f-\sigma W'_{j_{\ell}}\cdots \sigma W'_{j_1}f\|_{\mathcal{H}}^2\right)^{1/2}\\
&\leq \mathcal{A}\|f\|_{\mathcal{H}} +t_\ell\mathcal{C}\|f\|_{\mathcal{H}}  \end{align*}
by 
the definitions of $\cA$ and $\mathcal{C}$ and by \eqref{eqn: lemma 4.10 proof}. By the inductive hypothesis, we have that 
\begin{equation*} t_\ell\leq \mathcal{A}\sum_{k=0}^{\ell-1}\mathcal{C}^{k}\|f\|_{\mathcal{H}}.
\end{equation*}
Therefore,
\begin{equation*}
    t_{\ell+1}\leq \mathcal{A}\|f\|_{\mathcal{H}}+\mathcal{A}\sum_{k=0}^{\ell-1}\mathcal{C}^{k+1}\|f\|_{\mathcal{H}}=\mathcal{A}\sum_{k=0}^{\ell}\mathcal{C}^{k}\|f\|_{\mathcal{H}}.
\end{equation*}
Squaring both sides completes the proof of the second inequality.
\end{proof}

\section{The Proof of Theorem \getrefnumber{thm: scattering stability no window}}\label{sec: Proof of Stability no window}
 \begin{proof}
Let $\zeta\in\mathcal{G}$, and let $\mathcal{X}_\zeta$ be as in \eqref{eqn: Xzeta}.
By \eqref{eqn: invariance assumptions}, we have $    \overline{\bS^{(\zeta)}}V_\zeta f = \overline{\bS} f$. 
Therefore, we 
may use the definitions of the \rev{non-windowed} scattering transform 
to see that for each path $\pathvar$ we have
\begin{align*}
    &|\overline{\bS}[\pathvar]  f - \overline{\bS'}[\pathvar] \widetilde{f}|\\=& |\overline{\bS^{(\zeta)}}[\pathvar] V_\zeta f - \overline{\bS'}[\pathvar] \widetilde{f}|\\ \leq& | \langle \bU^{(\zeta)}[p]V_\zeta f,\varphi^{(\zeta)}_0\rangle_{\mathcal{H}^{(\zeta)}} - \langle \bU'[p]\widetilde{f},\varphi'_0\rangle_{\mathcal{H}'}|\\
    \leq&|\langle \bU^{(\zeta)}[p]V_\zeta f-\bU'[p]\widetilde{f},\varphi_0^{(\zeta)}\rangle_{\mathcal{H}^{(\zeta)}}|+|\langle \bU'[p]\widetilde{f},\varphi_0^{(\zeta)}-\varphi_0'\rangle_{\mathcal{H}^{(\zeta)}}|+|\langle \bU'[p]\widetilde{f},\varphi_0'\rangle_{\mathcal{H}^{(\zeta)}}-\langle \bU'[p]\widetilde{f},\varphi_0'\rangle_{\mathcal{H}'}|\\
    \eqqcolon&\hspace{.02in}  I[\pathvar] + II[\pathvar] + III[\pathvar].
\end{align*}
To bound $I[\pathvar]$, we use the Cauchy Schwarz inequality to observe
\begin{align*}
    |\langle \bU^{(\zeta)}[p]V_\zeta f-\bU'[p]\widetilde{f},\varphi_0^{(\zeta)}\rangle_{\mathcal{H}^{(\zeta)}}| &\leq |\langle\bU^{(\zeta)}[p]V_\zeta f-\bU^{(\zeta)}[p]\widetilde{f},\varphi_0^{(\zeta)}\rangle_{\mathcal{H}^{(\zeta)}}| + |\langle\bU^{(\zeta)}[p]\widetilde{f}-\bU'[p]\widetilde{f},\varphi_0^{(\zeta)}\rangle_{\mathcal{H}^{(\zeta)}}|\\
    &\leq |\overline{S^{(\zeta)}}[p]V_\zeta f-\overline{S^{(\zeta)}}[p]\widetilde{f}|+\|\bU^{(\zeta)}[p]\widetilde{f}-\bU'[p]\widetilde{f}\|_{\mathcal{H}^{(\zeta)}}.
\end{align*}
 Therefore, applying \eqref{eqn: invariance assumptions} and Lemma \ref{lem: nonexpansiveU} yields 
 \begin{align}\label{eqn: Ip bound}
     \sum_{\pathvar\in\mathcal{J}^\ell} I[\pathvar]^2 &\leq 2C_L
\|V_\zeta f-\tilde{f}\|^2_{\mathcal{H}^{(\zeta)}} + 2\|\mathcal{W}^{(\zeta)}-\mathcal{W}'\|_{\mathcal{H}^{(\zeta)}}^2 \left(\sum_{k=0}^{\ell-1}\|\mathcal{W}'\|_{\mathcal{H}^{(\zeta)}}^k \right)^2\|\tilde{f}\|^2_{\mathcal{H}^{(\zeta)}}. 
 \end{align}
For $II[\pathvar]$, we again use the Cauchy Schwarz inequality and \eqref{eqn: change of measure change of norm} 
to see
\begin{align*}
    |\langle \bU'[p]\widetilde{f},\varphi_0^{(\zeta)}-\varphi_0'\rangle_{\mathcal{H}^{(\zeta)}}|\leq R(\mathcal{H}^{(\zeta)},\mathcal{H}')\|\varphi_0^{(\zeta)}-\varphi_0'\|_{\mathcal{H}'}\|U'[p]\widetilde{f}\|_{\mathcal{H}'}
\end{align*}
Therefore, again applying Lemma \ref{lem: nonexpansiveU} implies 
\begin{equation}\label{eqn: IIp bound}
     \sum_{\pathvar\in\mathcal{J}^\ell} II[\pathvar]^2 \leq R(\mathcal{H}^{(\zeta)},\mathcal{H}')^2\|\varphi_0^{(\zeta)}-\varphi_0'\|^2_{\mathcal{H}'}\|\widetilde{f}\|^2_{\mathcal{H}'}.
\end{equation}
Lastly, to bound $III[\pathvar]$, 
we note that by \eqref{eqn: change of inner product} and the Cauchy Schwarz inequality, we have 
\begin{align*}
    |\langle \bU'[p]\widetilde{f},\varphi_0'\rangle_{\mathcal{H}^{(\zeta)}}-\langle \bU'[p]\widetilde{f},\varphi_0'\rangle_{\mathcal{H}'}| &\leq \kappa(\mathcal{H}',\mathcal{H}^{(\zeta)})\|\bU'[p]\widetilde{f}\|_{\mathcal{H}'}\|\varphi_0'\|_{\mathcal{H}'}\\
    &\leq \kappa(\mathcal{H}',\mathcal{H}^{(\zeta)})\|\bU'[p]\widetilde{f}\|_{\mathcal{H}'},
\end{align*}
and so summing over $\pathvar$ and once more applying Lemma \ref{lem: nonexpansiveU} gives 
\begin{equation*}
     \sum_{\pathvar\in\mathcal{J}^\ell} III[\pathvar]^2 \leq \kappa(\mathcal{H}',\mathcal{H}^{(\zeta)})\|\widetilde{f}\|_{\mathcal{H}'}.
\end{equation*}
Therefore, combining this with \eqref{eqn: Ip bound} and \eqref{eqn: IIp bound} yields 
\begin{align*}
    &\sum_{p\in\mathcal{J}^\ell} |\overline{\bS}[\pathvar]  f - \overline{\bS}'[\pathvar] \widetilde{f}|^2\\ \leq&  3\Bigg( 2C_L
\|V_\zeta f-\tilde{f}\|^2_{\mathcal{H}^{(\zeta)}}  + R(\mathcal{H}^{(\zeta)},\mathcal{H}')^2\|\varphi_0^{(\zeta)}-\varphi_0'\|^2_{\mathcal{H}'}\|\widetilde{f}\|^2_{\mathcal{H}'}\\
    &\qquad\quad+ 2\|\mathcal{W}^{(\zeta)}-\mathcal{W}'\|_{\mathcal{H}^{(\zeta)}}^2 \left(\sum_{k=0}^{\ell-1}\|\mathcal{W}'\|_{\mathcal{H}^{(\zeta)}}^k \right)^2\|\tilde{f}\|^2_{\mathcal{H}^{(\zeta)}}+ \kappa(\mathcal{H}',\mathcal{H}^{(\zeta)})\|\widetilde{f}\|_{\mathcal{H}'}\bigg).
\end{align*}
The result follows by taking the infimum over $\zeta\in\mathcal{G}.$
\end{proof}

\section{The Proof of Lemma \getrefnumber{lem: finite K}}
\label{sec: proof of lemma finite K}
\begin{proof}
By definition, we have
\begin{align}
    &H_t f(x)-H_t^\kappa f(x)\nonumber\\=&\int_{X} h_t(x,y)f(y)d\mu(y)-\int_{X} h^\kappa_t(x,y)f(y)d\mu(y)\nonumber\\
    =&\int_{X} \sum_{k=0}^\infty e^{-t\mu_k}\varphi_k(x)\varphi_k(y)f(y)d\mu(y)-\int_{X} \sum_{k=0}^\kappa e^{-t\mu_k}\varphi_k(x)\varphi_k(y)f(y)d\mu(y)\nonumber\\
    =&\int_{X} \sum_{k=\kappa+1}^\infty e^{-t\mu_k}\varphi_k(x)\varphi_k(y)f(y)d\mu(y)\label{eqn: pwQk}\\
    =&\sum_{k=\kappa+1}^\infty e^{-t\mu_k}\langle\varphi_k,f\rangle_{\mathbf{L}^2(\mathcal{X})}\varphi_k(x).\nonumber
\end{align} Therefore, since the $\varphi_k$ form an orthonormal basis, twice applying   Plancherel's theorem implies that 
\begin{align*}
    \|H_t^\kappa f(x)-H_tf(x)\|^2_{\mathbf{L}^2(\mathcal{X})} &= \sum_{k=\kappa+1}^\infty e^{-2t\mu_k}|\langle\varphi_k,f\rangle_{\mathbf{L}^2(\mathcal{X})}|^2\\
    &\leq e^{-2t\mu_{\kappa+1}} \sum_{k=\kappa+1}^\infty |\langle\varphi_k,f\rangle_{\mathbf{L}^2(\mathcal{X})}|^2\\
    &\leq e^{-2t\mu_{\kappa+1}}\|f\|_{\mathbf{L}^2(\mathcal{X})}^2.
\end{align*}
This completes the proof of \eqref{eqn: finite K}. To prove \eqref{eqn: finite K infinity}, we note that
\eqref{eqn: pwQk}  implies 
\begin{align*}
\|H_t^\kappa f-H_tf\|_\infty &\leq \sup_{x,y\in\mathcal{X}}|\sum_{k=\kappa+1}^\infty e^{-t\mu_k} \varphi_k(x)\varphi_k(y)| \|f\|_\infty.
\end{align*}
In \cite{dunson2021spectral}, the proof of Theorem 3, 
it is shown that $$\sup_{x,y\in\mathcal{X}}|\sum_{k=\kappa+1}^\infty e^{-t\mu_k} \varphi_k(x)\varphi_k(y)|\leq C_\mathcal{X} e^{-C_{\mathcal{X}}'t}\leq C_\mathcal{X}$$ and so the result follows. 
\end{proof}

\section{The Proof of Lemma \getrefnumber{lem: apply Hoeffding} }
\label{sec: proof of Hoeffding}
\begin{proof}
Let $f,g\in\mathcal{C}(\mathcal{X})$, and define random variables $X_i=f(x_i)g(x_i)$. Since the $x_i$ are sampled i.i.d.\ uniformly at random, we have
\begin{equation*}
    \langle \rho f,\rho g \rangle_2 = \frac{1}{N} \sum_{i=0}^{N-1} X_i
\end{equation*}
and \eqref{eqn: normalized} implies
\begin{equation*}
    \mathbb{E} \left(\frac{1}{N} \sum_{i=0}^{N-1} X_i\right) = \langle f,g\rangle_{\mathbf{L}^2(\mathcal{X})}.
\end{equation*}
Therefore, by Hoeffding's inequality, we have 
\begin{align*}
    \mathbb{P}\left(|\langle \rho f,\rho g \rangle_2 - \langle f,g\rangle_{\mathbf{L}^2(\mathcal{X})}| >\eta\right)   &= 
    \mathbb{P}\left(\left|\frac{1}{N} \left(\sum_{i=0}^{N-1} X_i - \mathbb{E}\sum_{i=0}^{N-1} X_i\right)\right| >\eta\right)\\
    &=\mathbb{P}\left(\left| \left(\sum_{i=0}^{N-1} X_i - \mathbb{E}\sum_{i=0}^{N-1} X_i\right)\right| >N\eta\right)
    \\
    &\leq 2\exp\left(\frac{-2N^2\eta^2}{4N\|fg\|^2_\infty}\right)\\
    &=
     2\exp\left(\frac{-N\eta^2}{2\|fg\|^2_\infty}\right).
\end{align*}
The result now follows by setting $\eta = \sqrt{\frac{18 \log N}{N}}\|fg\|_\infty$.
\end{proof}

\section{The Proof of Remark \getrefnumber{rem: alpha rate}}
\label{sec: remark proof}
To see this we note that the term $\alpha_k$ in Theorem 5.4 of \cite{cheng2021eigen} is first introduced in Proposition 5.2. We observe, by Equation (42), that 
\begin{equation*}
    \left|\left|(\mathbf{u}_k^{N,\epsilon})^T\mathbf{v}_k\right|-1\right|=\left|\frac{1}{|\alpha_k|}-1\right|=\mathcal{O}(|\text{Err}_\text{norm}|+\text{Err}^2_{\text{pt}}).
\end{equation*}
(Please see \cite{cheng2021eigen} for the definitions of  $\text{Err}_\text{norm}$ and $\text{Err}^2_{\text{pt}}.)$  
Since $|\alpha_k|$ converges to 1, for sufficiently large $N$, we have $\frac{1}{2}||\alpha_k|-1|\leq \left|\frac{1}{|\alpha_k|}-1\right|\leq 2 ||\alpha_k|-1|$
and therefore, we also have that 
\begin{align*}
    ||\alpha_k|-1|
    &= \mathcal{O}(|\text{Err}_\text{norm}|+\text{Err}^2_{\text{pt}}).
    \end{align*}
Immediately prior to Equation 42, the authors note 
\begin{equation*}
    \text{Err}_\text{norm}=\mathcal{O}\left(\sqrt{\frac{\log(N)}{N}}\right),
\end{equation*}
    and
Equation 40 shows that
\begin{equation*}
    \text{Err}_\text{pt}=\mathcal{O}(\epsilon)+\mathcal{O}\left(\sqrt{\frac{\log(N)}{N\epsilon^{d/2+1}}}\right)
\end{equation*}
In particular, if we set $\epsilon \sim N^{-2/(d+6)}$  we have 
\begin{align*}
    \text{Err}_\text{pt}=\mathcal{O}(N^{-2/(d+6)})+\mathcal{O}\left(\sqrt{\frac{\log(N)}{N^{4/(d+6)}}}\right)=\mathcal{O}\left(\sqrt{\frac{\log(N)}{N^{4/(d+6)}}}\right)
\end{align*}
%

\section{The Proof of Theorem \getrefnumber{thm: heat kernel discretization} and Corollary \getrefnumber{cor: heat kernel discretization}}\label{sec: heat convergence}
\begin{proof}[Proof of Theorem \ref{thm: heat kernel discretization}]
To avoid cumbersome notation, within this proof we will drop explicit dependence on $N$ and $\epsilon$ and simply write $\lambda_k$ in place of $\lambda_k^{N,\epsilon}$.

Let $\tilde{\mathbf{u}}_k=\text{sgn}(\alpha_k)\mathbf{u}_k$ where $\text{sgn}$ is the standard signum function. Then,
\begin{align}
&H_{N,\epsilon,\kappa,t}\rho f-\rho H_t^\kappa f\nonumber\\
=& \sum_{k=0}^\kappa e^{-{\lambda_k}t} \mathbf{u}_k \mathbf{u}_k^T \rho f - \rho \sum_{k=0}^\kappa e^{-\mu_kt}\langle f, \varphi_k\rangle_{\mathbf{L}^2(\mathcal{X})} \varphi_k\nonumber\\
=& \sum_{k=0}^\kappa e^{-{\lambda_k}t} \tilde{\mathbf{u}}_k \tilde{\mathbf{u}}_k^T \rho f - \rho \sum_{k=0}^\kappa e^{-\mu_kt}\langle f, \varphi_k\rangle_{\mathbf{L}^2(\mathcal{X})} \varphi_k\nonumber\\
=& \sum_{k=0}^\kappa e^{-{\lambda_k}t}  \langle \tilde{\mathbf{u}}_k, \rho f\rangle_2 \tilde{\mathbf{u}}_k-  \sum_{k=0}^\kappa e^{-\mu_kt}\langle f, \varphi_k\rangle_{\mathbf{L}^2(\mathcal{X})} \mathbf{v}_k\nonumber\\
=& \sum_{k=0}^\kappa (e^{-{\lambda_k}t} - e^{-\mu_kt}) \langle \tilde{\mathbf{u}}_k, \rho f\rangle_2 \tilde{\mathbf{u}}_k\nonumber\\
&\qquad + \sum_{k=0}^\kappa e^{-\mu_kt} \left(\langle \tilde{\mathbf{u}}_k, \rho f\rangle_2 - \langle f, \varphi_k\rangle_{\mathbf{L}^2(\mathcal{X})} \right)\tilde{\mathbf{u}}_k\nonumber\\
&\qquad\qquad + \sum_{k=0}^\kappa e^{-\mu_kt}\langle f, \varphi_k\rangle_{\mathbf{L}^2(\mathcal{X})} (\tilde{\mathbf{u}}_k-\mathbf{v}_k).\label{eqn: three terms}
\end{align}
Since $|\text{sgn}(\alpha_k)|=1$, $\{\widetilde{\mathbf{u}}_k\}_{k=0}^\kappa$ is an orthonormal basis for the span of $\{\mathbf{u}_k\}_{k=0}^\kappa$.
Therefore, to bound the first of the above terms, we may apply Parseval's theorem to see
\begin{align}
    &\|\sum_{k=0}^\kappa (e^{-{\lambda_k}t} - e^{-\mu_kt}) \langle \tilde{\mathbf{u}}_k, \rho f\rangle_2 \tilde{\mathbf{u}}_k\|_2^2\nonumber\\ =& \sum_{k=0}^\kappa |e^{-{\lambda_k}t} - e^{-\mu_kt}|^2 |\langle \tilde{\mathbf{u}}_k, \rho f\rangle_2|^2\nonumber\\
    \leq& \max_{0\leq k\leq \kappa}|e^{-{\lambda_k}t} - e^{-\mu_kt}|^2\sum_{k=0}^\kappa  |\langle \tilde{\mathbf{u}}_k, \rho f\rangle_2|^2\nonumber\\
    \leq& \max_{0\leq k\leq \kappa}|e^{-{\lambda_k}t} - e^{-\mu_kt}|^2\|\rho f\|_2^2\label{eqn: reduced T1}.
\end{align}
By Theorem \ref{thm: 5.4 of Chen and Wu}, we have 
\begin{align}
\max_{0\leq k\leq \kappa}|e^{-{\lambda_k}t} - e^{-\mu_kt}|&\leq t\max_{0\leq k\leq \kappa}|\lambda_k -\mu_k|\nonumber\\
&= t\mathcal{O}(N^{-2/(d+6)})\label{eqn: exp ev bound}
\end{align}
with probability at least $1-\mathcal{O}(N^{-9})$.

By Lemma \ref{lem: apply Hoeffding} we have
$$\|\rho f\|_2^2\leq \|f\|^2_{\mathbf{L}^2(\mathcal{X})}+\sqrt{\frac{18\log N}{N}}\|f\|^2_\infty $$
with probability at least $1-2/N^9$.
Therefore, combining \eqref{eqn: reduced T1} and \eqref{eqn: exp ev bound}, yields
\begin{align}
    &\left\|\sum_{k=0}^\kappa (e^{-{\lambda_k}t} - e^{-\mu_kt}) \langle \tilde{\mathbf{u}}_k, \rho f\rangle_2 \tilde{\mathbf{u}}_k\right\|_2^2\nonumber\\
    \leq& \max_{0\leq k\leq \kappa}|e^{-{\lambda_k}t} - e^{-\mu_kt}|^2\|\rho f\|_2^2\nonumber\\
    &\leq t^2 \left(\|f\|^2_{\mathbf{L}^2(\mathcal{X})}+\sqrt{\frac{\log N}{N}}\|f\|^2_\infty \right)\mathcal{O}(N^{-4/(d+6)})\label{eqn: first term bound}
\end{align}
with probability at least $1-\mathcal{O}\left(\frac{1}{N^9}\right)$.

To bound the second term from \eqref{eqn: three terms}, we use Parseval's Identity to see  
\begin{align}
    &\left\| \sum_{k=0}^\kappa e^{-\mu_kt} \left(\langle \tilde{\mathbf{u}}_k, \rho f\rangle_2 - \langle f, \varphi_k\rangle_{\mathbf{L}^2(\mathcal{X})} \right) \tilde{\mathbf{u}}_k\right\|_2^2\nonumber\\
    \leq&\sum_{k=0}^\kappa |\langle \tilde{\mathbf{u}}_k, \rho f\rangle_2 - \langle f, \varphi_k\rangle_{\mathbf{L}^2(\mathcal{X})}|^2 \nonumber\\
    \leq&2\sum_{k=0}^\kappa (|\langle \tilde{\mathbf{u}}_k, \rho f\rangle_2 - \langle \mathbf{v}_k, \rho f\rangle_{2}|^2+|\langle \mathbf{v}_k, \rho f\rangle_2 - \langle f, \varphi_k\rangle_{\mathbf{L}^2(\mathcal{X})}|^2) \nonumber\\
    \leq&2
    \sum_{k=0}^\kappa (\|\tilde{\mathbf{u}}_k-\mathbf{v}_k\|_2^2\|\rho f\|_2^2+|\langle \rho\varphi_k, \rho f\rangle_2 - \langle f, \varphi_k\rangle_{\mathbf{L}^2(\mathcal{X})}|^2).\label{eqn: simple second} 
\end{align}
By Remark \ref{rem: alpha rate}, 
  
$$\max\left\{||\alpha_k|-1|,\left|\frac{1}{|\alpha_k|}-1\right|\right\}\leq  \mathcal{O}\left(\sqrt{\frac{\log N}{N}}\right)+\mathcal{O}\left(\frac{\log(N)}{N^{4/(d+6)}}\right).$$ 
  Therefore,
    \begin{equation*}
        \left(\frac{|\alpha_k|-1}{\alpha_k}\right)^2\leq
 \mathcal{O}\left(\frac{\log N}{N}\right)+\mathcal{O}\left(\frac{\log(N)^2}{N^{8/(d+6)}}\right),
    \end{equation*}
and so we may recall the definition of $\tilde{\mathbf{u}}_k$, and use  Theorem \ref{thm: 5.4 of Chen and Wu} to see
\begin{align}
    \|\tilde{\mathbf{u}}_k-\mathbf{v}_k\|_2^2
    &= \|\text{sgn}(\alpha_k)\mathbf{u}_k-\mathbf{v}_k\|_2^2\nonumber\\
    &= \frac{1}{\alpha^2_k}\| |\alpha_k|\mathbf{u}_k-\alpha_k\mathbf{v}_k\|_2^2\nonumber\\
    &\leq \frac{2}{\alpha_k^2}(\|(|\alpha_k|-1)\mathbf{u}_k\|^2+ \|\mathbf{u}_k-\alpha_k\mathbf{v}_k\|_2^2 )\nonumber\\
    &\leq 2\left(\frac{|\alpha_k|-1}{\alpha_k}\right)^2 + \frac{2}{|\alpha_k|^2}\|\mathbf{u}_k-\alpha_k\mathbf{v}_k\|_2^2\nonumber\\
    &\leq \mathcal{O}\left(\frac{\log N}{N}\right)+\mathcal{O}\left(\frac{\log(N)^2}{N^{8/(d+6)}}\right)
 +\mathcal{O}(N^{-\frac{4}{d+6}}\log(N))\nonumber\\&=\mathcal{O}(N^{-\frac{4}{d+6}}\log N)\label{eqn: ev error}.
\end{align}
As noted earlier, by Lemma \ref{lem: apply Hoeffding}, we have 
$$\|\rho f\|_2^2\leq \|f\|^2_{\mathbf{L}^2(\mathcal{X})}+\sqrt{\frac{18\log N}{N}}\|f\|^2_\infty $$
with probability 
at least $1-\frac{2}{N^9}$
and again applying Lemma \ref{lem: apply Hoeffding}
we have 
\begin{equation*}
    |\langle \rho\varphi_k, \rho f\rangle_2 -\langle f, \varphi_k\rangle_{\mathbf{L}^2(\mathcal{X})}|\leq \sqrt{\frac{18\log{N}}{N}}\|f\varphi_k\|_\infty
\end{equation*}
with probability 
at least $1-\frac{2}{N^9}$.

It is known (see, e.g., \cite{shi:gradEigfcnManifold2010}) that $\|\varphi_k\|_\infty\leq C_{\mathcal{X}}\lambda_k^{(d-1)/4}$. Weyl's asymptotic formula (see, e.g., \cite{canzani:analysisManifoldsLaplacian2013} Theorem 72) implies that 
$\lambda_k\leq C_{\mathcal{X}} k^{2/d}$. Therefore,
\begin{equation*}\|\varphi_k\|_\infty\leq C_{\mathcal{X}}k^{\frac{2}{d}\frac{d-1}{4}}=C_{\mathcal{X}}k^{(d-1)/2d}=\mathcal{O}(1),
\end{equation*}
where the final equality uses the fact that the implied constants depend on $\kappa$ and the geometry of $\mathcal{X}$.
Therefore, by \eqref{eqn: simple second},
\begin{align}
    &\| \sum_{k=0}^\kappa e^{-\mu_kt} \left(\langle \tilde{\mathbf{u}}_k, \rho f\rangle_2 - \langle f, \varphi_k\rangle_{\mathbf{L}^2(\mathcal{X})} \right) \tilde{\mathbf{u}}_k\|_2^2\nonumber\\
    &\leq2\sum_{k=0}^\kappa (\|\tilde{\mathbf{u}}_k-\mathbf{v}_k\|_2^2\|\rho f\|_2^2+|\langle \rho\varphi_k, \rho f\rangle_2 - \langle f, \varphi_k\rangle_{\mathbf{L}^2(\mathcal{X})}|^2)\nonumber\\
    &\leq \kappa  \left(\mathcal{O}(N^{-\frac{4}{d+6}}\log N)  \left(\|f\|^2_{\mathbf{L}^2(\mathcal{X})}+\sqrt{\frac{18\log N}{N}}\|f\|^2_\infty \right) +\mathcal{O}\left(\frac{\log{N}}{N}\right)\max_{0\leq k\leq \kappa}\|f\varphi_k\|^2_\infty \right)\nonumber\\
        &\leq \kappa  \left(\mathcal{O}(N^{-\frac{4}{d+6}}\log N)  \left(\|f\|^2_{\mathbf{L}^2(\mathcal{X})}+\sqrt{\frac{\log N}{N}}\|f\|^2_\infty \right) +\mathcal{O}\left(\frac{\log{N}}{N}\right)\kappa^{(d-1)/d}\|f\|^2_\infty \right)\nonumber\\
        &\leq  \mathcal{O}\left(\frac{\log N}{N^{\frac{4}{d+6}}}\right) \|f\|^2_{\mathbf{L}^2(\mathcal{X})} + \left(\mathcal{O}\left(\frac{(\log N)^{3/2}}{N^{\frac{4}{d+6}+\frac{1}{2}}}\right) +\mathcal{O}\left(\frac{\log N}{N}\right)\right)\|f\|^2_\infty.\label{eqn: second term bound}
\end{align}

Finally, to bound the third term in \eqref{eqn: three terms}, we use \eqref{eqn: ev error} to see
\begin{align}
    &\|\sum_{k=0}^\kappa e^{-\mu_kt}\langle f, \varphi_k\rangle_{\mathbf{L}^2(\mathcal{X})} (\tilde{\mathbf{u}}_k-\mathbf{v}_k)\|_2^2\nonumber\\
    \leq & \kappa\sum_{k=0}^\kappa |\langle f, \varphi_k\rangle_{\mathbf{L}^2(\mathcal{X})}|^2\|\tilde{\mathbf{u}}_k-\mathbf{v}_k\|_2^2\nonumber\\
\leq&     \kappa\max_{0\leq k \leq \kappa} \|\tilde{\mathbf{u}}_k-\mathbf{v}_k\|_2^2 \|f\|^2_{\mathbf{L}^2(\mathcal{X})}\nonumber\\
\leq& \mathcal{O}(N^{-\frac{4}{d+6}}\log(N)) \|f\|^2_{\mathbf{L}^2(\mathcal{X})}.\label{eqn: third term bound}
\end{align}

Combining \eqref{eqn: first term bound}, \eqref{eqn: second term bound}, and \eqref{eqn: third term bound} with \eqref{eqn: three terms} implies that
in the case $d\geq 2$ we have

\begin{align}
&\|H_{N,\epsilon,\kappa,t}\rho f-\rho H_t^\kappa f\|_2^2\nonumber\\
\leq& 3\|\sum_{k=0}^\kappa (e^{-{\lambda_k^{N,\epsilon}}t} - e^{-\mu_kt}) \langle \tilde{\mathbf{u}}_k, \rho f\rangle_2 \tilde{\mathbf{u}}_k\|_2^2\nonumber\\
&\qquad + 3\|\sum_{k=0}^\kappa e^{-\mu_kt} \left(\langle \tilde{\mathbf{u}}_k, \rho f\rangle_2 - \langle f, \varphi_k\rangle_{\mathbf{L}^2(\mathcal{X})} \right)\tilde{\mathbf{u}}_k\|_2^2\nonumber\\
&\qquad\qquad + 3\|\sum_{k=0}^\kappa e^{-\mu_kt}\langle f, \varphi_k\rangle_{\mathbf{L}^2(\mathcal{X})} (\tilde{\mathbf{u}}_k-\mathbf{v}_k)\|_2^2\nonumber
\\
&\leq t^2 \left(\|f\|^2_{\mathbf{L}^2(\mathcal{X})}+\sqrt{\frac{\log N}{N}}\|f\|^2_\infty \right)\mathcal{O}(N^{-4/(d+6)})\nonumber\\
&+ \mathcal{O}\left(\frac{\log N}{N^{\frac{4}{d+6}}}\right) \|f\|^2_{\mathbf{L}^2(\mathcal{X})} + {\left(\mathcal{O}\left(\frac{(\log N)^{3/2}}{N^{\frac{4}{d+6}+\frac{1}{2}}}\right) +\mathcal{O}\left(\frac{\log N}{N}\right)\right)\|f\|^2_\infty}\nonumber\\
&+ \mathcal{O}(N^{-\frac{4}{d+6}}\log(N)) \|f\|^2_{\mathbf{L}^2(\mathcal{X})}
\nonumber\\
&\leq \max\{t^2,1\}\left(\mathcal{O}\left(\frac{\log N}{N^{\frac{4}{d+6}}}\right)\|f\|_{\mathbf{L}^2(\mathcal{X})}^2+ \mathcal{O}\left(\frac{(\log N)^{3/2}}{N^{\frac{4}{d+6}+\frac{1}{2}}}\right)\|f\|^2_\infty\right)\label{eqn: d2spot}\\&=
\max\{t^2,1\}\mathcal{O}\left(\frac{\log N}{N^{\frac{4}{d+6}}}\right)\left(\|f\|_{\mathbf{L}^2(\mathcal{X})}^2 +\sqrt{\frac{\log N}{N}}\|f\|^2_\infty\right)\nonumber
\end{align}
where in \eqref{eqn: d2spot} we used the fact that $d\geq 2.$ Repeating the final string of inequalities in the case where $d=1$, we instead obtain
\begin{align*}
    \|H_{N,\epsilon,\kappa,t}\rho f-\rho H_t^\kappa f\|_2^2\leq \max\{t^2,1\}\left(\mathcal{O}\left(\frac{\log N}{N^{4/7}}\right)\|f\|_{\mathbf{L}^2(\mathcal{X})}^2+ \mathcal{O}\left(\frac{\log N}{N}\right)\|f\|^2_\infty\right)
\end{align*}
as desired.
\end{proof}
\begin{proof}[Proof of Corollary \ref{cor: heat kernel discretization}]
We first note that 
$$ \|H_{N,\epsilon,\kappa,t}\rho f-\rho H_t f\|_2^2\leq 2\|H_{N,\epsilon,\kappa,t}\rho f-\rho H_t^\kappa f\|_2^2 +2 \|\rho H_tf-\rho H_t^\kappa f\|_2^2.$$
Lemma \ref{lem: apply Hoeffding} implies that with probability at least $1-\mathcal{O}\left(\frac{1}{N^9}\right)$
\begin{equation*}
    \|\rho(H_tf -  H^\kappa_tf)\|_2^2  \leq \| H_tf - H^\kappa_tf\|^2 _{\mathbf{L}^2(\mathcal{X})} + \|H_tf - H^\kappa_tf\|^2_\infty\sqrt{\frac{18\log N}{N}}.
\end{equation*}
Therefore, applying Lemma \ref{lem: finite K} implies
\begin{equation*}
    \|\rho(H_tf -  H^\kappa_tf)\|_2^2  \leq e^{-2t\mu_{\kappa+1}} \|f\|^2_{\mathbf{L}^2(\mathcal{X})} + \mathcal{O}\left(\sqrt{\frac{\log N}{N}}\right)\|f\|^2_\infty.
\end{equation*}
Applying Theorem \ref{thm: heat kernel discretization} thus completes the proof.
\end{proof}

\section{The Proof of Theorem \getrefnumber{thm: discretize U}}\label{sec: U conv}

In order to prove Theorem \ref{thm: discretize U}, we will need two lemmas. 
\begin{lemma}\label{lem: U infinity} Let $f\in\mathbf{L}^2(\mathcal{X})$, and 
let $\pathvar=(j_1,\ldots,j_m)$ be a path of length $m$, then 
$$\|U[\pathvar]f\|_\infty\leq 2^m \|f\|_\infty. $$
\end{lemma}

\begin{proof}[Proof of Lemma \ref{lem: U infinity}]
Young's inequality and \eqref{eqn: integrate to one} implies that for all $t>0$ we have $\|H_tf\|_\infty\leq \|f\|_\infty$. Therefore, the case where $m=1$ follows from the triangle inequality and the fact that $\sigma$ is non-expansive. The general case follows from the fact that $\|U[j_1,\ldots,j_m]=U[j_m]\ldots U[j_1]$.
\end{proof}

\begin{lemma}\label{lem: nonexpansive discrete} For all $\mathbf{x},\mathbf{y}\in\mathbb{R}^{N}$ and all $0\leq j\leq J$ we have 
\begin{equation*}
    \|A_{J,N}\mathbf{x}-A_{J,N}\mathbf{y}\|_2\leq \|\mathbf{x} - \mathbf{y}\|_2\end{equation*}
    and
    \begin{equation*}\|U_{N}[j]\mathbf{x}-U_{N}[j]\mathbf{y}\|_2\leq \|W_{j,N}\mathbf{x}-W_{j,N}\mathbf{y}\|_2\leq \|\mathbf{x}-\mathbf{y}\|_2.
\end{equation*}
\end{lemma}
\begin{proof}
By construction we have, for $1\leq j\leq J$
\begin{align*}
    W_{j,N}\mathbf{x}-W_{j,N}\mathbf{y} &= (H_{N,\epsilon,\kappa,2^{j-1}}-H_{N,\epsilon,\kappa,2^j})(\mathbf{x}-\mathbf{y})\\&=\sum_{k=0}^\kappa (e^{-{\lambda_k^{N,\epsilon}2^{j-1}}}-e^{-{\lambda_k^{N,\epsilon}2^{j}}}) \mathbf{u}_k\mathbf{u}_k^T(\mathbf{x}-\mathbf{y}).
\end{align*}
Therefore, the fact that $ \|W_{j,N}\mathbf{x}-W_{j,N}\mathbf{y}\|_2\leq \|\mathbf{x}-\mathbf{y}\|_2$ follows from the fact that the vectors $\{\mathbf{u}_k\}_{k=0}^\kappa$ are an orthonormal basis for their span and the fact that $$|e^{-{\lambda_k^{N,\epsilon}2^{j-1}}}-e^{-{\lambda_k^{N,\epsilon}2^{j}}}|\leq 1.$$ The bounds for $W_{0,N}$ and $A_{J,N}$ follow similarly and the bound for $U_{N}[j]$ follows from the fact that $\sigma$ is nonexpansive.
\end{proof}

\begin{proof}[Proof of Theorem \ref{thm: discretize U}]
We argue by induction on $m$. To establish the base case, we let $p=(j_1)$ and observe that $\sigma$ commutes with $\rho$. Therefore, we have 
\begin{align*}
    \|U_{N}[j_1]\rho f - \rho U[j_1] f\|_2^2 &=
    \|\sigma W_{j_1,N}\rho f - \rho \sigma W_j f\|_2^2\\
    &=\|\sigma W_{j_1,N}\rho f - \sigma\rho W_j f\|_2^2\\
    &\leq\|W_{j_1,N}\rho f - \rho W_j f\|_2^2,
\end{align*}
where the final inequality follows from the fact that $\sigma$ is non-expansive. Therefore, the case where $m=1$ now follows from from Theorem \ref{thm: wavelet discretization}.

Now suppose the theorem is true for $m-1$. Let $p=(j_1,\ldots,j_m)$ be a path of length $m$. Let $p_{m-1}=(j_1,\ldots,j_{m-1})$ so that $U[p]=U[j_m]U[p_{m-1}]$ and $U_N[p]=U_N[j_m]U_N[p_{m-1}]$. Then,
\begin{align*}
    &\|U_{N}[\pathvar]\rho f - \rho U[\pathvar] f\|_2^2\\
    =&\|U_{N}[j_m]U_{N}[\pathvar_{m-1}]\rho f - \rho U[j_m]U[\pathvar_{m-1}] f\|_2^2\\
    =&\|U_{N}[j_m]U_{N}[\pathvar_{m-1}]\rho f -
    U_{N}[j_m]\rho U[\pathvar_{m-1}]f +
    U_{N}[j_m]\rho U[\pathvar_{m-1}]f-\rho U[j_m]U[\pathvar_{m-1}] f\|_2^2\\
        \leq&2\|U_{N}[j_m]U_{N}[\pathvar_{m-1}]\rho f -
    U_{N}[j_m]\rho U[\pathvar_{m-1}]f\|^2_2 +
2\|    U_{N}[j_m]\rho U[\pathvar_{m-1}]f-\rho U[j_m]U[\pathvar_{m-1}] f\|_2^2\\
    \leq&2\|U_{N}[\pathvar_{m-1}]\rho f -
    \rho U[\pathvar_{m-1}]f\|^2_2 +
2\|    U_{N}[j_m]\rho U[\pathvar_{m-1}]f-\rho U[j_m]U[\pathvar_{m-1}] f\|_2^2,
\end{align*}
where in the final inequality we used Lemma \ref{lem: nonexpansive discrete}. 
The term $\|U_{N}[\pathvar_{m-1}]\rho f -   \rho U[\pathvar_{m-1}]f\|_2^2 
 $ may be immediately bounded by the inductive hypothesis. 
Moreover, we may also apply
the inductive hypothesis with 
$U[\pathvar_{m-1}]f$ in place of $f$ to see 
\begin{align*}
    &\|    U_{N}[j_m]\rho U[\pathvar_{m-1}]f-\rho U[j_m]U[\pathvar_{m-1}] f\|_2^2
    \\&\leq 2^{2j_{\max}}\left(\left(\mathcal{O}\left(\frac{\log N}{N^{\frac{4}{d+6}}}\right)+\mathcal{O}(e^{-\mu_{\kappa+1}})\right)\|U[\pathvar_{m-1}]f\|^2_{\mathbf{L}^2(\mathcal{X})}+ \mathcal{O}\left(\sqrt{\frac{\log{N}}{N}}\right)\|U[\pathvar_{m-1}]f\|^2_\infty\right)
\end{align*}
Iteratively applying Proposition \ref{prop: frame} implies  that $\|U[p_{m-1}]f\|_{\mathbf{L}^2(\mathcal{X})}\leq \|f\|_{\mathbf{L}^2(\mathcal{X})}$ and   Lemma \ref{lem: U infinity} implies$\|U[p_{m-1}]f\|_{\infty}\leq 2^{m-1}\|f\|_{\infty}$. Therefore, the result follows.
\end{proof}

\section{The Proofs of Theorems \getrefnumber{thm: convergence windowed} and \getrefnumber{thm: convergence nonwindowed}}\label{sec: scat conv proofs}

\begin{proof}[The Proof of Theorem \ref{thm: convergence windowed}]
\begin{align*}
    &\|S_{J,N}[\pathvar]\rho f - \rho S_J[\pathvar] f\|_2^2\\ =& \|A_{J,N}U_{J,N}[\pathvar]\rho f - \rho A_JU[\pathvar] f\|_2^2\\
    \leq& 2\|A_{J,N}U_{J,N}[\pathvar]\rho f - A_{J,N}\rho U[\pathvar] f\|_2^2 + 2\|A_{J,N}\rho U[\pathvar] f - \rho A_JU[\pathvar] f\|_2^2\\
    \leq& 2\|A_{J,N}\|_2\|U_{J,N}[\pathvar]\rho f - \rho U[\pathvar] f\|_2^2 + 2\|A_{J,N}\rho U[\pathvar] f - \rho A_JU[\pathvar] f\|_2^2\\
    \leq&2\|U_{J,N}[\pathvar]\rho f - \rho U[\pathvar] f\|_2^2 + 2\|A_{J,N}\rho U[\pathvar] f - \rho A_JU[\pathvar] f\|_2^2,
    \end{align*}
    where the last inequality uses Lemma  \ref{lem: nonexpansive discrete}.  To bound 
    $\|U_{J,N}[j]\rho f - \rho U[j] f\|_2^2$, we may apply Theorem \ref{thm: discretize U}.
    To bound the second term, we apply Corollary \ref{cor: heat kernel discretization} with $t=2^J$  to obtain
\begin{align*}
    &\|A_{J,N}\rho U[\pathvar] f - \rho A_JU[\pathvar]\|^2_2\\ \leq&2^{2J}\left(\left(\mathcal{O}\left(\frac{\log N}{N^{\frac{4}{d+6}}}\right)+\mathcal{O}(e^{-2^{J+1}\mu_{\kappa+1}})\right)\|f\|^2_{\mathbf{L}^2(\mathcal{X})}+ \mathcal{O}\left(\sqrt{\frac{\log{N}}{N}}\right)\|f\|^2_\infty\right). \end{align*}
Iteratively applying Proposition \ref{prop: frame} implies  that $\|U[p_{m-1}]f\|_{\mathbf{L^2(\mathcal{X})}}\leq \|f\|_{\mathbf{L^2(\mathcal{X})}}$ and   Lemma \ref{lem: U infinity} implies$\|U[p_{m-1}]f\|_{\infty}\leq 2^{m-1}\|f\|_{\infty}$. Therefore, the result follows.
\end{proof}

\begin{proof}[The Proof of Theorem \ref{thm: convergence nonwindowed}]Let $\alpha_0$ be the scalar from Theorem \ref{thm: 5.4 of Chen and Wu} with $k=0$. By Remark \ref{rem: independent of eigenbasis}, and the definition of the non-windowed scattering coefficients, we may assume without loss of generality that $\alpha_0$ is non-negative (since $-\varphi_0$ is also an eigenfunction). Thus,
recalling that $\mathbf{v}_0=\rho \varphi_0$, we see that by the definition of the non-windowed scattering coefficients, the triangle inequality, and the Cauchy-Schwarz inequality we have 
\begin{align}
    &|\overline{S}_{N}[\pathvar]\rho f -  \overline{S}[\pathvar] f|\\\leq&|\langle U_N[p]\rho f,\mathbf{u}_0\rangle_2-\langle U[p]f,\varphi_0\rangle_{\mathbf{L}^2(\mathcal{X})}|\nonumber\\
    \leq&|\langle U_N[p]\rho f,\mathbf{u}_0\rangle_2-\langle \rho U[p]f,\mathbf{v}_0\rangle_2|+|\langle \rho U[p]f,\mathbf{v}_0\rangle_2-\langle U[p]f,\varphi_0\rangle_{\mathbf{L}^2(\mathcal{X})}|\nonumber\\
    =&|\langle U_N[p]\rho f,\mathbf{u}_0\rangle_2-\langle \frac{1}{\alpha_0}\rho U[p]f,\alpha_0\mathbf{v}_0\rangle_2|+|\langle \rho U[p]f,\mathbf{v}_0\rangle_2-\langle U[p]f,\varphi_0\rangle_{\mathbf{L}^2(\mathcal{X})}|\nonumber\\
    \leq&|\langle U_N[p]\rho f,\mathbf{u}_0-\alpha_0\mathbf{v}_0 \rangle_2|+|\langle U_N[p]\rho f - \frac{1}{\alpha_0}\rho U[p]f,\alpha_0\mathbf{v}_0\rangle_2|+|\langle \rho U[p]f,\mathbf{v}_0\rangle_2-\langle U[p]f,\varphi_0\rangle_{\mathbf{L}^2(\mathcal{X})}|\nonumber\\
    \leq& \|U_N\rho f\|_2\|\mathbf{u}_0-\alpha_0\mathbf{v}_0\|_2+ \|U_N[p]\rho f - \frac{1}{\alpha_0}\rho U[p]f\|_2\|\alpha_0\mathbf{v}_0\|_2 + |\langle \rho U[p]f,\rho\varphi_0\rangle_2-\langle U[p]f,\varphi_0\rangle_{\mathbf{L}^2(\mathcal{X})}|.\label{eqn: bigalign no window}
\end{align}
 Lemmas \ref{lem: apply Hoeffding} and  \ref{lem: nonexpansive discrete} together with the inequality $\sqrt{a^2+b^2}\leq |a|+|b|$ imply
$$\|U_N\rho f\|_2\leq \|\rho f\|_2\leq \|f\|_{\mathbf{L}^2(\mathcal{X})} + \left(\frac{18\log N}{N}\right)^{1/4} \|f\|_\infty $$
with probability at least $1-\mathcal{O}\left(\frac{1}{N^9}\right)$ and Theorem \ref{thm: 5.4 of Chen and Wu} implies that 
$$ \|\mathbf{u}_0-\alpha_0\mathbf{v}_0\|_2 =\mathcal{O}\left(N^{-\frac{2}{d+6}}\sqrt{\log N}\right), $$
again with probability at least $1-\mathcal{O}\left(\frac{1}{N^9}\right)$. Therefore,
\begin{align}\label{eqn:  nowindowT1}
    \|U_N\rho f\|_2\|\mathbf{u}_0-\alpha_0\mathbf{v}_0\|_2\leq \mathcal{O}\left(N^{-\frac{2}{d+6}}\sqrt{\log N}\right)\|f\|_{\mathbf{L}^2(\mathcal{X})} + \mathcal{O}\left(N^{-\frac{2}{d+6}-\frac{1}{4}}(\log N)^{3/4}\right)\|f\|_\infty.
\end{align}
Theorem \ref{thm: 5.4 of Chen and Wu} shows that $|\alpha_0|=1+o(1)$, and \eqref{eqn: normalized} implies that $\|\varphi_0\|_{\mathbf{L}^2(\mathcal{X})}=\|\varphi_0\|_\infty=1$. Therefore, Lemma \ref{lem: apply Hoeffding} implies
\begin{equation}
    \|\alpha_0 \mathbf{v}_0\|_2\leq (1+o(1)) \|\rho \varphi_0\|_2\leq (1+o(1))\left(  \|\varphi_0\|_{\mathbf{L}^2(\mathcal{X})}^2 + \sqrt{\frac{\log N}{N}} \|\varphi_0\|^2_\infty\right)=\mathcal{O}(1).
\label{eqn: NOTWORKING}\end{equation}
 Proposition \ref{prop: frame} and a simple induction argument implies $\|U[p]f\|_{\mathbf{L}^2(\mathcal{X})}\leq \|f\|_{\mathbf{L}^2(\mathcal{X})}$, and Remark \ref{rem: alpha rate} implies $$\left|\frac{1}{\alpha_k}-1\right|\leq  \mathcal{O}\left(\sqrt{\frac{\log N}{N}}\right)+\mathcal{O}\left(\frac{\log(N)}{N^{4/(d+6)}}\right).$$ Therefore, by Theorem \ref{thm: discretize U}, Lemma \ref{lem: apply Hoeffding}, and Lemma \ref{lem: U infinity}, we have 
\begin{align}\|U_N[p]\rho f - \frac{1}{\alpha_0}\rho U[p]f\|_2&\leq\|U_N[p]\rho f - \rho U[p]f\|_2+ \bigg|\frac{1}{\alpha_0}-1\bigg|\left\|\rho U[p]f\right\|_2\nonumber\\
&\leq 2^{J}\left [\left(\mathcal{O}\left(\frac{\sqrt{\log N}}{N^{\frac{2}{d+6}}}\right)+\mathcal{O}(e^{-\mu_{\kappa+1}/2})\right)\|f\|_{\mathbf{L}^2(\mathcal{X})}+ \mathcal{O}\left(\left(\frac{\log{N}}{N}\right)^{1/4}\right)\|f\|_\infty\right]\nonumber\\
&\quad+\left(\mathcal{O}\left(\sqrt{\frac{\log N}{N}}\right)+\mathcal{O}\left(\frac{\log(N)}{N^{4/(d+6)}}\right)\right)\|f\|_{\mathbf{L}^2(\mathcal{X})}\nonumber\\
&\quad+ \left(\mathcal{O}\left(\left(\frac{\log N}{N}\right)^{3/4}\right)+\mathcal{O}\left(\frac{\log^{5/4}(N)}{N^{4/(d+6)+1/4}}\right)\right)\|f\|_\infty\nonumber\\
&= \left(\mathcal{O}\left(\frac{\sqrt{\log N}}{N^{\frac{2}{d+6}}}\right)2^{J}+\mathcal{O}\left(\sqrt{\frac{\log N}{N}}\right)+\mathcal{O}(e^{-\mu_{\kappa+1}/2})2^{J}\right)\|f\|_{\mathbf{L}^2(\mathcal{X})}\nonumber\\&\quad+ \mathcal{O}\left(\left(\frac{\log{N}}{N}\right)^{1/4}\right)2^{J}\|f\|_\infty.\label{eqn: nowindow alpha term}
\end{align}
Lastly, we again apply Lemma \ref{lem: apply Hoeffding} and Lemma \ref{lem: U infinity} to see that 
\begin{align}
    |\langle \rho U[p]f,\rho\varphi_0\rangle_2-\langle U[p]f,\varphi_0\rangle_{\mathbf{L}^2(\mathcal{X})}| &= \mathcal{O}\left(\sqrt{\frac{\log N}{N}}\right) \|U[p]f\varphi_0\|_\infty\nonumber\\
    &= \mathcal{O}\left(\sqrt{\frac{\log N}{N}}\right) \|f\|_\infty\label{eqn: small term}
\end{align}
with probability at least $1-\mathcal{O}\left(\frac{1}{N^9}\right).$
Combining \eqref{eqn: bigalign no window} with \eqref{eqn:  nowindowT1},   \eqref{eqn: NOTWORKING}, \eqref{eqn: nowindow alpha term}, and \eqref{eqn: small term} yields
\begin{align*}
    &|\overline{S}_{N}[\pathvar]\rho f -  \overline{S}[\pathvar] f|\\
    \leq & \|U_N\rho f\|_2\|\mathbf{u}_0-\alpha_0\mathbf{v}_0\|_2+ \|U_N[p]\rho f - \frac{1}{\alpha_0}\rho U[p]f\|_2\|\alpha_0\mathbf{v}_0\|_2 + |\langle \rho U[p]f,\rho\varphi_0\rangle_2-\langle U[p]f,\varphi_0\rangle_{\mathbf{L}^2(\mathcal{X})}|
    \\
    \leq & \mathcal{O}\left(N^{-\frac{2}{d+6}}\sqrt{\log N}\right)\|f\|_{\mathbf{L}^2(\mathcal{X})} + \mathcal{O}\left(N^{-\frac{2}{d+6}-\frac{1}{4}}(\log N)^{3/4}\right)\|f\|_\infty\\
&\quad+ \left(\mathcal{O}\left(\frac{\sqrt{\log N}}{N^{\frac{2}{d+6}}}\right)2^{J}+\mathcal{O}\left(\sqrt{\frac{\log N}{N}}\right)+\mathcal{O}(e^{-\mu_{\kappa+1}/2})2^{J}\right)\|f\|_{\mathbf{L}^2(\mathcal{X})}\nonumber\\&\quad+ \mathcal{O}\left(\left(\frac{\log{N}}{N}\right)^{1/4}\right)2^{J}\|f\|_\infty+ \mathcal{O}\left(\sqrt{\frac{\log N}{N}}\right) \|f\|_\infty\\\leq& 2^{J}\left [\left(\mathcal{O}\left(\frac{\sqrt{\log N}}{N^{\frac{2}{d+6}}}\right)+\mathcal{O}\left(e^{-\mu_{\kappa+1}/2}\right)\right)\|f\|_{\mathbf{L}^2(\mathcal{X})}+ \mathcal{O}\left(\left(\frac{\log N}{N}\right)^{1/4}\right)\|f\|_\infty\right].
\end{align*}
\end{proof}

\section{ Details on the Baseline Method}\label{sec: kmeans details}
For both biomedical datasets, in our baseline classification method, we  first performed  $k$-means clustering on all cells from all patients (modeled as points in either $\mathbb{R}^{30}$ or $\mathbb{R}^{14}$). The value of $k$ was based on expected subsets of immune cells: for the melanoma data we set $k=3$ based on expected subsets of CD4+ T helper cells, CD8+ killer T cells, and FOXP3+ T regulatory cells, and in COVID data we again set $k=3$ based on expected subsets of CD14+CD16++ non-classical monocytes, CD14++CD16 intermediate monocytes, and CD14++CD16- classical monocytes. Then, for each patient, we identified the proportion of cells corresponding to that patient lying within each cluster. We then used these features as input to a decision tree classifier.

\section{Training Details for Section \ref{sec: results digraph}}\label{sec: hyperparams} The results for baseline methods presented in Table \ref{tab: dsbm} are taken directly from \cite{zhang2021magnet}. Therefore, for a fair comparison, we use the same validation procedure when training our method as was used in \cite{zhang2021magnet}. For each of the three meta-graphs, we independently, randomly generated 5 realizations of the DSBM. 
For each of these realizations, we randomly generated 10 training/test/validation splits. To tune our hyperparameters, $J$, $q$, $c$ and $\gamma$ (the latter two of which are hyperparameters of the SVM), we picked a single realization of each model and performed a grid search, choosing the parameters with the best average validation accuracy over the $10$ splits. We then used these hyperparameters for all five realizations of each model (following the standard procedure of training on the training set and testing on the test set, holding out the validation set). The results reported in Table \ref{tab: dsbm} are the test accuracies averaged over both the 5 realizations of each model and the 10 training/test/validation splits (i.e., over all 50 of the test sets).  In our search, we selected $J$ from a pool of $\{2,3,\ldots,12\}$, magnetic Laplacian charge parameter $q$ from a pool of $\{0,.05,.10,.15,.20,.25\}$, and SVM parameters from pools of $c\in\{25,100,250,500,1000\}$ and $\gamma \in \{10^{-5},10^{-4},10^{-3},10^{-2},10^{-1}\}$.  

\bibliographystyle{plain}
\bibliography{main}

\begin{thebibliography}{10}

\bibitem{belkin:laplaciansEigen2003}
Mikhail Belkin and Partha Niyogi.
\newblock {L}aplacian eigenmaps for dimensionality reduction and data
  representation.
\newblock {\em Neural Computation}, 15(6):1373--1396, 2003.

\bibitem{belkin2007convergence}
Mikhail Belkin and Partha Niyogi.
\newblock Convergence of {L}aplacian eigenmaps.
\newblock In {\em Advances in Neural Information Processing Systems}, pages
  129--136, 2007.

\bibitem{bhaskar2021molecular}
Dhananjay Bhaskar, Jackson~D Grady, Michael~A Perlmutter, and Smita
  Krishnaswamy.
\newblock Molecular graph generation via geometric scattering.
\newblock In {\em 2022 IEEE 32nd International Workshop on Machine Learning for
  Signal Processing (MLSP)}, 2022.

\bibitem{Bogo:CVPR:2014}
Federica Bogo, Javier Romero, Matthew Loper, and Michael~J. Black.
\newblock {FAUST}: Dataset and evaluation for {3D} mesh registration.
\newblock In {\em Proceedings IEEE Conf. on Computer Vision and Pattern
  Recognition (CVPR)}, 2014.

\bibitem{boscaini2015learning}
Davide Boscaini, Jonathan Masci, Simone Melzi, Michael~M Bronstein, Umberto
  Castellani, and Pierre Vandergheynst.
\newblock Learning class-specific descriptors for deformable shapes using
  localized spectral convolutional networks.
\newblock In {\em Computer Graphics Forum}, volume~34, pages 13--23. Wiley
  Online Library, 2015.

\bibitem{boscaini2016learning}
Davide Boscaini, Jonathan Masci, Emanuele Rodol\`{a}, and Michael Bronstein.
\newblock Learning shape correspondence with anisotropic convolutional neural
  networks.
\newblock In {\em Advances in Neural Information Processing Systems 29}, pages
  3189--3197, 2016.

\bibitem{bronstein2021geometric}
Michael~M Bronstein, Joan Bruna, Taco Cohen, and Petar Veli{\v{c}}kovi{\'c}.
\newblock Geometric deep learning: Grids, groups, graphs, geodesics, and
  gauges.
\newblock {\em arXiv preprint arXiv:2104.13478}, 2021.

\bibitem{Bronstein:geoDeepLearn2017}
Michael~M. Bronstein, Joan Bruna, Yann LeCun, Arthur Szlam, and Pierre
  Vandergheynst.
\newblock Geometric deep learning: Going beyond {E}uclidean data.
\newblock {\em IEEE Signal Processing Magazine}, 34(4):18--42, 2017.

\bibitem{bruna:multiscaleMicrocanonical2018}
Joan Bruna and St\'{e}phane Mallat.
\newblock Multiscale sparse microcanonical models.
\newblock {\em Mathematical Statistics and Learning}, 1(3/4):257--315, 2018.

\bibitem{bruna2013spectral}
Joan Bruna, Wojciech Zaremba, Arthur Szlam, and Yann LeCun.
\newblock Spectral networks and locally connected networks on graphs.
\newblock In Yoshua Bengio and Yann LeCun, editors, {\em 2nd International
  Conference on Learning Representations, {ICLR} 2014, Banff, AB, Canada, April
  14-16, 2014, Conference Track Proceedings}, 2014.

\bibitem{Cahill2022}
Jameson Cahill, Joseph~W Iverson, Dustin~G Mixon, and Daniel Packer.
\newblock Group-invariant max filtering.
\newblock {\em arXiv preprint arXiv:2205.14039}, 2022.

\bibitem{Calder2019}
Jeff Calder and Nicolas~Garcia Trillos.
\newblock Improved spectral convergence rates for graph {L}aplacians on
  $\varepsilon$-graphs and k-nn graphs.
\newblock {\em Applied and Computational Harmonic Analysis}, 60:123--175, 2022.

\bibitem{canzani:analysisManifoldsLaplacian2013}
Yaiza Canzani.
\newblock Analysis on manifolds via the {L}aplacian.
\newblock Course notes for Math 253, Fall 2013, Harvard University, 2013.

\bibitem{chang2011libsvm}
Chih-Chung Chang and Chih-Jen Lin.
\newblock {LIBSVM}: a library for support vector machines.
\newblock {\em ACM {T}ransactions on {I}ntelligent {S}ystems and {T}echnology
  (TIST)}, 2(3):1--27, 2011.

\bibitem{cheng2021eigen}
Xiuyuan Cheng and Nan Wu.
\newblock Eigen-convergence of {G}aussian kernelized graph {L}aplacian by
  manifold heat interpolation.
\newblock {\em Applied and Computational Harmonic Analysis}, 61:132--190, 2022.

\bibitem{chodrow2022nonbacktracking}
Philip Chodrow, Nicole Eikmeier, and Jamie Haddock.
\newblock Nonbacktracking spectral clustering of nonuniform hypergraphs.
\newblock {\em SIAM Journal on Mathematics of Data Science}, 5(2):251--279,
  2023.

\bibitem{chung2005laplacians}
Fan Chung.
\newblock Laplacians and the {C}heeger inequality for directed graphs.
\newblock {\em Annals of Combinatorics}, 9(1):1--19, 2005.

\bibitem{church2017word2vec}
Kenneth~Ward Church.
\newblock Word2vec.
\newblock {\em Natural Language Engineering}, 23(1):155--162, 2017.

\bibitem{cloninger2017note}
Alexander Cloninger.
\newblock A note on {M}arkov normalized magnetic eigenmaps.
\newblock {\em Applied and Computational Harmonic Analysis}, 43(2):370--380,
  2017.

\bibitem{coifman:diffusionMaps2006}
Ronald~R. Coifman and St\'{e}phane Lafon.
\newblock Diffusion maps.
\newblock {\em Applied and Computational Harmonic Analysis}, 21:5--30, 2006.

\bibitem{coifman:diffWavelets2006}
Ronald~R. Coifman and Mauro Maggioni.
\newblock Diffusion wavelets.
\newblock {\em Applied and Computational Harmonic Analysis}, 21(1):53--94,
  2006.

\bibitem{cristianini2000introduction}
Nello Cristianini, John Shawe-Taylor, et~al.
\newblock {\em An introduction to support vector machines and other
  kernel-based learning methods}.
\newblock Cambridge university press, 2000.

\bibitem{cucuringu2021regularized}
Mihai Cucuringu, Apoorv~Vikram Singh, D{\'e}borah Sulem, and Hemant Tyagi.
\newblock Regularized spectral methods for clustering signed networks.
\newblock {\em Journal of Machine Learning Research}, 22(264):1--79, 2021.

\bibitem{czaja2019analysis}
Wojciech Czaja and Weilin Li.
\newblock Analysis of time-frequency scattering transforms.
\newblock {\em Applied and Computational Harmonic Analysis}, 47(1):149--171,
  2019.

\bibitem{DAVIES1984335}
E.B Davies and B~Simon.
\newblock Ultracontractivity and the heat kernel for {S}chrödinger operators
  and {D}irichlet {L}aplacians.
\newblock {\em Journal of Functional Analysis}, 59(2):335--395, 1984.

\bibitem{Defferrard2018}
Micha\"{e}l Defferrard, Xavier Bresson, and Pierre Vandergheynst.
\newblock Convolutional neural networks on graphs with fast localized spectral
  filtering.
\newblock In {\em Advances in Neural Information Processing Systems 29}, pages
  3844--3852, 2016.

\bibitem{dunson2021spectral}
David~B Dunson, Hau-Tieng Wu, and Nan Wu.
\newblock Spectral convergence of graph {L}aplacian and heat kernel
  reconstruction in l-infinity from random samples.
\newblock {\em Applied and Computational Harmonic Analysis}, 55:282--336, 2021.

\bibitem{f2020characterization}
Bruno~Messias F.~de Resende and Luciano~da F.~Costa.
\newblock Characterization and comparison of large directed networks through
  the spectra of the magnetic {L}aplacian.
\newblock {\em Chaos: An Interdisciplinary Journal of Nonlinear Science},
  30(7):073141, 2020.

\bibitem{fanuel2018magnetic}
Micha{\"e}l Fanuel, Carlos~M Ala{\'\i}z, {\'A}ngela Fern{\'a}ndez, and Johan~AK
  Suykens.
\newblock Magnetic eigenmaps for the visualization of directed networks.
\newblock {\em Applied and Computational Harmonic Analysis}, 44(1):189--199,
  2018.

\bibitem{fanuel2017magnetic}
Micha{\"e}l Fanuel, Carlos~M Alaiz, and Johan~AK Suykens.
\newblock Magnetic eigenmaps for community detection in directed networks.
\newblock {\em Physical Review E}, 95(2):022302, 2017.

\bibitem{feng2019hypergraph}
Yifan Feng, Haoxuan You, Zizhao Zhang, Rongrong Ji, and Yue Gao.
\newblock Hypergraph neural networks.
\newblock In {\em Proceedings of the AAAI {c}onference on {A}rtificial
  {I}ntelligence}, volume~33, pages 3558--3565, 2019.

\bibitem{sigma}
Stefano Fiorini, Stefano Coniglio, Michele Ciavotta, and Enza Messina.
\newblock Sigmanet: One laplacian to rule them all.
\newblock In {\em Proceedings of the AAAI Conference on Artificial
  Intelligence}, volume~37, pages 7568--7576, 2023.

\bibitem{furutani2019graph}
Satoshi Furutani, Toshiki Shibahara, Mitsuaki Akiyama, Kunio Hato, and Masaki
  Aida.
\newblock Graph signal processing for directed graphs based on the hermitian
  {L}aplacian.
\newblock In {\em Joint European Conference on Machine Learning and Knowledge
  Discovery in Databases}, pages 447--463. Springer, 2019.

\bibitem{gama:diffScatGraphs2018}
Fernando Gama, Alejandro Ribeiro, and Joan Bruna.
\newblock Diffusion scattering transforms on graphs.
\newblock In {\em International Conference on Learning Representations}, 2019.

\bibitem{gao:graphScat2018}
Feng Gao, Guy Wolf, and Matthew Hirn.
\newblock Geometric scattering for graph data analysis.
\newblock In {\em Proceedings of the 36th International Conference on Machine
  Learning, PMLR}, volume~97, pages 2122--2131, 2019.

\bibitem{grohs:cnnCartoonFcns2016}
Philipp Grohs, Thomas Wiatowski, and Helmut B{\"o}lcskei.
\newblock Deep convolutional neural networks on cartoon functions.
\newblock In {\em IEEE International Symposium on Information Theory}, pages
  1163--1167, 2016.

\bibitem{grover2016node2vec}
Aditya Grover and Jure Leskovec.
\newblock node2vec: Scalable feature learning for networks.
\newblock In {\em Proceedings of the 22nd ACM SIGKDD {I}nternational
  {C}onference on {K}nowledge {D}iscovery and {D}ata {M}ining}, pages 855--864,
  2016.

\bibitem{hamilton2017inductive}
William~L. Hamilton, Rex Ying, and Jure Leskovec.
\newblock Inductive representation learning on large graphs.
\newblock In {\em Proceedings of the 31st International Conference on Neural
  Information Processing Systems}, NIPS'17, page 1025–1035, Red Hook, NY,
  USA, 2017. Curran Associates Inc.

\bibitem{hammond:graphWavelets2011}
David~K. Hammond, Pierre Vandergheynst, and R\'{e}mi Gribonval.
\newblock Wavelets on graphs via spectral graph theory.
\newblock {\em Applied and Computational Harmonic Analysis}, 30:129--150, 2011.

\bibitem{he2022msgnn}
Yixuan He, Michael Perlmutter, Gesine Reinert, and Mihai Cucuringu.
\newblock Msgnn: A spectral graph neural network based on a novel magnetic
  signed laplacian.
\newblock In {\em Learning on Graphs Conference}, pages 40--1. PMLR, 2022.

\bibitem{hein2007graph}
Matthias Hein, Jean-Yves Audibert, and Ulrike~von Luxburg.
\newblock Graph {L}aplacians and their convergence on random neighborhood
  graphs.
\newblock {\em Journal of Machine Learning Research}, 8(6), 2007.

\bibitem{hoffmann2022spectral}
Franca Hoffmann, Bamdad Hosseini, Assad~A Oberai, and Andrew~M Stuart.
\newblock Spectral analysis of weighted {L}aplacians arising in data
  clustering.
\newblock {\em Applied and Computational Harmonic Analysis}, 56:189--249, 2022.

\bibitem{huang2022decade}
Alexander~C Huang and Roberta Zappasodi.
\newblock A decade of checkpoint blockade immunotherapy in melanoma:
  understanding the molecular basis for immune sensitivity and resistance.
\newblock {\em Nature {I}mmunology}, 23(5):660--670, 2022.

\bibitem{Keller2021}
Matthias Keller, Daniel Lenz, and Rados{\l}aw~K. Wojciechowski.
\newblock {\em Large Time Behavior of the Heat Kernel}, pages 241--254.
\newblock Springer International Publishing, Cham, 2021.

\bibitem{kipf2016semi}
T.~Kipf and M.~Welling.
\newblock Semi-supervised classification with graph convolutional networks.
\newblock {\em Proc. of ICLR}, 2016.

\bibitem{klicpera2018predict}
Johannes Klicpera, Aleksandar Bojchevski, and Stephan G{\"u}nnemann.
\newblock Predict then propagate: Graph neural networks meet personalized
  pagerank.
\newblock In {\em ICLR}, 2019.

\bibitem{ko2023spectral}
Taewook Ko, Yoonhyuk Choi, and Chong-Kwon Kim.
\newblock A spectral graph convolution for signed directed graphs via magnetic
  laplacian.
\newblock {\em Neural Networks}, 164:562--574, 2023.

\bibitem{RN3}
Manik Kuchroo et~al.
\newblock Multiscale {PHATE} identifies multimodal signatures of {COVID-19}.
\newblock {\em Nature Biotechnology}, 2022.

\bibitem{LEONARDUZZI201811}
Roberto Leonarduzzi, Haixia Liu, and Yang Wang.
\newblock Scattering transform and sparse linear classifiers for art
  authentication.
\newblock {\em Signal Processing}, 150:11--19, 2018.

\bibitem{Levie:CayleyNets2017}
Ron Levie, Federico Monti, Xavier Bresson, and Michael~M Bronstein.
\newblock Cayleynets: Graph convolutional neural networks with complex rational
  spectral filters.
\newblock {\em IEEE Transactions on Signal Processing}, 67(1):97--109, 2018.

\bibitem{lieb1993fluxes}
Elliott~H Lieb and Michael Loss.
\newblock Fluxes, {L}aplacians, and {K}asteleyn’s theorem.
\newblock In {\em Statistical Mechanics}, pages 457--483. Springer, 1993.

\bibitem{lim2020hodge}
Lek-Heng Lim.
\newblock Hodge laplacians on graphs.
\newblock {\em {SIAM} Review}, 62(3):685--715, 2020.

\bibitem{lindenbaum2020gaussian}
Ofir Lindenbaum, Moshe Salhov, Arie Yeredor, and Amir Averbuch.
\newblock Gaussian bandwidth selection for manifold learning and
  classification.
\newblock {\em Data {M}ining and {K}nowledge {D}iscovery}, 34:1676--1712, 2020.

\bibitem{Little2022Balancing}
Anna Little, Daniel McKenzie, and James~M. Murphy.
\newblock Balancing geometry and density: Path distances on high-dimensional
  data.
\newblock {\em SIAM Journal on Mathematics of Data Science}, 4(1):72--99, 2022.

\bibitem{little2020path}
Anna~V Little, Mauro Maggioni, and James~M Murphy.
\newblock Path-based spectral clustering: Guarantees, robustness to outliers,
  and fast algorithms.
\newblock {\em Journal of Machine Learning Research}, 21, 2020.

\bibitem{ma2019spectral}
Yi~Ma, Jianye Hao, Yaodong Yang, Han Li, Junqi Jin, and Guangyong Chen.
\newblock Spectral-based graph convolutional network for directed graphs.
\newblock {\em arXiv preprint arXiv:1907.08990}, 2019.

\bibitem{mallat:scattering2012}
St{\'e}phane Mallat.
\newblock Group invariant scattering.
\newblock {\em Communications on Pure and Applied Mathematics},
  65(10):1331--1398, October 2012.

\bibitem{maskey2022generalization}
Sohir Maskey, Ron Levie, Yunseok Lee, and Gitta Kutyniok.
\newblock Generalization analysis of message passing neural networks on large
  random graphs.
\newblock {\em Advances in neural information processing systems},
  35:4805--4817, 2022.

\bibitem{mcewen2021scattering}
Jason McEwen, Christopher Wallis, and Augustine~N. Mavor-Parker.
\newblock Scattering networks on the sphere for scalable and rotationally
  equivariant spherical {CNN}s.
\newblock In {\em International Conference on Learning Representations}, 2022.

\bibitem{Min2022MC}
Yimeng Min, Frederik Wenkel, Michael Perlmutter, and Guy Wolf.
\newblock Can hybrid geometric scattering networks help solve the maximum
  clique problem?
\newblock {\em Advances in Neural Information Processing Systems},
  35:22713--22724, 2022.

\bibitem{nadler:dmDynamic2006}
Boaz Nadler, St\'{e}phane Lafon, Ronald~R. Coifman, and Ioannis~G. Kevrekidis.
\newblock Diffusion maps, spectral clustering and reaction coordinates of
  dynamical systems.
\newblock {\em Applied and Computational Harmonic Analysis}, 21(1):113--127,
  2006.

\bibitem{narayanan2017:graph2vec}
Annamalai Narayanan, Mahinthan Chandramohan, Rajasekar Venkatesan, Lihui Chen,
  Yang Liu, and Shantanu Jaiswal.
\newblock graph2vec: Learning distributed representations of graphs.
\newblock {\em CoRR}, abs/1707.05005, 2017.

\bibitem{perlmutter:geoScatCompactManifold2020}
Michael Perlmutter, Feng Gao, Guy Wolf, and Matthew Hirn.
\newblock Geometric scattering networks on compact {R}iemannian manifolds.
\newblock In {\em Mathematical and Scientific Machine Learning Conference},
  2020.

\bibitem{perlmutter2019understanding}
Michael Perlmutter, Alexander Tong, Feng Gao, Guy Wolf, and Matthew Hirn.
\newblock Understanding graph neural networks with generalized geometric
  scattering transforms.
\newblock {\em arXiv preprint arXiv:1911.06253}, 2019.

\bibitem{PtacekA59}
Jason Ptacek, Matthew Vesely, David Rimm, Monirath Hav, Murat Aksoy, Ailey
  Crow, and Jessica Finn.
\newblock 52 characterization of the tumor microenvironment in melanoma using
  multiplexed ion beam imaging ({MIBI}).
\newblock {\em Journal for ImmunoTherapy of Cancer}, 9(Suppl 2):A59--A59, 2021.

\bibitem{saito2022multiscale}
Naoki Saito, Stefan~C Schonsheck, and Eugene Shvarts.
\newblock Multiscale transforms for signals on simplicial complexes.
\newblock {\em arXiv preprint arXiv:2301.02136}, 2022.

\bibitem{saito2023multiscale}
Naoki Saito, Stefan~C Schonsheck, and Eugene Shvarts.
\newblock Multiscale hodge scattering networks for data analysis.
\newblock {\em arXiv preprint arXiv:2311.10270}, 2023.

\bibitem{saito2017underwater}
Naoki Saito and David~S Weber.
\newblock Underwater object classification using scattering transform of sonar
  signals.
\newblock In {\em Wavelets and Sparsity XVII}, volume 10394, pages 103--115.
  SPIE, 2017.

\bibitem{schaub2021signal}
Michael~T Schaub, Yu~Zhu, Jean-Baptiste Seby, T~Mitchell Roddenberry, and
  Santiago Segarra.
\newblock Signal processing on higher-order networks: Livin’on the edge...
  and beyond.
\newblock {\em Signal Processing}, 187:108149, 2021.

\bibitem{shi:gradEigfcnManifold2010}
Yiqian Shi and Bin Xu.
\newblock Gradient estimate of an eigenfunction on a compact {R}iemannian
  manifold without boundary.
\newblock {\em Annals of Global Analysis and Geometry}, 38:21--26, 2010.

\bibitem{shuman:emerging2013}
David~I. Shuman, Sunil~K. Narang, Pascal Frossard, Antonio Ortega, and Pierre
  Vandergheynst.
\newblock The emerging field of signal processing on graphs: Extending
  high-dimensional data analysis to networks and other irregular domains.
\newblock {\em IEEE Signal Processing Magazine}, 30(3):83--98, 2013.

\bibitem{singer2011orientability}
Amit Singer and Hau-Tieng Wu.
\newblock Orientability and diffusion maps.
\newblock {\em Applied and Computational Harmonic Analysis}, 31(1):44--58,
  2011.

\bibitem{singer2012vector}
Amit Singer and Hau-Tieng Wu.
\newblock Vector diffusion maps and the connection {L}aplacian.
\newblock {\em Communications on Pure and Applied Mathematics},
  65(8):1067--1144, 2012.

\bibitem{singh2022signed}
Rahul Singh and Yongxin Chen.
\newblock Signed graph neural networks: A frequency perspective.
\newblock {\em arXiv preprint arXiv:2208.07323}, 2022.

\bibitem{sprechmann2015audio}
Pablo Sprechmann, Joan Bruna, and Yann LeCun.
\newblock Audio source separation with discriminative scattering networks.
\newblock In {\em International Conference on Latent Variable Analysis and
  Signal Separation}, pages 259--267. Springer, 2015.

\bibitem{tombari2010unique}
Federico Tombari, Samuele Salti, and Luigi Di~Stefano.
\newblock Unique signatures of histograms for local surface description.
\newblock In {\em European Conference on Computer Vision}, pages 356--369,
  2010.

\bibitem{tong2022data}
Alexander Tong, Frederik Wenkel, Dhananjay Bhaskar, Kincaid Macdonald, Jackson
  Grady, Michael Perlmutter, Smita Krishnaswamy, and Guy Wolf.
\newblock Learnable filters for geometric scattering modules.
\newblock {\em arXiv preprint arXiv:2208.07458}, 2022.

\bibitem{tong2020digraph}
Zekun Tong, Yuxuan Liang, Changsheng Sun, Xinke Li, David Rosenblum, and Andrew
  Lim.
\newblock Digraph inception convolutional networks.
\newblock {\em Advances in Neural Information Processing Systems},
  33:17907--17918, 2020.

\bibitem{tong:directedGCN2020}
Zekun Tong, Yuxuan Liang, Changsheng Sun, David~S. Rosenblum, and Andrew Lim.
\newblock Directed graph convolutional network.
\newblock arXiv:2004.13970, 2020.

\bibitem{velivckovic2017graph}
Petar Veli{\v{c}}kovi{\'c}, Guillem Cucurull, Arantxa Casanova, Adriana Romero,
  Pietro Li{\`o}, and Yoshua Bengio.
\newblock Graph attention networks.
\newblock In {\em International Conference on Learning Representations}, 2018.

\bibitem{wang2021stability}
Zhiyang Wang, Luana Ruiz, and Alejandro Ribeiro.
\newblock Stability of neural networks on riemannian manifolds.
\newblock In {\em 2021 29th European Signal Processing Conference (EUSIPCO)},
  pages 1845--1849, 2021.

\bibitem{wang2021stabilityrel}
Zhiyang Wang, Luana Ruiz, and Alejandro Ribeiro.
\newblock Stability of neural networks on manifolds to relative perturbations.
\newblock In {\em ICASSP 2022-2022 IEEE International Conference on Acoustics,
  Speech and Signal Processing (ICASSP)}, pages 5473--5477. IEEE, 2022.

\bibitem{wenkel2022overcoming}
Frederik Wenkel, Yimeng Min, Matthew Hirn, Michael Perlmutter, and Guy Wolf.
\newblock Overcoming oversmoothness in graph convolutional networks via hybrid
  scattering networks.
\newblock {\em arXiv preprint arXiv:2201.08932}, 2022.

\bibitem{wiatowski:frameScat2015}
Thomas Wiatowski and Helmut B{\"o}lcskei.
\newblock Deep convolutional neural networks based on semi-discrete frames.
\newblock In {\em Proceedings of IEEE International Symposium on Information
  Theory}, pages 1212--1216, 2015.

\bibitem{wiatowski:mathTheoryCNN2018}
Thomas Wiatowski and Helmut B{\"o}lcskei.
\newblock A mathematical theory of deep convolutional neural networks for
  feature extraction.
\newblock {\em IEEE Transactions on Information Theory}, 64(3):1845--1866,
  2018.

\bibitem{xu2018how}
Keyulu Xu, Weihua Hu, Jure Leskovec, and Stefanie Jegelka.
\newblock How powerful are graph neural networks?
\newblock In {\em International Conference on Learning Representations}, 2019.

\bibitem{zhang2021magnet}
Xitong Zhang, Yixuan He, Nathan Brugnone, Michael Perlmutter, and Matthew Hirn.
\newblock {M}ag{N}et: A neural network for directed graphs.
\newblock {\em Advances in Neural Information Processing Systems}, 34, 2021.

\bibitem{zou:graphScatGAN2019}
Dongmian Zou and Gilad Lerman.
\newblock Encoding robust representation for graph generation.
\newblock In {\em International Joint Conference on Neural Networks}, 2019.

\bibitem{zou:graphCNNScat2018}
Dongmian Zou and Gilad Lerman.
\newblock Graph convolutional neural networks via scattering.
\newblock {\em Applied and Computational Harmonic Analysis},
  49(3)(3):1046--1074, 2019.

\end{thebibliography}
\end{document}